\documentclass[11pt]{article}

\usepackage{amsmath,amsfonts,bm}

\def\eqref#1{equation~\ref{#1}}
\def\1{\bm{1}}

\DeclareMathAlphabet{\mathsfit}{\encodingdefault}{\sfdefault}{m}{sl}
\SetMathAlphabet{\mathsfit}{bold}{\encodingdefault}{\sfdefault}{bx}{n}

\newcommand{\R}{\mathbb{R}}

\newcommand{\diag}{\mathrm{diag}}

\newcommand{\SVD}{\textrm{SVD}}

\DeclareMathOperator*{\argmax}{arg\,max}
\DeclareMathOperator*{\argmin}{arg\,min}

\DeclareMathOperator{\Tr}{Tr}

\usepackage[a4paper,top=3cm,bottom=3cm,left=3cm,right=3cm,marginparwidth=1.75cm]{geometry}
\usepackage[utf8]{inputenc} % allow utf-8 input
\usepackage[T1]{fontenc}    % use 8-bit T1 fonts
\usepackage{babel}
\usepackage{url}            % simple URL typesetting
\usepackage{booktabs}       % professional-quality tables
\usepackage{amsfonts}       % blackboard math symbols
\usepackage{nicefrac}       % compact symbols for 1/2, etc.
\usepackage{microtype}      % microtypography
\usepackage{lipsum}		% Can be removed after putting your text content
\usepackage{graphicx}
\usepackage{natbib}
\usepackage{doi}
\usepackage{physics}
\usepackage{color}
\usepackage{soul}
\usepackage{multirow}
\usepackage{algorithm}
\usepackage{algorithmicx}
\usepackage{algpseudocode}
\usepackage{caption}
\usepackage{amsmath,amssymb,amsthm}
\usepackage{xr}
\usepackage{blindtext}
\usepackage{caption}
\usepackage{hyperref}       % hyperlinks
\usepackage[title]{appendix}

\newtheorem{thm}{Theorem}[section]
\newtheorem{prop}{Proposition}[section]
\newtheorem{lem}{Lemma}[section]
\newtheorem{cor}{Corollary}[section]
\theoremstyle{definition}
\newtheorem{dfn}{Definition}[section]
\theoremstyle{remark}
\newtheorem{rem}{Remark}[section] 
\newtheorem{asm}{Assumption}[section]
\numberwithin{equation}{section}

\newcommand{\lm}[1]{\textcolor{blue}{(RN: #1)}}

\newcommand{\wl}[1]{\textcolor{red}{(Wenlong: #1)}}

\newcommand{\bc}{\color{blue}}

\title{Understanding Multimodal Contrastive Learning\\and Incorporating Unpaired Data}

\author{
    Ryumei Nakada\footnote{Rutgers University. Email: \href{mailto:rn375@rutgers.edu}{rn375@rutgers.edu}.}
    \and
    Halil Ibrahim Gulluk\footnote{Stanford University. Email: \href{mailto:gulluk@stanford.edu}{gulluk@stanford.edu}.}
    \and
    Zhun Deng\footnote{Columbia University. Email: \href{mailto:zhundeng@g.harvard.edu}{zhundeng@g.harvard.edu}.}
    \and
    Wenlong Ji\footnote{Stanford University. Email: \href{mailto:jwl2000@stanford.edu}{jwl2000@stanford.edu}.}
    \and
    James Zou\footnote{Stanford University. Email: \href{mailto:jamesz@stanford.edu}{jamesz@stanford.edu}.}\,\;\footnotemark[7]
    \and
    Linjun Zhang\footnote{Rutgers University. Email: \href{mailto:lz412@stat.rutgers.edu}{lz412@stat.rutgers.edu}.}
    \footnote{Corresponding authors.}
}

\begin{document}
\maketitle

\begin{abstract}
    Language-supervised vision models have recently attracted great attention in computer vision.
    A common approach to build such models is to use contrastive learning on paired data across the two modalities, as exemplified by Contrastive Language-Image Pre-Training (CLIP).
    In this paper, under linear representation settings, (i) we initiate the investigation of a general class of nonlinear loss functions for multimodal contrastive learning (MMCL) including CLIP loss and show its connection to singular value decomposition (SVD). 
    Namely, we show that each step of loss minimization by gradient descent can be seen as performing SVD on a contrastive cross-covariance matrix.
    Based on this insight, (ii) we analyze the performance of MMCL.
    We quantitatively show that the feature learning ability of MMCL can be better than that of unimodal contrastive learning applied to each modality even under the presence of wrongly matched pairs. This characterizes the robustness of MMCL to noisy data.
    Furthermore, when we have access to additional unpaired data, (iii) we propose a new MMCL loss that incorporates additional unpaired datasets. We show that the algorithm can detect the ground-truth pairs and improve performance by fully exploiting unpaired datasets. The performance of the proposed algorithm was verified by numerical experiments.
\end{abstract}

\section{Introduction}
%\vspace{-1em}
Multimodal learning is a broad class of machine learning algorithms that take advantage of the association of multiple modalities such as text, image, and audio. As the technology of both data collection and sensors advances, we have growing access to data with multiple modes. It has a wide range of applications, including media description, stock return prediction, and drug discovery \citep{baltruvsaitis2018multimodal,lee2020multimodal,hu2021learning}.

Focusing on the research of vision-language models,
there have been many breakthroughs in large-scale vision-language pre-training methods \citep{li2019visualbert,lu2019vilbert,tan2019lxmert,li2020oscar,radford2021learning,jia2021scaling,li2022vision,yaometa, du2022survey}.
One of the vision-language models is Contrastive Language-Image Pre-training (CLIP) \citep{radford2021learning}.
Through contrastive loss, CLIP trains dual encoders in the shared representation space by maximizing the similarity of the observed pairs of text and images while minimizing the similarity of the artificially paired data. 
Through the flexibility of its architecture, CLIP successfully achieves outstanding zero-shot learning performance on ImageNet, outperforming other few-shot linear probes of BiT-M, SimCLRv2, and ResNet50 \citep{radford2021learning}.
CLIP and its successors are widely used, for example, in semantic segmentation, image generation from captions, and video summarization \citep{generating2021,narasimhan2021clip,li2021supervision,xu2022multimodal,wang2022clip}.

Despite the great success of multimodal contrastive learning (MMCL), the theoretical understanding of MMCL is still limited.
From the perspective of multimodal learning, 
it has been empirically \citep{ngiam2011multimodal,radford2021learning} and theoretically \citep{zadeh2020foundations,huang2021makes} shown that
the use of multimodal data produces a better representation compared to the use of unimodal data when focused on a single modality.
Namely, \citet{huang2021makes} 
showed that the difference in downstream task performance of multimodal learning using different subset of modalities depends on the term named latent representation quality.
In particular, they showed that multimodal learning using smaller subset of modalities can perform worse than multimodal learning with a full set of modalities under linear representation settings.
However, the feature recovery performance of multimodal learning, as well as the comparison with that of unimodal contrastive learning, have not been considered in previous works.

Additionally, a practical issue in multimodal learning is the problem of noisy pairs; the collected raw data may not be correctly aligned due to an error in the data collection procedure.
For example, \citep{radford2021learning,jia2021scaling} uses a dataset collected from various public sources on the Internet and feeds the collected data directly into the algorithm without cleaning up the noisy association. However, the quality of association between images and word queries used for the search depends on the context of the website, which possibly leads to incorrect alignment of images and words in the collected dataset. 
To avoid this problem, \citet{li2020oscar} proposed OSCAR that detects object tags in images and uses them as anchor points for alignment.
Although the problem has been recognized in the literature, the analysis of the effect of noisy pairs in MMCL has not been addressed.

Furthermore, multimodal learning requires datasets with specified pair information among modalities. However, data collection procedures can be very expensive in practice. If we can combine abundant unpaired dataset in a semi-supervised manner, we can expect to improve the quality of representation learning with lower cost.

The purpose of this paper is to provide insights on the feature learning ability of MMCL. 
We summarize \textbf{our contributions} as follows.
(i) We establish a connection between the general multimodal contrastive loss and the SVD analysis. Namely, assuming that representations are linear,
we show that the gradient descent of minimizing multimodal contrastive loss function is equivalent to the gradient ascent of the SVD objective function with contrastive cross-covariance matrix.
(ii) We analyze the learning capacity of MMCL under linear loss when there are noisy paired data and the pairs are assumed to be one-to-one.
We show that as long as the observed pairs contain an inignorable portion of ground-truth pairs, MMCL can recover the core features with a parametric rate. However, in practice, many-to-many correspondences between modalities are often observed. We showed by real-data analysis that cleaning up pairs by Bipartite Spectral Graph Multi-Partitioning \citep{dhillon2001co} improves the performance of learned representations.
(iii) We propose a method that incorporates unpaired data and improves the performance of MMCL in a linear representation setting.
The theoretical concept was verified by a numerical experiment.

The outline of this paper is as follows. Section \ref{sec: multimodal and SVD} establishes the connection between MMCL and SVD. 
Based on the results of Section \ref{sec: multimodal and SVD}, in Section \ref{sec: results of MMCL}, we provide a theoretical analysis of MMCL using linear loss on feature learning ability.
In Section \ref{sec: results of RMMCL}, we discuss the possible improvement of MMCL when additional unpaired data is available. 
In Section \ref{sec: experiment}, 
we numerically verify the result of Section \ref{sec: results of RMMCL} that the performance of MMCL improves with additional unpaired data. We also conduct a real-data analysis that deals with many-to-many correspondences of multimodal data.
We discuss and conclude our results in Section \ref{sec: discussion}.

\subsection{Related Works}

%\lm{I will cut some of the related works to meet the page limit.}
%\vspace{-1em}
\paragraph{Multimodal Learning}
%Multimodal learning is a general class of learning algorithms that integrates multiple sources of data. 
Due to its applicability and generality, there has been a large amount of literature on multimodal learning since the 1980s \citep{yuhas1989integration}.
Recently, the development of deep learning brought many advances in multimodal representation learning \citep{sun2020tcgm}.
Especially, \citet{ngiam2011multimodal,srivastava2012multimodal} proposed multimodal learning algorithms to obtain joint representations.
Multimodal contrastive representation learning and generative models are also proposed \citep{shi2020relating,yuan2021multimodal,radford2021learning,jia2021scaling}.
The missing modality problem has been addressed by \citet{ma2021smil,ma2021maximum}.
From a theoretical point of view, multimodal learning has been shown to outperform unimodal learning focused on one modality \citep{zadeh2020foundations,subramanian2021multi,huang2021makes}.
For an overview of multimodal learning, see \citet{baltruvsaitis2018multimodal,zhang2020multimodal,xu2022multimodal,liang2022foundations}.

\paragraph{Self-Supervised Contrastive Learning}
%Other closely related line of research is self-supervised contrastive learning (SSCL) for unimodal data. SSCL trains encoders by contrasting artificially generated augumented data. Pairs of augmented semantically similar images are called positive pairs, while pairs of unrelated data are called negative pairs. SSCL trains encoders to maximize the similarity of positive pairs while minimizing the similarity of negative pairs.
Another closely related line of research is unimodal self-supervised contrastive learning (SSCL) for unimodal data. Representation learning has been crucial in modern machine learning \cite{bengio2013representation, burhanpurkar2021scaffolding,zhangdoes,deng2021adversarial,kawaguchi2022understanding}. SSCL is a group of self-supervised learning algorithms that learn representations by contrasting two views generated by data augmentation. It has gained popularity in computer vision, natural language processing, and graph learning \citep{jaiswal2020survey,liu2021self}. 
In particular, \citet{chen2020improved} proposed SimCLR, which uses contrastive loss to train encoders so that they maximize the similarity of similar views generated by data augmentation and minimize the similarity of unrelated views.
%Seen as contrastive learning, CLIP feeds the contrastive loss with the observed pairs as positive samples, whereas unobserved artificially associated pairs as negative samples.
There have been many works on the theoretical guarantee of SSCL \citep{saunshi2019theoretical,wang2020understanding,ash2021investigating,wen2021toward,huang2021towards,ji2021power,tian2022deep,saunshi2022understanding,ye2022freeze}.
%Language pre-training methods that learn from raw text
%\citet{radford2021learning,jia2021scaling} utilize the multimodality of the training data by contrasting each modality and predict whether the given pair is actually the observed pair.
%It trains the encoders based on the pseudo-label which takes 1 for the observed pairs and 0 otherwise.
%
%It is shown that minimizing the contrastive loss results in minimizing the upper bound of downstream loss of the best linear classifier using the learned representation on average \citep{saunshi2019theoretical,ash2021investigating}.
In particular, \citet{wen2021toward} proved that contrastive learning using shallow neural networks with appropriate data augmentation can learn the sparse signal despite the presence of spurious noise.
%\citep{huang2021towards} showed that the generalization guarantee of contrastive learning depends on three factors: alignment, divergence, and concentration.
\citet{tian2022deep} analyzed unimodal contrastive learning and showed that gradient descent applied to nonlinear contrastive loss can be interpreted as gradient ascent of PCA objective function under game-theoretical formulation.
\citet{ji2021power} proved that contrastive learning is equivalent to a variant of PCA under linear loss settings. They also showed the superiority of contrastive learning over autoencoders under constant signal-to-noise regime.
\citet{ko2022revisiting} established a connection between contrastive learning and neighborhood component analysis \citep{goldberger2004neighbourhood} which learns Mahalanobis distance metrics. 
%They also developed new contrastive losses to improve adversarial robustness.

\paragraph{SVD analysis and CCA}
The goal of SVD analysis is to find projections that maximize the variance between two projected variables.
%SVD is a widely used technique in climatology, \TBE due to its simplicity and computational efficiency.
A closely related method is canonical correlation analysis (CCA) \citep{harold1936relations,kettenring1971canonical}. 
In CCA, the goal is to find linear projections such that \textit{correlation} between two projected variables is maximized, so that the learned projections fully exploit the association of two datasets. 
%For comparison, recall that SVD analysis tries to maximize the covariance between two projected variables.
To deal with nonlinear data sets, artificial neural networks were applied to transform data \citep{lai1998canonical,lai1999neural} and kernels were used to allow flexibility in representation space \citep{akaho2006kernel,hardoon2004canonical}. In particular, Deep CCA \citep{andrew2013deep,benton2017deep,wang2016deep} learns nonlinear embeddings using deep neural networks. Deep CCA was shown to identify latent variables shared between multiple modalities \citep{lyu2020nonlinear,lyu2022finite}.
For an overview of CCA-related methods, see \citet{yang2019survey}.
%\wl{a little long, we may save pages from this section.}
%The difference between CCA and SVD anlaysis is that while CCA tries to find two projections that maximizes the correlation between two datasets, SVD tries to maximize the cross-covariance.  %\citep{cherry1996singular}\TBE 
%(Also SVD does not require the invertibility of the covariance matrix of each modality.)
%(When the covariance matrix of each modality is identity, then CCA is identical to SVD analysis.)
%\lmr{mention about the difference between CCA and SVD}

%\subsection{Outline}
%\wl{I suggest modify this subsection and put it at the end of introduction (before contributions)?}

\subsection{Notation}
%\vspace{-0.5em}
For two sequences of positive numbers $(a_k)_k$ and $(b_k)_k$ indexed by $k \in \mathcal{K}$,
we write $a_k\lesssim b_k$ if and only if there exists a constant $C >0$ independent of the index $k$ such that $\sup_{k \in \mathcal{K}} a_k / b_k < C$ holds. 
%We omit the subscript $\mu$ when it is clear from the context. 
We also write $a_k = O(b_k)$ if $a_k \lesssim b_k$ holds and $a_k = \Omega(b_k)$ if $a_k \gtrsim b_k$ holds.
We write $a_k \asymp b_k$ when $a_k \lesssim b_k$ and $a_k \gtrsim b_k$ holds simultaneously. 
For any matrix $A$, let $\|A\|$ and $\|A\|_F$ denote the operator norm and Frobenius norm of $A$, respectively.
$\mathbb{O}_{d,r} \triangleq \{O \in \R^{r\times d} : O^\top O = I_r\}$ is a set of orthogonal matrices of order $d \times r$.
%We use $|A|$ to denote the cardinality of a set $A$. 
For any positive integer $I$, let $[I]=\{1,2,\cdots,I\}$. 
%We use $\{e_i\}_{i=1}^d$ to denote the canonical basis in $d$-dimensional Euclidean space $\mathbb{R}^d$, that is, $e_i$ is the vector whose $i$-th coordinate is $1$ and all the other coordinates are $0$.
We write $a \vee b$ and $a \wedge b$ to denote $\max(a, b)$ and $\min(a, b)$, respectively.
For any matrix $A$, let $P_r(A)$ be the top-$r$ right singular vectors of $A$.
When the right singular vectors are not unique, we choose arbitrary singular vectors.
%since we are only concerned with the projection $P_r(A) P_r(A)^\top$.
For any matrix $A$, let $\lambda_j(A)$ be the $j$-th largest singular value of $A$. Let $\lambda_{\min}(A)$ and $\lambda_{\max}(A)$ be the minimum and maximum singular values of $A$, respectively.
For any mean zero random variables $X$ and $\tilde X$, we define the covariance matrix of $X$ as $\Sigma_X \triangleq \mathbb{E}[X X^\top]$, and the cross-covariance matrix of $X$ and $\tilde X$ as $\Sigma_{X, \tilde X} \triangleq \mathbb{E}[X \tilde X^\top]$.
%Define the correlation matrix of $X$ and $Y$ as $R_{X,Y} \triangleq \Sigma_X^{-1/2} \Sigma_{X,Y} \Sigma_Y^{-1/2}$.
For any square matrix $A$, define its effective rank $r(A)$ as $r(A) = \Tr(A)/\|A\|$.
%For any symmetric positive semi-definite matrix $A$, let its square root $A^{1/2}$ be $A^{1/2} \triangleq P \Lambda^{1/2} P^\top$, where $A = P \Lambda P$ is the eigenvalue decomposition and $\Lambda^{1/2}$ has diagonal entries that are square roots of the corresponding elements of $\Lambda$. % \zhun{suggest using $\mathbb{P}$ and $\mathbb{E}$ for probability and expectation. Otherwise probabilty and projection are the same symbol.}

%\begin{enumerate}
    %\item For a finite set $A$, $|A|$ is the cardinality of $A$.
    %\item For any vector $x$, $\|x\|$ is the usual $\ell_2$ norm of $x$.
    %\item For any matrix $A$, let $(A)_{ij}$ be the $(i, j)$-th element of $A$.
    % $\kappa_z \triangleq \lambda_1(\Sigma_z)/\lambda_r(\Sigma_z)$ and $\kappa_{\tilde z} \triangleq \lambda_1(\Sigma_{\tilde z})/\lambda_r(\Sigma_{\tilde z})$ and $\kappa_{z,\tilde z} \triangleq \lambda_1(R_{z,\tilde z})/\lambda_r(R_{z,\tilde z})$..
%\end{enumerate}

\section{Multimodal Contrastive Learning and SVD}\label{sec: multimodal and SVD}
%\vspace{-1em}
In this section, we establish the connection between MMCL and SVD.
In the following sections, we focus on MMCL with two-modality data.
%For unimodal SSCL, \citet{tian2022deep} showed that the gradient of contrastive loss is equal to the gradient of PCA objective function with cross-covariance. 

Suppose that we have $n$ pairs of observations $\{(x_i, \tilde x_i)\}_{i=1}^n \subset \R^{d_1 + d_2}$, where $x_i\in\R^{d_1}$ and $\tilde{x}_i\in\R^{d_2}$. The multimodal contrastive loss maximizes the similarity of observed pairs, while minimizing the similarity of generated pairs to learn the encoders $g_1: \R^{d_1} \to \R^r$ and $g_2: \R^{d_2} \to \R^r$ that share the same representation space.
%Let $s(x, \tilde x; \theta_1, \theta_2)$ be a function for the similarity of the pair $(x, \tilde x)$, where $\theta_1$ and $\theta_2$ are the parameters of the embeddings $g_1$ and $g_2$, respectively.
As in the previous literature, we adapt \textbf{inner product} of the representation space as a measure of the similarity of two representations for theoretical brevity; Given two encoders $g_1$ and $g_2$ for each modality, we measure the similarity of the pair $(x, \tilde x)$ by $\langle g_1(x), g_2(\tilde x) \rangle$. This inner product measure has been widely used in the literature \citep{he2020momentum,ji2021power,wang2021understanding,radford2021learning,jia2021scaling}.

\if0
\textbf{Inner Product for Similarity Measure} For theoretical brevity, we adapt the inner product as a similarity measure of the pair $(x, \tilde x)$; $s(x, \tilde x) = s(x, \tilde x; g_1, g_2) \triangleq \langle g_1(x), g_2(\tilde x) \rangle$.
Note that the similarity measure $s$ is not necessarily symmetric, i.e., $s(x, \tilde x)$ is not necessarily equal to $s(\tilde x, x)$. \wl{I think we do not need to say this as we will always deal with inner product in this paper?}
\fi

In this paper, we consider \textbf{linear representation settings}, that is, $g_1(x) = G_1 x$ and $g_2(\tilde x) = G_2 \tilde x$ for $G_1 \in \R^{r \times d_1}$ and $G_2 \in \R^{r \times d_2}$.
The linear representation setting has been widely adapted in the machine learning literature \citep{jing2021understanding,tian2021understanding,ji2021power,wu2022sparse,tian2022deep}.
%\linjun{We may add a few more other transfer learning papers here, and say this setting is commonly used in machine learning theories.}

\subsection{Minimizing Nonlinear Loss via Gradient Descent}\label{sec: general loss}

Here, we consider a general class of nonlinear loss functions\footnote{A similar class of loss functions in SSCL is considered in \citet{tian2022deep}, where the similarity is measured for augmented views.}, which includes linearized loss, CLIP loss \citep{radford2021learning} or ALIGN loss \citep{jia2021scaling}. 
Let $\phi, \psi: \R \to \R$ be differentiable and non-decreasing smooth functions.
The non-decreasing property of $\phi$ and $\psi$ ensures that the loss becomes small when the encoders align only with observed pairs.
Define the loss function as follows:
%\vspace{-0.5em}
\begin{align}
    &\mathcal{L}(G_1, G_2) \triangleq \frac{1}{2 C_n} \sum_i \phi\qty(\epsilon \psi(0) + \sum_{j: j \neq i} \psi(s_{ij} - s_{ii}))\label{eq: general loss}\\
    &\quad+ \frac{1}{2 C_n} \sum_i \phi\qty(\epsilon \psi(0) + \sum_{j: j \neq i} \psi(s_{ji} - s_{ii}))
    + R(G_1, G_2),\nonumber
\end{align}
where $s_{ij} \triangleq \langle G_1 x_i, G_2 \tilde x_j\rangle$, $\epsilon \geq 0$, $C_n$ is a normalizing constant depending only on $n$, and $R(G_1, G_2)$ is a sufficiently smooth regularization term. We note that regularization techniques have been widely adapted in unimodal SSCL practice \citep{chen2020simple,he2020momentum,grill2020bootstrap}.
%\zhun{better add more interpretation 1. why non-decreasing functions, and 2. the choice of $\phi$ and $\psi$ for CLIP and ALIGN. We could actually add a remark seperately for these 2 points.}

%The following proposition states that if we optimize the loss $\mathcal{L}$ via gradient descent, then each step is equivalent to maximizing the SVD objective minus the negative regularization term, 
We consider gradient descent as an optimization method under linear representation settings.
The following proposition states that the gradient of loss \eqref{eq: general loss} with respect to $G_k$ equals to the negative gradient of the SVD objective function minus the penalty term.
%Thus, if we ignore the penalty for a moment, and optimize the loss in \ref{eq: general loss} via gradient descent, the search direction of the parameter is exactly the direction that maximizes the SVD objective function
Thus, if we optimize the loss in \eqref{eq: general loss} via gradient descent, the search direction of the parameter is exactly the direction that maximizes the SVD objective function
%\zhun{more interpretation is appreciated. Try to use a few sentences to describe the gradient descent for SVD, also the penalty term does not belong to SVD optimization, right?} 
with the contrastive cross-covariance matrix\footnote{A closely related notion is the contrastive covariance, which is the covariance matrix of data subracted by the covariance matrix of background noise, introduced in \citet{abid2017contrastive}. In the work, authors proposed contrastive principal component analysis, where PCA is applied to contrastive covariance, aiming to eliminate the effect of background noise from the data.}.
A similar result holds for smooth nonlinear representations, which is deferred to Appendix \ref{prop: min nonlinear loss is SVD restatement}.
\if0
\begin{prop}\label{prop: min nonlinear loss is SVD}
    %Let $s_{ij} \triangleq \langle x_i, \tilde x_j \rangle$.
    Consider minimizing the nonlinear loss function $\mathcal{L}$ defined above. Then,
    \begin{align*}
        \frac{\partial \mathcal{L}}{\partial G_k} = -\eval{\frac{\partial \tr(G_1 S(\beta) G_2^\top)}{\partial G_k}}_{\beta=\beta(G_1, G_2)} + \frac{\partial R(G_1, G_2)}{\partial G_k},\ \ \ \ \ \ k \in \{1,2\},
    \end{align*}
    where the contrastive cross-covariance $S(\beta)$ is given by:
    \begin{align*}
        S(\beta) &= \frac{1}{C_n} \sum_{i=1}^n \beta_i x_i \tilde x_i^\top - \frac{1}{C_n} \sum_{i\neq j} \beta_{ij} x_i \tilde x_j^\top,\ \ 
        \beta_{ij} = \frac{\alpha_{ij} + \bar \alpha_{ji}}{2},\ \ \beta_i = \sum_{j: j \neq i} \frac{\alpha_{ij} + \bar \alpha_{ij}}{2},
    \end{align*}
    with
    \begin{align*}
        \alpha_{ij} &= \phi'\qty(\sum_{j': j' \neq i} \psi(s_{ij'}-s_{ii})) \psi'(s_{ij} - s_{ii}),\ \ 
        \bar\alpha_{ij} = \phi'\qty(\sum_{j': j' \neq i} \psi(s_{j'i}-s_{ii})) \psi'(s_{ji} - s_{ii}).
    \end{align*}
    \if0
    Furthermore, if we use $s_{ij} = - \|G_1 x_i - G_2 \tilde x_j\|^2$, then
    \begin{align*}
        \frac{\partial \mathcal{L}}{\partial G_k} = -\eval{\frac{\partial \tr(G_1 S G_2^\top)}{\partial G_k}}_{\beta=\beta(G_1, G_2)} + \frac{\partial R(G_1, G_2)}{\partial G_k},\ \ \ \ \ \ k \in \{1,2\},
    \end{align*}
    \fi
\end{prop}
\fi
\begin{prop}[Informal]\label{prop: min nonlinear loss is SVD}
    %Let $s_{ij} \triangleq \langle x_i, \tilde x_j \rangle$.
    Let $\beta = \beta(G_1, G_2) \triangleq ((\beta_i)_i, (\beta_{ij})_{i,j})$, where $\beta_i$ and $\beta_{ij}$ also depend on the choice of $\phi$, $\psi$, $\epsilon$ and $\nu$.
    Define the contrastive cross-covariance $S(\beta) \triangleq C_n^{-1} \sum_{i=1}^n \beta_i x_i \tilde x_i^\top - C_n^{-1} \sum_{i\neq j} \beta_{ij} x_i \tilde x_j^\top$.
    Consider minimizing the nonlinear loss function $\mathcal{L}$ defined above. Then, for $k \in \{1,2\}$,
    %\vspace{-1em}
    \begin{align*}
        \frac{\partial \mathcal{L}}{\partial G_k} = -\eval{\frac{\partial \tr(G_1 S(\beta) G_2^\top)}{\partial G_k}}_{\beta=\beta(G_1, G_2)} \hspace{-10pt} + \frac{\partial R(G_1, G_2)}{\partial G_k}.
    \end{align*}
    %\vspace{-1em}
\end{prop}
The formula of $\beta_i$ and $\beta_{ij}$ is available in Appendix \ref{prop: min nonlinear loss is SVD restatement}. 
%\zhun{the formula is quite big, could we consider informal statement and ignore specific expression for S or alpha} \lm{Hi Zhun, thanks for the advice! Do you think the above proposition is enough to capture the result?}
%The proof of Proposition \ref{prop: min nonlinear loss is SVD} is deferred to Appendix \ref{proof: prop: min nonlinear loss is SVD}.
In the case of (unimodal) SSCL, it has been shown that gradient descent of minimizing the contrastive loss is equivalent to the gradient ascent of the PCA objective function, where the target matrix to apply PCA is given by the contrastive covariance matrix \citep{tian2022deep}. %We can regard the following theorem \linjun{above proposition?} as an analogy to this result.
We can consider Proposition \ref{prop: min nonlinear loss is SVD} as an analogy to this result.

\begin{rem}
    If $C_n = n(n-1)$ and the loss function is linear, that is, $\phi$ and $\psi$ are identity functions, then $S = (1/n) \sum_i x_i \tilde x_i^\top - 1/(n(n-1)) \sum_{i \neq j} x_i \tilde x_j^\top = 1/(n-1) \sum_i (x_i - \bar x) (\tilde x_i - \bar{\tilde x})^\top$, which is the centered cross-covariance matrix of $x$ and $\tilde x$.
    For InfoNCE loss, $C_n=n$ and $\phi$ and $\psi$ are set to $\phi(x) = \tau \log(x)$ and $\psi(x) = \exp(x/\tau)$ for some $\epsilon \geq 0$, where $\tau > 0$ is the temperature parameter.
    Setting $\epsilon = 1$ gives the CLIP and ALIGN loss.
    % By setting $\phi(x) = x$ and $\psi(x) = x$, we can recover the SVD objective function with centered cross-covariance $n^{-1} \sum_{i \in [n]} (x_i - \bar x) (\tilde x_i - \bar{\tilde x})^\top$.
    %Notice that for the linear loss function, using $s_{ij} = (1/2)\|G_1 x_i - G_2 \tilde x_j\|^2$ gives the same result, since the gradient of $\mathcal{L}$ only depends on $s_{ij}$ through $\partial_{G_k} s_{ij} - s_{ii}$.
\end{rem}

To encourage the encoders to learn diverse features and prevent the collapse of representations, we simultaneously regularize by $\tr(G_1 G_1^\top G_2 G_2^\top)$. A similar penalty has been considered in unimodal SSCL \citep{ji2021power}.
This is especially beneficial when the loss is linear, that is, $\phi$ and $\psi$ are identity functions, since we can easily make the first two terms in \eqref{eq: general loss} very small by choosing large $G_1$ and $G_2$.
For this reason, we consider the regularization term $R(G_1, G_2) = (\rho/2)\|G_1^\top G_2\|_F^2$ for $\rho > 0$. For this regularization, we have the following result, which directly follows from Eckart-Young-Mirsky theorem.
\begin{lem}\label{lem: EYM 2}
    Fix any $A \in \R^{d_1 \times d_2}$ and $\rho > 0$. Let the SVD of $A$ be $\sum_{j=1}^d c_j U_{1,j} U_{2,j}^\top$, where $d$ is the rank of the sum, $c_1 \geq c_2 \geq \dots \geq c_d > 0$ and $(U_{1,1}, \dots, U_{1,d}), (U_{2,1}, \dots, U_{2,d}) \in \mathbb{O}_{r,d}$. Then,
    \begin{align}
        &\biggl\{ (G_1, G_2) \in \R^{r\times d_1} \times \R^{r \times d_2} : G_1^\top G_2 = \frac{1}{\rho} \sum_{j=1}^r c_j U_{1,j} U_{2,j}^\top \biggr\}\nonumber\\
        &\quad= \hspace{-15pt} \argmax_{G_1 \in \R^{r\times d_1}, G_2 \in \R^{r \times d_2}} \hspace{-10pt} \tr(G_1 A G_2^\top) - (\rho/2) \|G_1^\top G_2\|_F^2.\label{eq: SVD objective}
    \end{align}
\end{lem}
%\vspace{-1em}
In particular, the right singular vectors of $G_1$ and $G_2$ are uniquely determined (up to orthogonal transformation) independent of the choice of $\rho > 0$.

Using Lemma \ref{lem: EYM 2}, Proposition \ref{prop: min nonlinear loss is SVD} implies that, at each step of gradient descent during minimization of the of the regularized CLIP loss, the increment of parameter is in the direction of top-$r$ singular vectors of $S$. Therefore, our result shows the equivalence between the minimization of the loss function \ref{eq: general loss} through gradient descent and top-$r$ SVD with cross-covariance matrix.
%\wl{refer the SVD objective function here.}
\if0
\begin{rem}
    If we smoothly parameterize $G_k$ as $G_k(\theta_k)$ so that $G_k(\theta_k) G_k(\theta_k)^\top = I_r$, then 
    \begin{align*}
        \frac{\partial \mathcal{L}}{\partial \theta_k} = -\eval{\frac{\partial \tr(G_1(\theta_1) S G_2(\theta_2)^\top)}{\partial \theta_k}}_{\alpha=\alpha(\theta)}.
    \end{align*}
\end{rem}
\fi

%\zhun{may put this remark to the place I mentioned, more related}
%\section{Feature Recovery Ability of Multimodal Contrastive Learning}\label{sec: results of MMCL}

\section{Robustness of Multimodal Contrastive Learning to Data Noise}\label{sec: results of MMCL}
%\section{Feature Recovery Ability of Multimodal Contrastive Learning}\label{sec: results of MMCL}
%\vspace{-0.5em}
%\wl{I suggest say feature recovery ability of multimodal contrastive learning to make the title more specific}
%In this section, we derive the feature recovery ability when we minimize the contrastive loss function in linear representation settings.
In this section, we investigate the robustness of MMCL against noisy pairs.
%We simply use the $n$ pairs $(x_i, \tilde x_i)_{i=1}^n$ as positive samples and use the $n(n-1)$ pairs $(x_i, \tilde x_j)_{i \neq j}$ as negative samples.
We analyze the following linear loss, which is the loss function in \ref{eq: general loss} with $\phi(x) = x$, $\psi(x) = x$ and $C_n = n(n-1)$.
\begin{align}
    \mathcal{L}(G_1, G_2) &= \frac{1}{n(n-1)} \sum_{j \neq i} (s_{ij} - s_{ii}) + R(G_1, G_2).\label{eq: linear loss}
\end{align}
Note that this loss function can be rewritten as $\tr(G_1 \bar S G_2) + R(G_1, G_2)$, where $\bar S \triangleq (n-1)^{-1} \sum_i (x_i - \bar x) (\tilde x_i - \bar{\tilde x})^\top$ and thus in this case the minimizer of the loss function is \textit{exactly} the maximizer of the SVD objective function $\tr(G_1 \bar S G_2) - R(G_1, G_2)$.

The linear loss function for analyzing representation learning has been used in metric learning \citep{schroff2015facenet,he2018triplet}, contrastive learning \citep{ji2021power} and MMCL \citep{won2021multimodal,alsan2021multimodal}.
Analysis on MMCL using InfoNCE loss is deferred to the Appendix \ref{sec: MMCL with InfoNCE}.

\subsection{Data Generating Process}\label{sec: data generation}
%\vspace{-0.5em}
%The settings in this subsection is different from the settings in main.tex.
For each modality, we consider the spiked covariance model \citep{bai2012sample,yao2015sample,zhang2018heteroskedastic,zeng2019double,ji2021power} as the data generation process.
Suppose that we have $n$ observed pairs $\{x_i\}_{i \in [n]}$ and $\{\tilde x_i\}_{i \in [n]}$ drawn from the following model:
\begin{align}
    x_i = U_1^* z_i + \xi_i,\ \ &\tilde x_i = U_2^* \tilde z_i + \tilde \xi_i,\label{model: multimodal}\\
    z_i = \Sigma_z^{1/2} w_i, \ \ \tilde z_i = \Sigma_{\tilde z}^{1/2} \tilde w_i, \ \ &\xi_i = \Sigma_{\xi}^{1/2} \zeta_i, \ \ \tilde \xi_i = \Sigma_{\tilde \xi}^{1/2} \tilde \zeta_i,\nonumber
\end{align}\noindent
where $w_i$, $\tilde w_i$, $\zeta_i$, and $\tilde \zeta_i$ have i.i.d. coordinates, each of which follows sub-Gaussian distribution with parameter $\sigma$ and unit variance.
Notice that in model \ref{model: multimodal}, $U_1^*$ and $z_i$ are only identifiable up to orthogonal transformation.
%; we can only identify the set $\{(U_1^* O, O^\top z_i) : O \in \mathbb{O}_{r,r}\}$. Thus, by taking $O$ as the eigenvectors of $\Sigma_z$ and substituting $U_1^* \leftarrow U_1^* O$ and $z_i \leftarrow O^\top z_i$, 
Thus, we can assume that $\Sigma_z$ is a diagonal matrix with $(\Sigma_z)_{1,1} \geq (\Sigma_z)_{2,2} \geq \dots \geq (\Sigma_z)_{r,r}$ without loss of generality.
A similar argument holds for $U_2^*$ and $\tilde z_i$,  and we assume that $\Sigma_{\tilde z}$ is also a diagonal matrix with $(\Sigma_{\tilde z})_{1,1} \geq (\Sigma_{\tilde z})_{2,2} \geq \dots \geq (\Sigma_{\tilde z})_{r,r}$.
Furthermore,  without loss of generality, we assume $\|\Sigma_z\| = \|\Sigma_{\tilde z}\| = 1$.

%\zhun{is this necessary? If not we can say Wlog}

%Since we are concerned with multimodal data, we additionally assume the following one-to-one matchings between two modalities.
%Assume that for any $i$, $w_i = \tilde w_i$ while $\xi_i$ and $\xi_i$ are independent. For $i \neq j$, assume the independence between $w_i$ and $\tilde w_j$, and between $\xi_i$ and $\xi_j$.

Since the data are multimodal, we additionally assume the following (noisy) matches between two modalities.
Recall that we have $n$ observed pairs $(x_1, \tilde x_1), \dots, (x_n, \tilde x_n)$. 
Define the set of observed indices as $\mathcal{C} \triangleq \{(1, 1), \dots, (n, n)\}$.
Let $\mathcal{E} \subset [n] \times [n]$ be the set of $n$ pairs.
%, and let its complement denote $\mathcal{E}^\perp = [n] \times [n] \setminus \mathcal{E}$.
For the pairs $(i_1, j_1) \in \mathcal{E}$, assume that $w_{i_1} = \tilde w_{j_1}$ while $\xi_{i_1}$ and $\tilde\xi_{j_1}$ are independent. For pairs $(i_1, j_1) \in [n]^2 \setminus \mathcal{E}$, assume the independence between $w_{i_1}$ and $\tilde w_{j_1}$, and between $\xi_{i_1}$ and $\tilde \xi_{j_1}$.
We note that we can regard the set of pairs as the subset of the edges of the bipartite graph $\{(i, j) : i, j \in [n]\}$. Hereafter, we sometimes call the pairs in $\mathcal{E}$ ground truth edges and the pairs in $\mathcal{C}$ observed edges.

Let $m \triangleq |\mathcal{C} \cap \mathcal{E}| \in \{0,1,\dots,n\}$ be the number of observed ground-truth edges, and we define $p_n = 1 - m / n \in [0, 1]$ as the \textit{distortion rate} of the bipartite graph. 
%In short, the number of unobserved ground-truth edges is $n-m$. 
When $p_n$ is small, the information of association in collected data is highly reliable, while when $p_n$ is large, the data contains many noisy pairs, which do not have the valid information between each modality.

We measure the ``goodness" of pre-trained encoders by the quality of right-singular vectors of the encoders, since the fine-tuned predictors in downstream tasks only depend on the right-singular vectors in many cases. To see this, we decompose $G_1 \in \R^{r\times d}$ by SVD as $G_1 = V C U^\top$, where $U \in \mathbb{O}_{r,d}$, $V \in \mathbb{O}_{r,r}$ and $C$ is diagonal.
Suppose that we have a sample $(y, x) \in \R^{1+d}$ in the downstream task with some metric $D$.
Through fine-tuning, we obtain $f^* = \argmin_{f \in \mathcal{F}} D(y, f(G_1 x))$. 
For linear benchmarks, $\mathcal{F} = \{f: f(z) = w^\top z, w \in \R^r\}$ and $f^*$ does not depend on $V$ and $C$.
This also holds for neural networks with a similar argument.
To measure the quality of $P_r(G_1)$ (or $P_r(G_2)$, we employ $\sin\Theta$ distance; For two orthogonal matrices $U_1,U_2\in\mathbb{O}_{d, r}$, the distance is defined as $\|\sin \Theta\left(U_{1}, U_{2}\right)\|_F \triangleq \|U_{1 \perp}^\top U_{2}\|_F$ for any orthogonal compliment $U_{1 \perp}$ of $U_1$.

\subsection{Analysis on Multimodal Contrastive Loss Function}\label{sec: analysis of MMCL}

Before formalizing the previous argument, we introduce several assumptions.
\begin{asm}\label{asm: signal condition number}
    Assume that the condition numbers of $\Sigma_z$ and $\Sigma_{\tilde z}$ are bounded; $\|\Sigma_z\|/\lambda_{\min}(\Sigma_z) \leq \kappa_z^2$ and $\|\Sigma_{\tilde z}\|/\lambda_{\min}(\Sigma_{\tilde z}) \leq \kappa_{\tilde z}^2$ for some constants $\kappa_z^2, \kappa_{\tilde z}^2 > 0$.
\end{asm}
\begin{asm}\label{asm: signal-to-noise ratio}
    Assume that the signal-to-noise ratio is bounded below; $\|\Sigma_z\|/\|\Sigma_{\xi}\| \geq s_1^2$ and $\|\Sigma_{\tilde z}\|/\|\Sigma_{\tilde \xi}\| \geq s_2^2$ for some constants $s_1^2, s_2^2 > 0$.
\end{asm}
%Assumptions \ref{asm: signal condition number} and \ref{asm: signal-to-noise ratio} are assumed for theoretical clarity.
%Although it is possible to relax Assumptions \ref{asm: signal condition number} and \ref{asm: signal-to-noise ratio}, they are assumed for the purpose of presentation. \linjun{not convincing to me. Can we simply cite several papers and say these assumptions are standard in the machine learning theory litearture? }
Assumption \ref{asm: signal condition number} imposes regularity conditions on covariance matrices, and Assumption \ref{asm: signal-to-noise ratio} assumes the signal-to-noise ratio is not too small. %is assumed in the analysis of PCA \citep{yan2021inference}
%Also, a similar assumption to Assumption \ref{asm: signal-to-noise ratio} is assumed in the analysis of contrastive learning \citep{ji2021power,wen2021toward}.
Similar assumptions have been commonly used in the machine learning theory literature, e.g., \citet{cai2019chime, yan2021inference,cai2021cost, ji2021power,wen2021toward}.

For this setting, we have the following result.
\begin{thm}\label{thm: linear loss}
    Suppose that we have a collection of pairs $(x_i, \tilde x_i)_{i=1}^{n}$ generated according to the model \ref{model: multimodal}.
    Suppose Assumptions \ref{asm: signal condition number} and \ref{asm: signal-to-noise ratio} hold.
    Let $G_1$ and $G_2$ be the solution to minimizing the loss \ref{eq: linear loss}.
    If $p_n \leq 1 - \eta$ for some constant $\eta > 0$, then,
    %Also set $\rho^2 = 1$.
    with probability $1 - O(n^{-1})$, we have
    \begin{align*}
        &\|\sin\Theta(P_r(G_1), U_1^*)\|_F \vee \|\sin\Theta(P_r(G_2), U_2^*)\|_F\\
        &\lesssim \sqrt{r} \wedge \frac{1}{\eta} \sqrt{\frac{r (r + r(\Sigma_\xi) + r(\Sigma_{\tilde \xi})) \log (n+d_1+d_2)}{n}}.
    \end{align*}
\end{thm}
From Theorem \ref{thm: linear loss}, as long as $1 - p_n$ is strictly bounded away from $0$, the feature recovery ability attains square root convergence. In other words,
%as long as unignorable portion of data $1 - p_n$ is bounded above strictly less than 1, the feature recovery ability is of order $\sqrt{r/n}$.
MMCL can learn representations even in the presence of noisy pairs whenever there are inignorable portion of observed ground-truth pairs. The case where $p_n \uparrow 1$ is treated in Section \ref{sec: thm: linear loss} in the appendix.

%with the same parametric convergence rate.

%Based on Theorem \ref{thm: linear loss}, we can quantitatively show how MMCL outperforms unimodal contrastive learning applied separately to each modality.
%For theoretical comparison of unimodal learning against multimodal learning,
%\citet{zadeh2020foundations} showed that the information entropy of multimodal learning is the lower bound to the information entropy of unimodal learning.
%\citet{huang2021makes} proved that the downstream task performance of multimodal learning using smaller subset of modatilies,

In the following, we compare the performance of MMCL with (unimodal) SSCL applied to each modality separately.
SSCL aims to learn representations by contrasting pairs generated by data augmentation \citep{jaiswal2020survey,liu2021self,yao2022improving}. In particular, we consider SSCL similar to SimCLR \citep{chen2020improved}, where encoders are trained to have similar representations for augmented views from the same sample and to discriminate augmented views from different samples.

Consider minimizing the unimodal linear contrastive loss using the first-mode data $\{x_i\}_{i=1}^n$.
Let $A$ be the \textit{random masking augmentation} defined as $A = \diag(a_1, \dots, a_{d_1})$, where $a_i$ follows i.i.d. $\textrm{Ber}(1/2)$ distribution. Given $A$, positive pairs are generated as $(A x_i, (I-A) x_i)_{i=1}^n$. 
Let $\mathcal{L}_{c}(G_1) \triangleq \mathcal{L}(G_1, G_1; (x_i, \tilde x_i)_{i \in [n]})$ be the linear loss in \ref{eq: linear loss} fed with generated positive pairs.
Let $G_1^c$ be the solution to minimizing the expected loss $\mathbb{E}_A[\mathcal{L}_{c}]$, where expectation is taken with respect to the data augmentation $A$.
\if0
\begin{dfn}[Random Masking Augmentation]
	\label{aug: random masking}
	The two views of the original data are generated by randomly dividing its dimensions to two sets, that is,
	%\begin{equation*}
		$g_1(x_i)=Ax_i,\text{~and~} g_2(x_i)=(I-A)x_i$, 
	%\end{equation*}
	where $A=\diag(a_1,\cdots,a_d)\in\mathbb{R}^{d\times d}$ is the diagonal masking matrix with $\{a_i\}_{i=1}^{d}$ being $i.i.d.$ random variables sampled from a Bernoulli distribution with mean $1/2$. 
\end{dfn}
\fi
For the learned representation, \citet{ji2021power} showed that when $\Sigma_{\xi}$ and $\Sigma_{\tilde \xi}$ are bounded above, then
\begin{align*}
	\mathbb{E}\left\|\sin \Theta\qty(U_1^\star, P_r(G_1^u))\right\|_F \lesssim\frac{r^{3/2}}{d}\log d+\sqrt{\frac{dr}{n}}. 
\end{align*}
The detailed statement is available in Appendix \ref{sec: omitted}.
Note that the assumption that the condition numbers of $\Sigma_\xi$ and $\Sigma_{\tilde \xi}$ are bounded above implies that $r(\Sigma_\xi) \gtrsim d_1$ and $r(\Sigma_{\tilde \xi}) \gtrsim d_1$. Ignoring the logarithmic term, and provided that $d_1 \asymp d_2$, we notice that the bound in Theorem \ref{thm: linear loss} improves the rate by reducing the bias term $r^{3/2} \log d / d$, while the variance term remains the same.
%Thus the variance bound is almost the same ignoring the logarithm term.
The bias term is due to the fact that core feature $U_1^*$ loses its information when the random masking data augmentation is applied to the original data.
%Also, note that our result \ref{thm: feature recovery via MMCL} does not require the incoherent constant assumption, because we can separate core features from noise using the fact that the core features are highly correlated, while noises are not correlated between two modalities.
%the improvement of bias term is due to our sharper bound of the rate of convergence for cross-covariance matrix.
\citet{ji2021power} also provably showed that when the noise covariance shows strong heteroskedasticity, the feature recovery performance of the representations obtained by autoencoders stays constant, while contrastive learning can mitigate the effect of heteroskedasticity. Therefore, under strong heteroskedasticity, MMCL can learn representations better than autoencoders applied to each modality separately.

%\vspace{-0.5em}
%Define $\bar S = (1/n) \sum_i x_i \tilde x_i^\top$.
%As discussed in Section \ref{sec: general loss}, we add a regularization term $R(G_1, G_2) = (\rho^2/2) \|G_1^\top G_2\|_F^2 + (\rho_2^2/2) \|G_1^\top G_2 - \bar S\|_F^2$ to the loss function in \eqref{eq: linear loss}, where $\rho$ and $\rho_2$ are some constants. \wl{Yeah, we can just move this sentence to the beginning of section 3} 
%The second penalty term encourages that the similarity measure becomes closer to the inner product $\langle x, \tilde x \rangle_{\bar S}$, which is a technical assumption to analyze the problem. Although we will see later, the addition of the second penalty strengthens the \textit{signal strength} of the contrastive cross-covariance matrix.
%\wl{According to our previous experience, the regularization would be challenged by the reviewer. I think we can explain the meaning of each term and say that is a technical assumption to make the problem mathematically trackable.}
%Consider minimizing \eqref{eq: linear loss} with regularization term
%$R(G_1, G_2) = (\rho^2/2) \|G_1^\top G_2\|_F^2 + (\rho_2^2/2)\|G_1^\top G_2 - \bar S\|_F^2$, where
%The reason of adding the term $(\rho_2^2/2)\|G_1^\top G_2 - \bar S\|_F^2$ will be discussed later in this section.

\subsection{Extension to Many-to-Many Correspondence Case}\label{sec: experiment many-to-many summary}

For the analysis in Section \ref{sec: analysis of MMCL}, we assumed one-to-one matches between the modalities and showed that MMCL can learn representations regardless of the distortion rate as long as there is an inignorable portion of ground-truth pairs in observed pairs.
However, in practice, the performance of MMCL can be improved by eliminating noisy pairs. In addition, we face a multimodal dataset with many-to-many correspondence.
To detect noisy pairs in the many-to-many correspondence case, we employ the Bipartite Spectral Graph Multi-Partitioning algorithm (BSGMP) \citep{dhillon2001co}.
BSGMP is a generalization of the spectral graph partitioning algorithm by which we can detect and eliminate wrongly aligned edges in a bipartite graph that we expect to have a clustered shape.

Consider applying BSGMP first to the dataset generated by matching the MNIST and Fashion-MNIST datasets using labels. Then, we perform MMCL with InfoNCE loss defined as the loss function in \eqref{eq: general loss} with $\phi(x) = \tau \log(x)$ and $\psi(x) = \exp(x/\tau)$. Details of the algorithm and results are deferred to Section \ref{sec: experiment many-to-many}.
\if0
\begin{align*}
    &\mathcal{L}^I(G_1, G_2) \triangleq - \frac{1}{2N} \sum_{(i, j) \in \mathcal{E}^u} \log \frac{e^{s_{ii}^u / \tau}}{\sum_{j \in [N]} e^{s_{ji}^u/\tau}}\\
    &\quad- \frac{1}{2N} \sum_{(i, j) \in \mathcal{E}^u} \log \frac{e^{s_{ii}^u / \tau}}{\sum_{j \in [N]} e^{s_{ij}^u/\tau}} + R(G_1, G_2).
\end{align*}\label{eq: InfoNCE}
\fi

Figure \ref{fig: improvement by BSGMP} shows that if we apply BSGMP with parameter $k=10$, the performance of the downstream prediction task improves with moderate distortion rate.
\begin{figure}[H]
    \centering
    %\vspace{-3em}
    \includegraphics[width=0.4\textwidth]{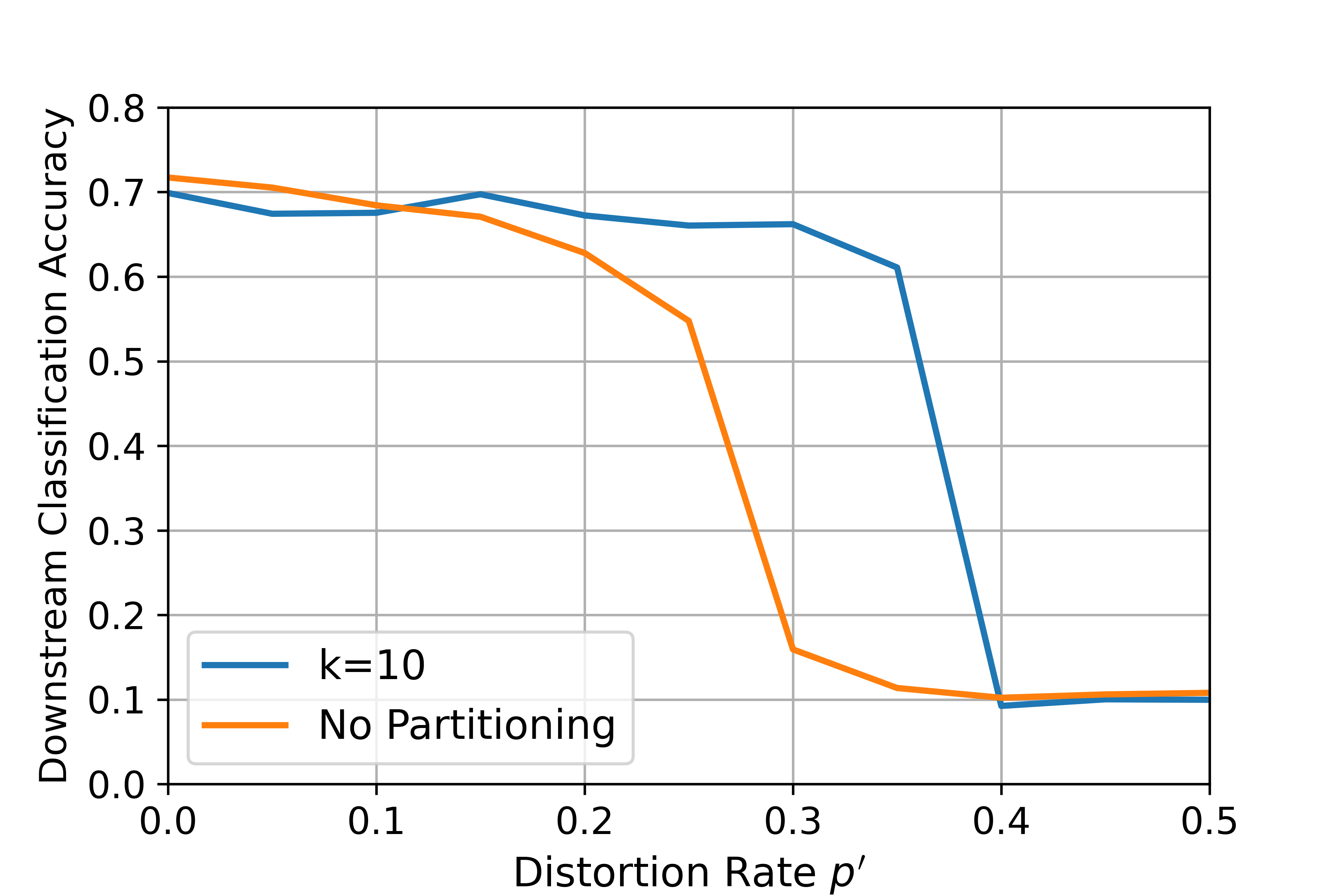}
    \caption{The downstream task performance of MMCL versus the distortion rate $p'$. The orange curve indicates MMCL without BSGMP, whereas the blue curve indicates MMCL with BSGMP with parameter $k=10$.}
    %the \james{Is the distortion rate $p_n$ or $p'$? Also explain the orange curve}}
    \label{fig: improvement by BSGMP}
    %\vspace{-3em}
\end{figure}

%\section{Robust Multimodal Contrastive Loss against Messy Pairings}\label{sec: results of RMMCL}
\section{Improving Multimodal Learning by Incorporating Unpaired Data}\label{sec: results of RMMCL}
%\vspace{-0.5em}

\if0
\lm{
Since the following loss function does not work,
\begin{align*}
    - \frac{1}{N} \sum_{i\in [N]} \log \frac{1}{1 + \sum_{j\in [N]} e^{s_{ij}/\tau}} + R(G_1, G_2),
\end{align*}
we may need to change the loss function to
\begin{tiny}
\begin{align}
    &- \frac{1}{N} \sum_{(i, j) \in \mathcal{\hat E}^u} s_{ij} / \tau + \frac{1}{2N} \sum_{i \in [N]} \log  \sum_{j \in [N]} e^{s_{ij}/\tau} + \frac{1}{2N} \sum_{i \in [N]} \log  \sum_{j \in [N]} e^{s_{ji}/\tau}\\
    &\quad+ R(G_1, G_2),
\end{align}
\end{tiny}
where $\mathcal{\hat E}^u \subset [N]^2$ is the set of the estimated ground truth edges defined as
\begin{align*}
    \mathcal{\hat E}^u \triangleq \{(i, j) \in [N]^2 : s_{ij} \geq s_N\},
\end{align*}
where $s_N$ is the $N$-th largest element in $\{s_{ij}\}_{i,j \in [N]}$.
%\begin{align*}
%    \tilde k(i) \triangleq \argmax_{j \in [N]} s_{ij}, \ \ k(i) \triangleq \argmax_{j \in [N]} s_{ji}
%\end{align*}
The loss function above is an analogy to the original CLIP loss
\begin{tiny}
\begin{align*}
    &- \frac{1}{2N} \sum_{i \in [N]} \log \frac{e^{s_{ii}/\tau}}{\sum_{j \in [N]} e^{s_{ij}/\tau}} - \frac{1}{2N} \sum_{i \in [N]} \log \frac{e^{s_{ii}/\tau}}{\sum_{j \in [N]} e^{s_{ji}/\tau}} + R(G_1, G_2)\\
    &= - \frac{1}{N} \sum_{i \in [N]} s_{ii/\tau} + \frac{1}{2N} \sum_{i \in [N]} \log  \sum_{j \in [N]} e^{s_{ij}/\tau} + \frac{1}{2N} \sum_{i \in [N]} \log  \sum_{j \in [N]} e^{s_{ij}/\tau}\\
    &\quad+ R(G_1, G_2),\nonumber
\end{align*}
\end{tiny}
}
\fi

In this section, we propose a modification of %multimodal contrastive 
CLIP loss to incorporate additional unpaired data and investigate its theoretical property. Due to the abundance of unpaired data in practice, this would greatly improve multimodal learning with the paired training data is limited.

%A related line of research is \citet{shi2020relating}. In the work, authors proposed a method that effectively incorporates unlabeled data in multimodal generative models by introducing a ralatedness variable following Bernoulli distribution. In the work, authors proposed a loss function that trains generative models by maximizing the pointwise mutual information (PMI) of related pairs while minimizing PMI of unrelated pairs.
Since MMCL projects the data into the shared representation space, we can explicitly calculate the similarity of any pair using given initial representations. This information, in turn, can be used to test whether a new pair is actually associated or not.
More specifically, we consider a setting in which we have access to both the paired dataset $(x_i, \tilde x_i)_{i=1}^n$ and the unpaired datasets $(x_i^u)_{i=1}^N$ and $(\tilde x_i^u)_{i=1}^N$. 
%Since we can regard the pairing information as label information, we refer to this setting as semi-supervised setting.
Since we can regard the information on association between modalities as labels of pairs, we refer to this setting as a semi-supervised setting and call additional datasets unpaired data.
%Since we already have labeled data, we can calculate similarity scores for unpaired data using the representations obtained by the labeled data. These similarity scores among all possible pairs constructed from unpaired data can, in turn, be used to detect the ground-truth edges.
We define the data generation process as follows.
Suppose that the paired data $(x_i, \tilde x_i)_{i=1}^n$ are generated according to the model \eqref{model: multimodal} described in Section \ref{sec: multimodal and SVD}. 
%Suppose that we also have access to additional unpaired datasets  whose association is unknown.
For the unpaired datasets $(x_i^u)_{i=1}^N$ and $(\tilde x_i^u)_{i=1}^N$, 
we assume the same spiked covariance model.
%In order to show that our method can detect the ground-truth association,
We further assume the matches between two modalities as in Section \ref{sec: analysis of MMCL}.
%We additionally assume that each node in both modalities is matched with more than one nodes in the other modality. That is, for any $x_i^u$, we can find $x_{j_1}^u$ and $x_{j_2}^u$ with $j_1 \neq j_2$ such that $w_i = \tilde w_{j_1} = \tilde w_{j_2}$, and $w_i$ are independent of $w_j$ for all $j \not\in \{i,j_1,j_2\}$. The noise $\xi_i$ and $\tilde \xi_j$ are assumed to be independent for all $i, j \in [N]$.
%while $\xi_i$ and $\xi_j$ are independent and for the pairs $(i, j) \in [N]^2 \setminus \mathcal{E}^u$, assume independence between $w_i$ and $w_j$, and between $\xi_i$ and $\xi_j$.
%in order to show that our method can detect the ground-truth association,
%we assume the matches between two modalities.
%as in Section \ref{sec: analysis of MMCL}, 
%we assume the same matches between two modalities as described in Section \ref{sec: analysis of MMCL}. 
To avoid confusion of notation, let $\mathcal{E}^u \subset [N]^2$ denote the set of $N$ ground-truth pairs for unpaired data.
%, and let its complement denote $\mathcal{E}^\perp = [N] \times [N] \setminus \mathcal{E}$.
%Recall that for the pairs $(i, j) \in \mathcal{E}^u$, $w_i = \tilde w_j$ while $\xi_i$ and $\xi_j$ are independent and for the pairs $(i, j) \in [N]^2 \setminus \mathcal{E}^u$, assume the independence between $w_i$ and $w_j$, and between $\xi_i$ and $\xi_j$.
%We note that we can regard the set of pairs as the subset of the edges of the bipartite graph $\{(i, j) : i, j \in [N]\}$. 
%Hereafter, we call the pairs in $\mathcal{E}^u$ ground truth edges.
%\wl{Should have some explanation on the reason of this modification. It seems that this loss will uniformly penalize on similarity between any pairs and why it makes sense?}
We continue to use the linear representation settings.
Define the similarity between $x_i^u$ and $\tilde x_j^u$ as $s_{ij}^u = s_{ij}^u(G_1, G_2) \triangleq \langle G_1 x_i^u, G_2 \tilde x_j^u \rangle$.
%Suppose that we have a set of pairs $\mathcal{\bar E}^u$.
%using paired dataset $(x_i, \tilde x_i)_{i=1}^n$.
We consider the following loss function $\mathcal{L}^u = \mathcal{L}^u(G_1, G_2; \mathcal{\bar E}^u)$ with respect to any set of pairs $\mathcal{\bar E}^u \subset [N]^2$ to incorporate the unpaired data:
\begin{align}
    \mathcal{L}^u &\triangleq - \frac{\nu}{N} \sum_{(i, j) \in \mathcal{\bar E}^u} s^u_{ij} + \frac{\tau}{2N} \sum_{i \in [N]} \log \sum_{j \in [N]} e^{s^u_{ij}/\tau} + \frac{\tau}{2N} \sum_{i \in [N]} \log\sum_{j \in [N]} e^{s^u_{ji}/\tau} + R(G_1, G_2),\label{eq: RMMCL}
\end{align}
%where $\nu_{ij} = \nu \geq 1$ if $(i,j) \in \mathcal{\bar E}^u$, otherwise $\nu_{ij} = 1$.
where $\nu \geq 1$.
%where $\tilde k(i), k(j)$ are the nodes such that $(i,\tilde k(i)), (k(j), j) \in \mathcal{E}^u$.
Note that this loss is exactly the InfoNCE loss function when $\mathcal{\bar E}^u = \{(1, 1), \dots, (N, N)\}$, $\epsilon=1$ and $\nu=1$. Setting $\nu > 1$ corresponds to choosing different temperature parameters for the similarity of positive pairs and negative pairs.

Given the loss function, we propose the semi-supervised MMCL in Algorithm \ref{alg: MMCL Unlabeled}.
\begin{algorithm}
   \caption{Semi-supervised MMCL}\label{alg: MMCL Unlabeled}
    \begin{algorithmic}
        \State \textbf{Input:} Data $(x_i)_{i \in [n]}$ and $(\tilde x_i)_{i \in [n]}$, rank $r \geq 1$, parameters $\tau, \nu > 0$.
        \State Obtain the initial representations $G_1^{(0)}$ and $G_2^{(0)}$ from the paired dataset $(x_i, \tilde x_i)_{i\in[n]}$ by minimizing InfoNCE loss given by the loss in \eqref{eq: linear loss} with $\phi(x) = \tau \log (1+x)$ and $\psi(x) = e^{x/\tau}$. Calculate the similarity of pairs by $s_{ij}^u = \langle G_1^{(0)} x_i^u, G_2^{(0)} \tilde x_j^u \rangle$.
        \State Estimate the set of ground truth pairs $\mathcal{E}^u$ by
        \begin{align}
            \mathcal{\hat E}^u &\triangleq \{(i, j) \in [N]^2: s^u_{ij} \geq s^u_{(N)}\},\label{eq: E hat}
        \end{align}
        where $s^u_{(N)}$ is the $N$-th largest value of $\{s^u_{ij}: i \in [N], j \in \argmax_{j'} s_{ij'}^u \} \cup \{ (i, j): j \in [N], i \in \argmax_{i'} s_{i'j}^u \}$.
        %$
        %    \mathcal{\hat E}^u \triangleq \{(i, j) \in [N]^2 : s^u_{ij} \geq s^u_{(N)}\},
        %$
        %where $s^u_{(N)}$ is the $N$-th largest element in $\{s^u_{ij}\}_{i,j \in [N]}$.
        \State \textbf{Output:} $G_1$ and $G_2$ obtained by minimizing the loss $\mathcal{L}^u(G_1, G_2; \mathcal{\hat E}^u)$.
    \end{algorithmic}
\end{algorithm}
\if0
\begin{align*}
    &\mathcal{L}'(G_1, G_2) \triangleq - \frac{1}{2N} \sum_{(i, j) \in \mathcal{E}^u} \log \frac{e^{s_{ii}^u / \tau}}{\sum_{j \in [N]} e^{s_{ji}^u/\tau}}\\
    &\quad- \frac{1}{2N} \sum_{(i, j) \in \mathcal{E}^u} \log \frac{e^{s_{ii}^u / \tau}}{\sum_{j \in [N]} e^{s_{ij}^u/\tau}} + R(G_1, G_2),
\end{align*}\label{Eq:UnsupLoss}
\fi

\if0
\begin{align}
    - \frac{1}{N} \sum_{i\in [N]} \log \frac{1}{1 + \sum_{j\in [N]} e^{s_{ij}^u/\tau}} + R(G_1, G_2).\label{eq: RMMCL}
\end{align}
\fi
Similar to Proposition \ref{prop: min nonlinear loss is SVD}, we can connect the gradient of the loss function and the negative gradient of the SVD objective.
Define the contrastive cross-covariance matrix given some set of pairs $\mathcal{\bar E}^u$ as $S^u = S^u(\{\beta_{ij}^u\}_{i,j}; \mathcal{\bar E}^u) \triangleq \nu N^{-1} \sum_{(i, j) \in \mathcal{\bar E}^u} x_i \tilde x_j^\top - N^{-1} \sum_{i,j \in [N]} \beta_{ij}^u x_i \tilde x_j^\top$, where the formula for $\beta_{ij}^u$ is available in Appendix \ref{prop: min nonlinear unsupervised loss is SVD restatement}.
Then, for $k \in \{1,2\}$,
\begin{align}
    \frac{\partial \mathcal{L}^u}{\partial G_k} = -\eval{\frac{\partial \tr(G_1 S^u G_2^\top)}{\partial G_k}}_{\beta_{ij}^u=\beta_{ij}^u(G_1, G_2)} \hspace{-27pt} + \frac{\partial R(G_1, G_2)}{\partial G_k}.\label{eq: min nonlinear unsupervised loss is SVD}
\end{align}
Observe that each step of the gradient descent corresponds to the gradient ascent of the SVD objective with the negative cross-covariance matrix $S^u$.

%we have the following result as an analogy to Proposition \ref{prop: min nonlinear loss is SVD}.
\if0
\begin{prop}\label{prop: min nonlinear unsupervised loss is SVD}
    %Let $s_{ij} \triangleq \langle x_i, \tilde x_j \rangle$.
    Define the negative cross-covariance matrix for all possible pairs $S = S(\{\beta_{ij}^u\}_{i,j\in[N]}) \triangleq - N^{-1} \sum_{i=1}^N \beta_{ij}^u x_i^u \tilde x_i^{u \top}$, where
    \begin{align*}
        \beta_{ij}^u = \beta_{ij}^u(G_1, G_2) \triangleq \frac{e^{s_{ij}^u/\tau}}{1 + \sum_{j' \in [N]} e^{s_{ij'}^u/\tau}}.
    \end{align*}
    Consider minimizing the nonlinear loss function $\mathcal{L}$ defined in \eqref{eq: RMMCL}. Then, for $k \in \{1,2\}$,
    
    \begin{align*}
        \frac{\partial \mathcal{L}}{\partial G_k} = -\eval{\frac{\partial \tr(G_1 S(\{\beta_{ij}\}_{i,j}) G_2^\top)}{\partial G_k}}_{\beta_{ij}=\beta_{ij}^u(G_1, G_2)} + \frac{\partial R(G_1, G_2)}{\partial G_k}.
    \end{align*}
\end{prop}
Observe that each step of the gradient descent corresponds to the gradient ascent of the SVD objective with the negative cross-covariance matrix $S(\beta)$.
\fi

Motivated by Lemma \ref{lem: EYM 2}, we consider the following two-step procedure to analyze the performance of Algorithm \ref{alg: MMCL Unlabeled}.
%Suppose we have estimated the set of ground truth pairs $\mathcal{E}^u$ for unpaired dataset by $\mathcal{\hat E}^u \subset [N]^2$.
%\linjun{is it possible to write our algorithm in an algorithm box? Ohterwise, the only algorithm in the box is the Bipartite Spectral Graph Multi-Partitioning, which is not proposed by us. But it's OK if this is not easy to do and we leave it as is. This may require an overhaul of the structure. }\lm{Let me try this. I removed the algorithm because of the page limit.}
%\linjun{If we don't have enough space, we may shorten the discussion of the comparison with the unimodal CL, eg. move the formal statement of Lemma~\ref{thm: feature recovery from CL} to the appendix. No need to introduce the notation and technical conditions for this lemma; we may just mention its convergence rate, and compare it with our result. }
\paragraph{Step 1.} Obtain the initial representations $G_1^{(0)}$ and $G_2^{(0)}$ from the paired dataset $(x_i, \tilde x_i)_{i \in [n]}$ by minimizing the linear loss in \eqref{eq: linear loss}.
%, which is equivalent to applying SVD to the cross-covariance matrix constructed from the labeled data. 
\paragraph{Step 2.} Estimate the set of ground truth pairs $\mathcal{E}^u$ by $\mathcal{\hat E}^u$ as in Algorithm \ref{alg: MMCL Unlabeled}.
%defined as
%$
%    \mathcal{\hat E}^u \triangleq \{(i, j) \in [N]^2 : s^u_{ij} \geq s^u_{(N)}\},
%$
%where $s^u_{(N)}$ is the $N$-th largest element in $\{s^u_{ij}\}_{i,j \in [N]}$ and $\{s^u_{ij}\}_{i,j \in [N]}$ is calculated with the initial representations obtained in Step 1. 
Solve the following maximization problem, as an approximation to the minimization of the loss in \eqref{eq: RMMCL} with $\hat S^u \triangleq S^u\bigl( \{\beta^{u (0)}_{ij}\}_{i,j}; \mathcal{\hat E}^u \bigr)$.
\begin{align}
    \max_{G_1 \in \R^{r\times d_1}, G_2 \in \R^{r \times d_2}} \tr(G_1 \hat S^u G_2^\top) - R(G_1, G_2),\label{eq: approximated RMMCL}
\end{align}
where $\beta^{u (0)}_{ij} = \beta_{ij}^u(G_1^{(0)}, G_2^{(0)})$ is obtained using initial representations $G_1^{(0)}$ and $G_2^{(0)}$. 

Although this is a two-step procedure, we note that an iterative version of this procedure can also be considered. The details and results are deferred to the Appendix \ref{sec: edge detection RMMCL}.
%\linjun{which section?}
%, which is proven to have exactly the same theoretical guarantee.
%\lm{but gives better representation performance in practice?}
%The details of the iterative algorithm are deferred to Algorithm \ref{alg: MMCL Unlabeled} in the appendix.
%Recall that $\beta_{i} \equiv 1$ in Algorithm \ref{alg: MMCL}.

%In the following, the penalty term $R$ is set to $R(G_1, G_2) = (\rho^2/2) \|G_1 G_2^\top\|_F^2$. Since, unlike the case of multimodal contrastive loss \ref{eq: linear loss}, unobserved ground-truth edges do not reduce the signals from observed ground-truth edges, we do not consider the regularization $\|G_1 G_2^\top - \bar S\|_F^2$.
%\wl{would be more clear if we have one more sentence. like observed false edges becomes the major issue in this regime thus information from $S$ can be misleading.}
%Using an argument similar to that in Lemma \ref{lem: EYM 2}, minimizing \eqref{eq: RMMCL} via gradient descent is approximately equivalent to performing CCA on $S = n^{-2} \sum_{ij} \beta_{ij} x_i \tilde x_j^\top$.

%To obtain the initial representations $G_1^{(0)}$ and $G_2^{(0)}$,
%we minimize the loss in \eqref{eq: linear loss}, which is equivalent to applying SVD to the labeled data. 
To ensure that the obtained initial representations are accurate enough to detect ground truths, we assume the following assumptions.
\begin{asm}\label{asm: n large enough}
    %Suppose we have another set of data  according to the data generating process described in Section \ref{sec: data generation}.
    Suppose that the number of labeled pairs $n$ satisfies
    \begin{align*}
        n \geq \frac{C}{\rho^2} \frac{(r + r(\Sigma_\xi) + r(\Sigma_{\tilde \xi}))^3}{r} \log N \cdot \log (n+d_1+d_2),
    \end{align*}
    where $C > 0$ is some constant depending on $\sigma, s_1, s_2, \kappa_z^2$ and $\kappa_{\tilde z}^2$.
\end{asm}
\begin{asm}\label{asm: asymptotics 2}
    Assume that $r$ and $n$ satisfy $\log n \leq c r$, where $c > 0$ is some constant depending only on $\sigma, s_1, s_2, \kappa_z^2$ and $\kappa_{\tilde z}^2$.
    %\wl{Does it conflict with the fact that we need to find n vectors that are orthogonal to each other in $\mathbb{R}^r$ or I misunderstand something?}
\end{asm}

Assumption \ref{asm: n large enough} ensures that $\|G_1^{(0) \top} G_2^{(0)} - \rho^{-1} U_1^* \Sigma_z^{1/2} \Sigma_{\tilde z}^{1/2} U_2^{* \top}\|^2 = O((r \log n)^{-1})$ occurs with high probability, allowing one to detect ground truth pairs precisely with high probability. 
Assumption \ref{asm: asymptotics 2} is required to ensure that the similarity of ground truth pairs increases larger than the similarity of uncorrelated pairs. If $r \ll \log n$, it is more likely that two of $n$ independent random vectors in $\R^r$ are close to each other, making the false positive rate in edge detection intolerably high.

\if0
Given the initial representations, we can estimate the ground truth pairs by $\mathcal{\hat E}^u \subset [N]^2$ defined as
$
    \mathcal{\hat E}^u \triangleq \{(i, j) \in [N]^2 : s^u_{ij} \geq s^u_{(N)}\},
$
where $s^u_{(N)}$ is the $N$-th largest element in $\{s^u_{ij}\}_{i,j \in [N]}$.
\fi
For the estimation of ground-truth pairs, we have the following lemma.
%then for Algorithm \ref{alg: MMCL}, only one iteration is enough to detect the ground-truth edges that are not observed.
\begin{lem}\label{lem: edge detection RMMCL}
    Suppose Assumptions \ref{asm: signal condition number}, \ref{asm: signal-to-noise ratio}, \ref{asm: n large enough}, and \ref{asm: asymptotics 2} hold.
    Fix any $\gamma > 0$.
    %Choose $\tau \leq C(1 + \gamma)^{-1} \sqrt{r / \log n}$, where $C > 0$ is some constant depending on $\sigma, s_1, s_2, \kappa_z^2, \kappa_{\tilde z}^2$.
    %Consider applying Algorithm \ref{alg: MMCL} to the data generated from the model \ref{model: multimodal}.
    Then, with probability $1 - O(N^{-1} \vee n^{-1})$, $\mathcal{\hat E}^u = \mathcal{E}^u$ and 
    \if0
    $
        \max_{(i, j) \not\in \mathcal{\hat E}^u} \beta_{i}^{u (0)} \lesssim N^{-(1+\gamma)}% \frac{1}{N^{1 + \gamma}}
    $
    hold.
    \fi
    \begin{align*}
        %\min_{(i, j) \in \mathcal{E}^u \setminus \mathcal{C}} \beta_{ij} &\geq \frac{1}{2},\\
        \min_{(i, j) \in \mathcal{\hat E}^u} \beta_{ij}^{u (0)} &= 1 - O\qty(\frac{1}{N^\gamma}),\ \ 
        \max_{(i, j) \not\in \mathcal{\hat E}^u} \beta_{i}^{u (0)} \lesssim \frac{1}{N^{1 + \gamma}}
    \end{align*}
    hold.
\end{lem}
Lemma \ref{lem: edge detection RMMCL} states that the cross-covariance matrix $\hat S^u$ behaves as if $\hat S^u \approx (\nu - 1) N^{-1} \sum_{(i, j) \in \mathcal{E}^u} x_i \tilde x_j^\top$.
%We can see that when we have a good initialization, 
Thus, when $\nu > 1$, Algorithm \ref{alg: MMCL Unlabeled} has the ability to exploit ground-truth pairs, even if they are not observed. 
Notice that Assumption \ref{asm: n large enough} is mild in the sense that it only requires $\tilde \Omega(r^2)$ number of samples when $r(\Sigma_\xi) \vee r(\Sigma_{\tilde \xi}) \lesssim r$.
Taking advantage of this result, we can improve the performance of feature learning by incorporating unpaired data, as summarized in the next result.

\if0
From a similar argument as in Theorem \ref{thm: feature recovery via  MMCL}, we have the following result for Algorithm \ref{alg: MMCL Unlabeled}
\begin{lem}[{\bc The proof is to be added}]\label{lem: edge detection robust SVD}
    Suppose Assumptions \ref{asm: asymptotics 2}, \ref{asm: signal condition number}, and \ref{asm: signal-to-noise ratio} hold.
    Fix any $\gamma > 0$.
    Choose $\tau = o((1 + \gamma)^{-1} \sqrt{r / \log n})$.
    If
    \begin{align*}
        \|G_1^{(0) \top} G_2^{(0)} - U_1^* \Sigma_z^{1/2} \Sigma_{\tilde z}^{1/2} U_2^{* \top}\|^2 = o\qty( \frac{r}{(r + r(\Sigma_\xi))(r + r(\Sigma_{\tilde \xi})) \log n} ),
    \end{align*}
    then, with probability $1 - O(n^{-1})$,
    \begin{align*}
        %\min_{(i, j) \in \mathcal{E}^u \setminus \mathcal{C}} \beta_{ij} &\geq \frac{1}{2},\\
        \min_{(i, j) \in \mathcal{E}^u} \beta_{ij}^{(1)} &\geq 1 - O\qty(\frac{1}{n^\gamma}),\\
        \max_{(i, j) \in \mathcal{E}^u^\perp} \beta_{ij}^{(1)} &\leq \frac{1}{n^{1 + \gamma}}.
    \end{align*}
\end{lem}

By Lemma \ref{lem: edge detection robust SVD}, the matrix of the SVD objective $S = n^{-1} \sum_{ij} \beta_{ij}^{(1)} x_i \tilde x_j^\top \approx n^{-1} \sum_{(i,j) \in \mathcal{E}^u} x_i \tilde x_j^\top$. This gives the following theorem:
\fi

%For the feature recovery ability, we provide the finite sample analysis as in the theorem below.
\begin{thm}\label{thm: feature recovery via RMMCL}
    Suppose Assumptions \ref{asm: signal condition number}, \ref{asm: signal-to-noise ratio}, \ref{asm: n large enough}, and \ref{asm: asymptotics 2} hold.
    Fix any $\gamma > 1$ and $\nu > 1$.
    %Choose $\tau$ as in Lemma \ref{lem: edge detection RMMCL}.
    Choose $\tau \leq C(1 + \gamma)^{-1} \sqrt{r / \log n}$, where $C > 0$ is some constant depending on $\sigma, s_1, s_2, \kappa_z^2, \kappa_{\tilde z}^2$.
    Let $G_1$ and $G_2$ be the solution to the maximization problem in \eqref{eq: approximated RMMCL}.
    %to the data generated from \eqref{model: multimodal}.
    %\begin{align}
    %    |p - \frac{1}{2}| \geq \frac{(r + r(\Sigma_\xi) + r(\Sigma_{\tilde \xi})) \log n}{n^\gamma} + \sqrt{\frac{(r + r(\Sigma_\xi) + r(\Sigma_{\tilde \xi})) \log (nd_1 + nd_2)}{n}},
    %\end{align}
    Then, with probability $1 - O(N^{-1} \vee n^{-1})$,
    \begin{align*}
        &\|\sin\Theta(P_r(G_1), U_1^*)\|_F \vee \|\sin\Theta(P_r(G_2), U_2^*)\|_F\\
        &\quad\lesssim \sqrt{r} \wedge \sqrt{\frac{r (r + r(\Sigma_\xi) + r(\Sigma_{\tilde \xi})) \log (N+d_1+d_2)}{N}}.
    \end{align*}
    %holds for all $t \geq 1$.
\end{thm}
%The proof of Theorem \ref{thm: feature recovery via RMMCL} is deferred to Appendix \ref{proof: thm: feature recovery via RMMCL}. 
%Since the loss in \eqref{eq: RMMCL} is defined in a way that it considers any possible pairs between modalities, we can see that the bound given in Theorem \ref{thm: feature recovery via RMMCL} does not depend on the distortion rate $p$. 
%Therefore, the upper bound in Theorem \ref{thm: feature recovery via RMMCL} improves the upper bound in Theorem \ref{thm: feature recovery via MMCL} when the distortion rate $p$ is very close to a particular value.
Theorem~\ref{thm: feature recovery via RMMCL} suggests that our proposed procedure is able to process the unpaired data as if they are paired, greatly improving the multimodal learning performance.% In the setting where the number of unpaired data is abundant, this would lead to a great improvement of multimodal learning.  %The convergence rate on the right hand side depends on the number of unpaired data, greatly improving the performance.

\section{Numerical Experiments}\label{sec: experiment}
%\subsection{Comparison of Multimodal Contrastive Learning and Algorithm \ref{alg: MMCL Unlabeled}}
%\vspace{-0.5em}
    %\centering
\if0
Candidate experiments are
\begin{enumerate}
    \item comparing the feature recovery performance of CLIP and Contrastive Learning applied separately%\lm{I will try this}
    \item showing the robustness of CLIP against distortion
    \item showing the performance of semi-supervised CLIP against the number of unpaired data
\end{enumerate}
\fi

In this section, we first show that clearning the noisy pairs using the BSGMP algorithm \citep{dhillon2001co} helps MMCL improve the downstream task performance in the many-to-many correspondence case.
In addition, we show that we can improve the performance of MMCL by incorporating the unpaired data. \footnote{The code is available at \href{https://github.com/nswa17/MMCL}{https://github.com/nswa17/MMCL}.}
%We also emphasize that our first methodology works well when we have one-to-one pairs in the data distribution and when unpaired data is available, which is very common in real-world applications. 
%On the other hand, graph partitioning algorithm improves the performance considerably when we have many-to-many mappings, namely one data point is paired with many data points in the other modality, and when there are messy pairs in the training data. These messy pairs are basically wrong edges between different modalities. 

%\subsection{Graph Partitioning}
%\subsection{Real Data Analysis with Graph Partitioning}
\subsection{Eliminating Noisy Pairs}\label{sec: experiment many-to-many}
As briefly mentioned in Section \ref{sec: experiment many-to-many summary}, we use the BSGMP algorithm to eliminate incorrectly aligned edges from the training data and compare the performance for different distortion rates.
%This is a generalization of the spectral graph partitioning algorithm that enables us to detect wrong edges in a bipartite graph that we expect to have a clustered shape. 
%In real-world multi-modal datasets, messy pairs are common, and we show that using graph partitioning we can detect some of the messy pairs and delete them from the training data, which results in better performance.

\begin{algorithm}
   \caption{Bipartite Spectral Graph Multi-Partitioning}\label{alg: Graph}
    \begin{algorithmic}
        \State \textbf{Input}: two-modal dataset $\mathcal{E}'$. The number of clusters $k$.
        \State Calculate an adjacency matrix $A$ of the bipartite graph.
        \State Let $A_n \triangleq D_1^{-1/2} A D_2^{-1/2}$, where $D_1$ and $D_2$ are diagonal matrices with $(D_1)_{ii} = |\{j \in [N] : (x_i,\tilde{x}_j) \in \mathcal{E}'\}|$ and $(D_2)_{ii} = |\{j \in [N] : (x_j,\tilde{x}_i) \in \mathcal{E}'\}|$.
        \State Let $u_2,..,u_{l+1}$ be the left singular vectors and  $v_2,..,v_{l+1}$ be the right singular vectors of $A_n$, where $l \triangleq \log_{2} k$.
        \State Define $
            Z \triangleq \begin{bmatrix}
                    D_1^{-1/2} & U\\
                    D_2^{-1/2} & V
            \end{bmatrix}$, where $U\triangleq[u_2,..,u_{l+1}]$ and $V\triangleq[v_2,..,v_{l+1}]$.
        \State Apply $k$-means algorithm to columns of the matrix $Z$.
        \State \textbf{Output}: $\mathcal{E}'$ without all intercluster pairs.
    \end{algorithmic}
\end{algorithm}

In this experiment, we use MNIST and Fashion-MNIST datasets as different modalities. We pair images if they belong to the same class in each modality. 
For example, all digit-2 images in MNIST and all pullover images in Fashion-MNIST are connected.
Note that, apart from the settings in Sections \ref{sec: analysis of MMCL} and \ref{sec: results of RMMCL}, we have a bipartite graph with many-to-many edges. 
In experiments, the number of training samples is set $n=500$ for each modality, and the samples are equally distributed among different classes. 
%In order to enhance different modality concept, we use Fashion-MNIST dataset as images whereas we flatten all the images in MNIST dataset and feed to our model as 1-dimensional vectors. 
For the MNIST side, we have fully-connected neural networks, while we use convolutional neural networks for the Fashion-MNIST side. The dimension of the latent space is chosen to be $r=128$. 
After creating our dataset, we distort the pairs in the following way. For all $1\leq i,j\leq n$, if $x_i$ and $\tilde{x}_j$ are paired, we remove this pair with probability $p'$. Similarly, if $x_i$ and $\tilde{x}_j$ are not paired, we pair them in our dataset with probability $p'$. 
Note that due to this distortion, our bipartite graph loses its clustered structure, which we try to regain with Algorithm \ref{alg: Graph}.
Although we know the true number of clusters $k=10$, we treat $k$ as unknown and try a different number of clusters $k=5, 7, 10, 13, 15$. We also perform MMCL without applying algorithm \ref{alg: Graph}.
%We demonstrate results for no partitioning as well as partitioning into $k$ clusters, where $k$ takes different values. 
%For our dataset, we know that the ground-truth bipartite graph has $10$ different clusters. However, it may not be possible to know the number of clusters in practice. 
%In these situations, one can make estimations for the number of classes. 
%This is why we use different values for $k$ to see their performance.
%The performance of the downstream prediction task is measured by the classification accuracy of the representations obtained.
\if0
The performance of the learned representations is measured by a downstream classification task as in \citet{radford2021learning} using test data generated in the same way as the training samples.
For the prediction task, the test data is fed into obtained encoders.
For any test data in the Fashion-MNIST side, the most similar data are chosen from the test data in the MNIST side.
We then measure the accuracy of this prediction task for different $k$ and $p'$.
\fi
The performance of the learned representations is measured by a downstream classification accuracy as in \citet{radford2021learning} using test data, which is generated in the same way as the training data. That is, for each test datum $x$ on the Fashion-MNIST side, the most similar test datum $\tilde x$ on the MNIST side is chosen. We then measure the accuracy by the rate of $x$ and $\tilde x$ whose labels are equal. The experiment was performed for different $k$ and $p'$.
%To test our models, we take 1000 samples from the Fashion-MNIST dataset and 10 samples from the MNIST dataset. MNIST samples belong to different classes. For each sample from Fashion-MNIST we look at the similarity of this sample with all the 10 samples from the MNIST dataset. If our data sample belongs to the $i$-th class and if the most similar one belongs to the $i$-th class, then we count it as a correct classification; otherwise, we count it as a missclassification.

\begin{table}[!h]
\small
\caption{Downstream task classification accuracy with different distortion rates $p'$.} 
\label{table:GraphTable}
\begin{center}
\begin{tabular}{|l|c|c|c|}
\textbf{Partitioning}  &\textbf{$p'=0.1$} &\textbf{$p'=0.2$}
&\textbf{$p'=0.3$}\\
\hline
$k=5$         &0.370 & 0.407 & 0.353\\
$k=7$              &0.438  &0.482 &0.481\\
$k=10$            &0.675 &\textbf{0.672} &\textbf{0.662}\\
$k=13$             &0.667 &0.669  &0.596\\
$k=15$            &0.650 &0.612  &0.579\\
No Partitioning  &\textbf{0.684} &0.628 &0.159 \\
\end{tabular}
\end{center}
%\vspace{-2em}
\end{table}
The result is reported in Table \ref{table:GraphTable}. When the distortion rate is $p'=0.1$, the naive use of MMCL ("No Partitioning") performs the best. 
As $p'$ increases, the downstream task performance decreases if no partitioning is applied. However, applying algorithm \ref{alg: Graph} with $k=10$ yields the highest accuracy for $p'=0.2$ and $p'=0.3$. Note that this allows one to choose the number of clusters $k$ by cross-validation.

\subsection{Incorporating Unpaired Data}\label{sec: experiment unlabeled}

\begin{figure}[!h]
    \centering
    \includegraphics[width=0.4\textwidth]{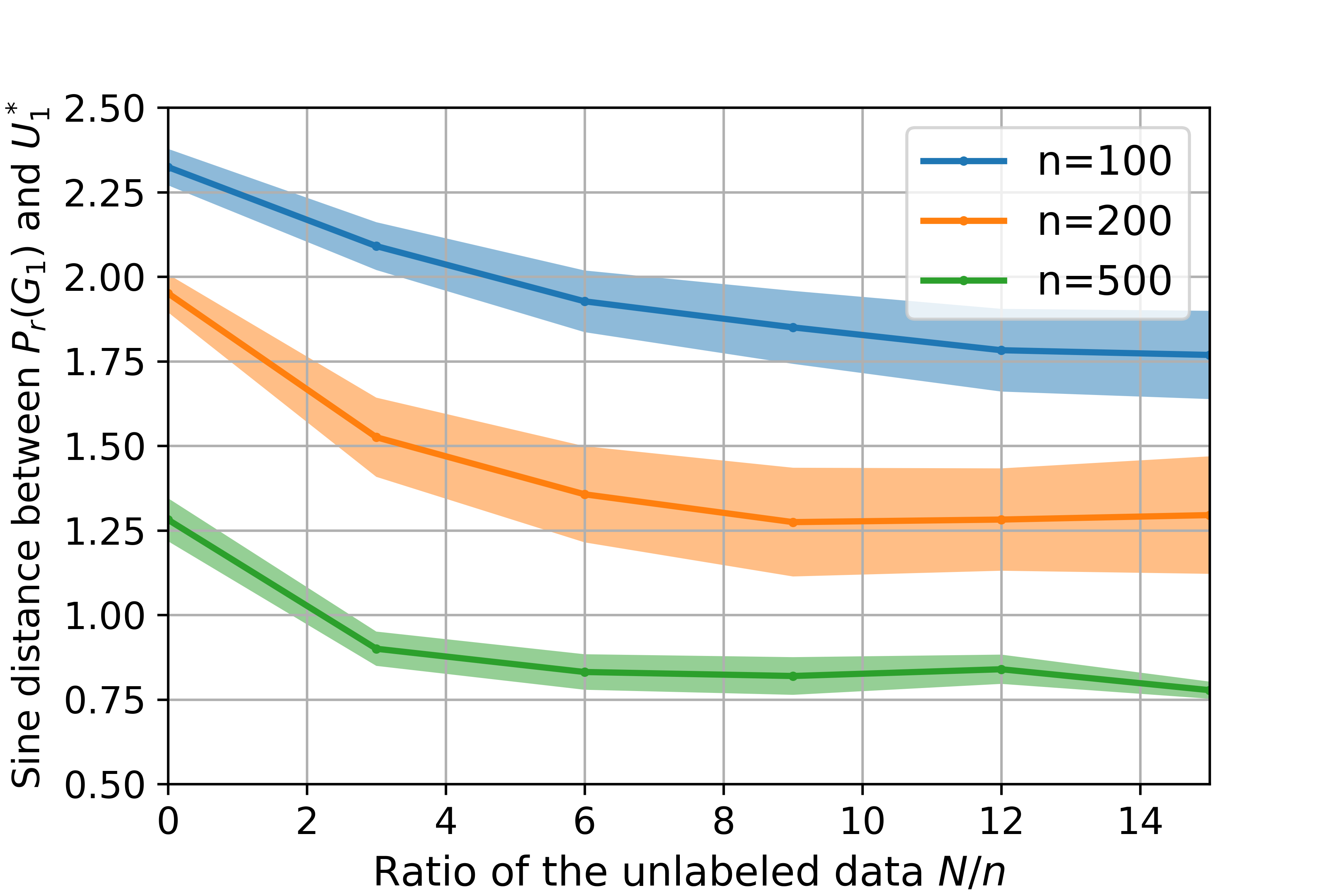}
    \caption{Performance of feature recovery ability measured by $\|\sin\Theta(P_r(G_1), U_1^*)\|_F$ with additional unpaired data with different $n=100,200,500$.}
    %\james{Add legends and more explanations. If time allows, good to have multiple runs with error bars.}}
    \label{fig:unsup}
\end{figure}
In this experiment, we show that we can improve the performance of MMCL by incorporating unpaired data. 
Specifically, we generate a synthetic dataset with $d_1=40,d_2=39,n=300,r=10$ according to the model in \eqref{model: multimodal}.  $U_1^* \in \mathcal{O}_{d_1,r}$ and $U_2^* \in \mathcal{O}_{d_2,r}$ are random orthonormal matrices, and 
$z_i$ and $\tilde{z}_i$ are sampled from the standard Gaussian distribution for $i \in [n]$. Similarly, $\xi_i$ and $\tilde{\xi}_i$ are sampled from the zero-mean Gaussian distribution with standard deviation $0.3$. 
The unpaired data are generated in the same way as the paired data with $N$ number of samples for each modalities. 
We first train the initial linear encoders by minimizing InfoNCE 
%\lm{just want to check, is it trained by InfoNCE or linear loss?} \ig{Yes it is. I used the training part of the code you sent me, it was InfoNCE} \lm{Thanks Ibrahim!}
loss with the paired data. Then we train the linear encoders using the procedure described in Section \ref{sec: results of RMMCL} with unpaired data.
%minimizing the loss defined in \eqref{eq: RMMCL} with unpaired data.
%We also generate test data in the same way with the training data with $10000$ samples for each modality. 
The performance of the obtained representations is measured by the $\sin\Theta$ distance between $U_1^*$ and $P_r(G_1)$. We consider different ratios $N/n$ with $n=100,200,500$, and the results are summarized in Figure \ref{fig:unsup}.
%the downstream prediction task using a similar prediction method described in Section \ref{sec: experiment unpaired}; For any data in one modality, the most similar data chosen from the other modality serve as the predicted value. We then measure the accuracy of this prediction task against different numbers of unpaired data $N/n$ for $n=100,200,300$.

%For the prediction task, the test data are fed into trained representations. For any data in one modality, the most similar data are chosen from the other modality.
%We then measure the accuracy of this prediction task against different number of unpaired data $N$.
%In the test phase, for each data sample, we look at similarities with the data samples from the other modality. If the most similar sample is the paired one, then we count it as classified correctly; otherwise, it is treated as a missclassification. 
%That is, testing is done as $10000$-class classifications with $10000$ samples. 
%We look at the accuracy for different values of $N$ in Fig. \ref{fig:unsup}. 

%The blue curve indicates the downstream task performance of representations obtained.
As the ratio of unpaired data with labeled data $N/n$ increases, we can observe that the $\sin\Theta$ distance decreases, which validates our theory in the sense that if the model is initialized with relatively good accuracy, having more and more unpaired data improves the performance. 

%We emphasize that our first methodology works well when we have one-to-one pairs in the data distribution and when unpaired data is available, which is very common in real-world applications. 
%

\subsection{Application to Real Datasets}\label{sec: incorporating real data unlabeled}

%The gray curve indicates the test accuracy of initial representations with $n=500$ against the number of epochs. The orange curve indicates the accuracy of MMCL when we additionally feed $N=59000$ pairs. The blue curve indicates the accuracy of semi-supervised CLIP with $N=59000$ unpaired data.
%A more detailed discussion is added in the revised Appendix.
%We also run an experiment when $n=1000$ and $N=50000$, which is deferred to Section 7.1 in the revised Appendix. 

In this experiment, we show that semi-supervised CLIP improves the performance for real data. The dataset is created by artificially pairing images from MNIST and Fashion-MNIST datasets based on their labels. Namely, we generate pairs $(x_i, \tilde x_i)_{i=1}^n$ so that the digit in MNIST is paired with a random image in Fashion-MNIST with the corresponding digit. 
Similarly, we generate validation pairs $(x_i^v, \tilde x_i^v)_{i=1}^v$ with some $v$, the details of which are explained later.
We use random data from each modality as unlabeled datasets $(x_i^u)_{i=1}^N$ and $(\tilde x_i^u)_{i=1}^N$.

In CLIP and ALIGN \citep{jia2021scaling,radford2021learning}, pre-trained encoders such as Vision Transformers and BERT are used.
To speed up the learning process, we first reduce the dimension of data; we train autoencoders consisting of 2-dimensional multilayer convolutional neural networks on datasets $(x_i^u)_{i=1}^N$ and $(\tilde x_i^u)_{i=1}^N$ separately. Let the encoders obtained be $E_1 : \R^{28^2} \to \R^{16}$ and $E_2 : \R^{28^2} \to \R^{16}$ for MNIST and Fashion-MNIST, respectively.

We first train the initial representations for the dataset $(E_1(x_i), E_2(\tilde x_i))_{i=1}^n$ while fixing the parameters of $E_1$ and $E_2$. %\lm{We note that to achieve the performance improvement, the parameters of encoders need to be fixed during the training.}
Then, we train the semi-supervised CLIP with another set of data whose association is unknown. 
%The number of samples is se$N=1000$.
The initial representations consist of two two-layer neural networks.

We call the representations used to estimate the ground-truth pairs \textit{anchor representations}. Given anchor representations $G_1^a$ and $G_2^a$, we estimate the ground-truth pairs by $\mathcal{\hat E}^u$ as in \eqref{eq: E hat}.
%\begin{align}
%    \mathcal{\hat E}^u &\triangleq \{ (i, j): i \in [N], j \in \argmax_{j'} s_{ij'}^u \}\\
%    &\quad\cup \{ (i, j): j \in [N], i \in \argmax_{i'} s_{i'j}^u \},\label{eq: E hat}
%\end{align}
%where $s_{ij}^u = \langle G_1^a x_i^u, G_2^a \tilde x_i^u \rangle$.

Let the initial representations trained on the dataset $(E_1(x_i), E_2(\tilde x_i))_{i=1}^n$ be initial anchor representations.
Since the performance of semi-supervised CLIP is restricted by the performance of anchor representations, we update anchor representations when the learned representations outperform the current anchor representations by a certain ratio. The performance of representations is measured on validation pairs $(x_i^v, \tilde x_i^v)_{i=1}^v$

We employ the AdamW algorithm %with learning rate $10^{-3}$, 
and use mini-batch optimization with batch size $64$ to reduce the load of calculating the similarity matrix.
The number of epochs for both initial training and training with labeled data is set $100$.

We compare the performance of semi-supervised CLIP with CLIP trained with unknown association. 
\begin{figure}[H]
    \centering
    %\vspace{-3em}
    \includegraphics[width=0.5\textwidth]{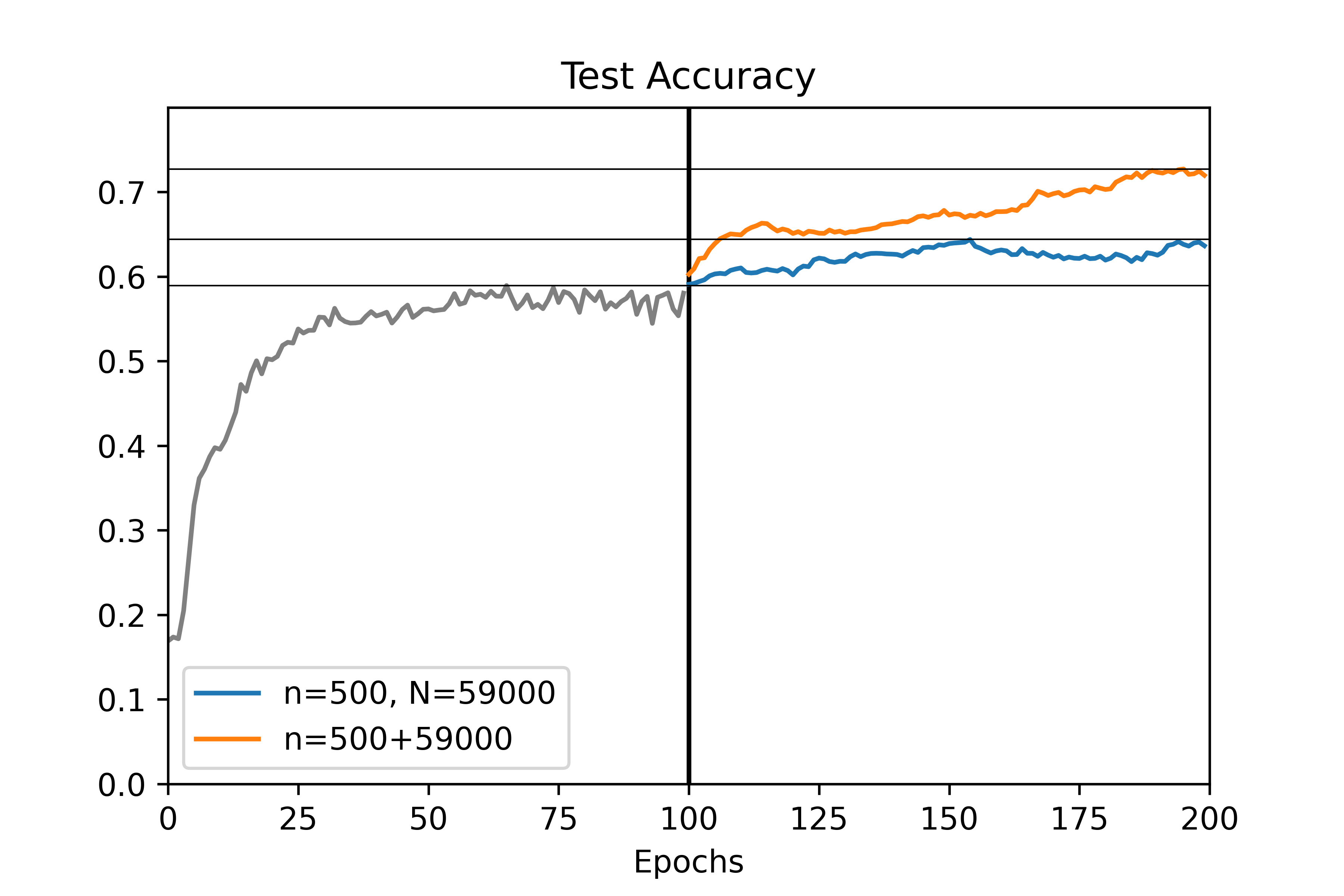}%{fig/Test Accuracy.png}
    \caption{The comparison of the downstream task performance of semi-supervised CLIP and oracle CLIP. $n$, $N$, and $v$ are the number of samples used to obtain initial representations, to train semi-supervised CLIP, and to validate semi-supervised CLIP, respectively. 
    The gray curve indicates the performance when training initial representations.
    The orange curve indicates the performance of semi-supervised CLIP and the blue curve indicates the performance of oracle CLIP.}\label{fig: semi-supervised CLIP when N1000}
\end{figure}
Figure \ref{fig: semi-supervised CLIP when N1000} shows the improvement of test accuracy when we additionally train CLIP with unlabeled data. The parameters are set as $N=59000$, $n=v=500$, $\tau=1$ and $\eta = 1.1$. 
The gray curve indicates the test accuracy of initial representations against the number of epochs. The orange curve indicates the accuracy of representations when we additionally feed the ground-truth pairs of unlabeled data. The blue curve indicates the accuracy of representations with additional unlabeled data. From this result, we can see the improvement in test accuracy with additional unlabeled data.

\section{Discussion}\label{sec: discussion}
%\vspace{-0.5em}
In this paper, we provide a theoretical understanding of MMCL in linear representation settings with two-modal data.
We showed that for a general class of non-linear MMCL loss performing gradient descent on the loss function is equivalent to gradient ascent of the SVD objective function with contrastive cross-covariance matrix. Using this result, we analyze the feature recovery ability of the approximated algorithm under linear loss in the presence of noisy pairs. We note that the feature recovery ability of MMCL attains a better theoretical bound compared to that attained by SSCL applied separately to each modality.
For data with many-to-many correspondence, we numerically show that we can improve the performance of MMCL by eliminating incorrectly paired edges using BSGMP.
We also proposed a semi-supervised framework that incorporates the unpaired dataset to enhance the performance of MMCL. Given a small number of labeled data, it can successfully detect the ground-truth alignment for unpaired data and improve the representations. The improvement in performance with linear encoders is verified by numerical experiment.
%\lm{We numerically verified that the (approximated) proposed algorithm outperforms the (approximated) InfoNCE type contrastive loss when the distortion rate is high.
%\lm{We numerically verified that for synthetic data, (...) }
To the best of our knowledge, this is the first work on the theoretical analysis of MMCL that incorporates the unpaired data.
%To our knowledge, this is the first work on theoretical analysis of MMCL assuming the existence of noisy pairs in collected data.
We emphasize that our main contribution is to provide theoretical analysis and insights on MMCL.
The analysis of other multimodal pre-train learning algorithms remains a future work. While we verified our theory with proof-of-concept experiments, systematic experiments with non-linear representations on larger datasets is a good direction of future work.
It is also possible to extend two-modal contrastive learning to more than three modalities by summing up the loss \eqref{eq: general loss} for all pairs of modalities. As an analogy to the results of Section \ref{sec: multimodal and SVD}, we can interpret loss minimization via gradient descent as maximizing the sum of the SVD objective functions (Corollary \ref{cor: min nonlinear loss is SVD with more than three modalities}.) An analysis of its feature recovery ability remains a future work.
%\james{Good to emphasize that our main contribution is the theoretical analysis and insights on MMCL. Our experiments are limited because they are meant to be proof-of-concept; systematic experiments on larger datasets is good future work. }

%\subsubsection*{References}

\section*{Acknowledgements}
The research of Linjun Zhang is partially supported by NSF DMS-2015378. The research of
James Zou is partially supported by funding from NSF CAREER and the Sloan Fellowship.

\bibliography{main.bib}
\bibliographystyle{apalike}

\newpage
% If your paper is accepted, change the options for the package
% aistats2022 as follows:
%
%\usepackage[accepted]{aistats2022}
%
% This option will print headings for the title of your paper and
% headings for the authors names, plus a copyright note at the end of
% the first column of the first page.

% If you set papersize explicitly, activate the following three lines:
%\special{papersize = 8.5in, 11in}
%\setlength{\pdfpageheight}{11in}
%\setlength{\pdfpagewidth}{8.5in}

% If you use natbib package, activate the following three lines:
%\usepackage[round]{natbib}
%\renewcommand{\bibname}{References}
%\renewcommand{\bibsection}{\subsubsection*{\bibname}}

% If you use BibTeX in apalike style, activate the following line:
%\bibliographystyle{apalike}

% If your paper is accepted and the title of your paper is very long,
% the style will print as headings an error message. Use the following
% command to supply a shorter title of your paper so that it can be
% used as headings.
%
%\runningtitle{I use this title instead because the last one was very long}

% If your paper is accepted and the number of authors is large, the
% style will print as headings an error message. Use the following
% command to supply a shorter version of the authors names so that
% they can be used as headings (for example, use only the surnames)
%
%\runningauthor{Surname 1, Surname 2, Surname 3, ...., Surname n}

% Supplementary material: To improve readability, you must use a single-column format for the supplementary material.
\appendix

\begin{appendices}

\section*{Appendix}

In this appendix, we define the following notations.
Let $\1\{A\}$ be an indicator function that takes $1$ when $A$ is true, otherwise takes $0$.
For any square matrix of the same order $A$ and $B$, we write $A \prec B$ if $u^\top (B - A) u \geq 0$ holds for all unit vector $u$.
Define the set of pairs in $[n]^2 \setminus \mathcal{E}$ containing $x_i$ as $\mathcal{E}^\perp_{i,\cdot} \triangleq \{(i, j) \in \mathcal{E}^\perp : j \in [n]\}$ and similarly the pairs containing $\tilde x_j$ as $\mathcal{E}^\perp_{\cdot,j} \triangleq \{(i, j) \in \mathcal{E}^\perp : i \in [n]\}$.
For any matrix $A$, let $\SVD_r(A)$ be the rank-r approximation of $A$.
Let $\mathbb{S}^{d-1} \triangleq \{x \in \R^d: x^\top x = 1\}$ be a sphere on $\R^d$.

\section{Omitted Contents}\label{sec: omitted}

\subsection{Numerical Experiments}\label{sec: experiment omitted}

Here is the algorithm used in the experiment in Section \ref{sec: incorporating real data unlabeled}.
\begin{algorithm}
   \caption{Semi-supervised MMCL}\label{alg: Semi-supervised MMCL experiment}
    \begin{algorithmic}
        \State \textbf{Input:} Labeled pairs $(x_i, \tilde x_i)_{i=1}^n$, validation pairs $(x_i^v, \tilde x_i^v)_{i=1}^n$, unlabeled data $(x_i^u)_{i=1}^N, (\tilde x_i^u)_{i=1}^N$, temperature $\tau > 0$, update ratio $\eta > 0$.
        \State Obtain the initial representations $G_1^{(0)}$ and $G_2^{(0)}$ from the paired dataset $(E_1(x_i), E_2(\tilde x_i))_{i\in[n]}$ by minimizing CLIP loss.
        \State Let $G_1^{a} = G_1^{(0)}$ and $G_2^{a} = G_2^{(0)}$.
        \Repeat
            \State Calculate the similarity of all possible unlabeled pairs by $s_{ij}^u = \langle G_1^{(0)} E_1(x_i^u), G_2^{(0)} E_2(\tilde x_j^u) \rangle$ for $i, j \in [N]$.
            \State Estimate the set of ground truth pairs according to \eqref{eq: E hat}.
            \State Obtain $G_1$ and $G_2$ by minimizing CLIP loss with artificially paired dataset $(E_1(x_i^u), E_2(\tilde x_j^u))_{(i, j) \in \mathcal{\hat E}^u}$.
            \If{$G_1$ and $G_2$ outperforms $G_1^{a}$ and $G_2^{a}$ on validation set $(E_1(x_i^v), E_2(\tilde x_i^v))_{i=1}^n$ by $\eta$,}
                \State Set $G_1^{a} = G_1$ and $G_2^{a} = G_2$.
            \EndIf
        \Until convergence
        \State \textbf{Output:}  $G_1$ and $G_2$.
    \end{algorithmic}
\end{algorithm}

\if0
\subsection{Numerical Experiments}\label{sec: experiment omitted}

\begin{figure}[H]
    \centering
    %\vspace{-3em}
    \includegraphics[width=0.5\textwidth]{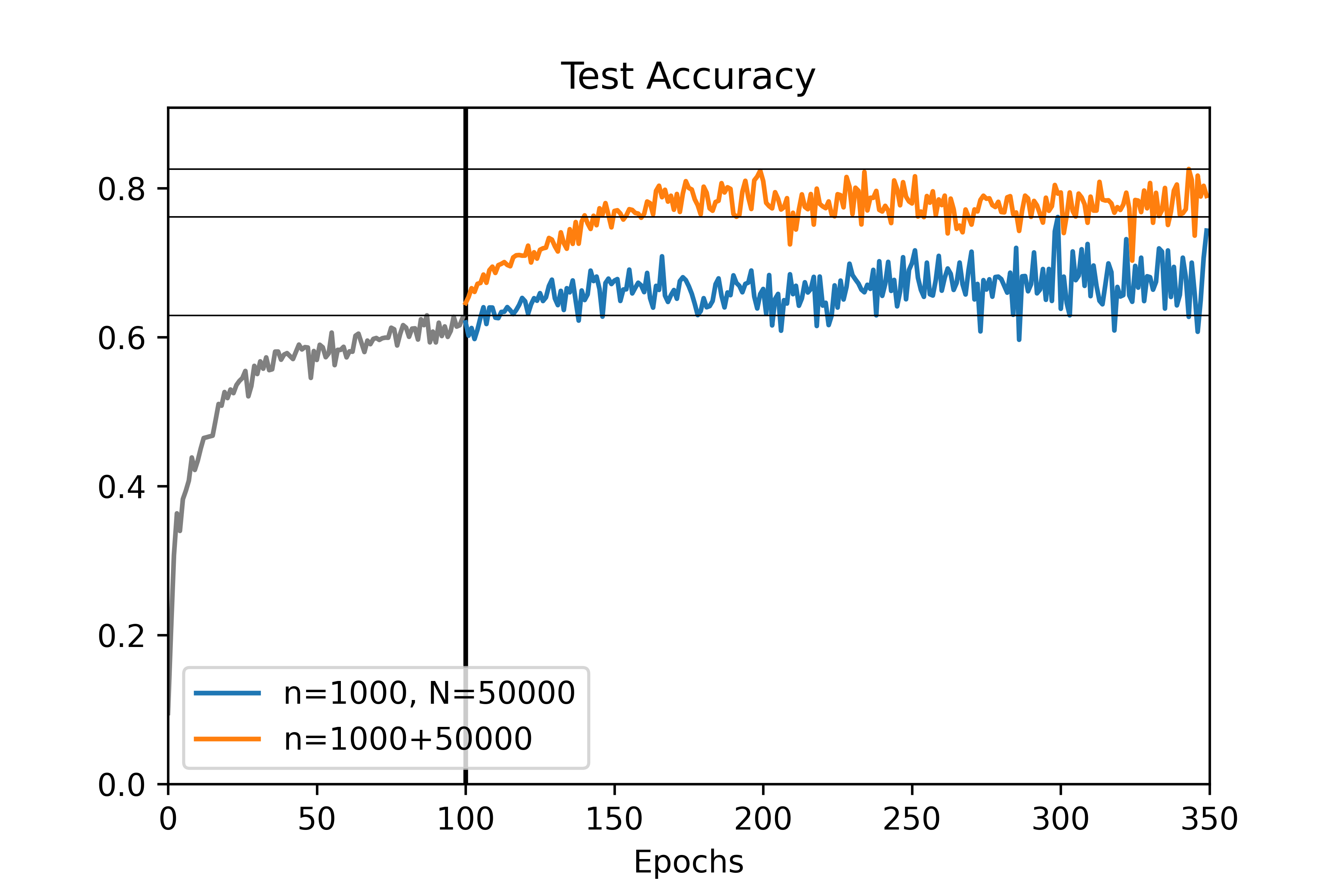}
    \caption{The comparison of the downstream task performance of semi-supervised CLIP and oracle CLIP. $n$, $N$, and $v$ are the number of samples used to obtain initial representations, to train semi-supervised CLIP, and to validate semi-supervised CLIP, respectively. We set $n=1000$, $N=50000$, $v=1000$.
    The gray curve indicates the performance when training initial representations.
    The orange curve indicates the performance of semi-supervised CLIP and the blue curve indicates the performance of oracle CLIP.}
\end{figure}
\fi

\subsection{Feature Recovery via SSCL}

Define the incoherent constant, which measures the closeness between the standard basis and orthonormal column vectors of a matrix $U \in \mathbb{O}_{d,r}$ as $I(U) \triangleq \max_{i\in [d]} \|e_i^\top U\|^2$.
For the learned representation, we invoke the following theorem from \citet{ji2021power}.
\begin{lem}[Theorem 3.11 from \citet{ji2021power}]\label{lem: feature recovery from CL}
    Suppose that $n > d \gg r$ and the condition number of $\Sigma_{\xi}$ and $\Sigma_{\tilde \xi}$ are bounded above, and $I(U^\star)=O\qty(r\log d/d)$.
    Consider applying random masking augmentation. Then,
    % in Definition \ref{aug: random masking},
    %Let $W_{\CL}$ be any solution that minimizes Equation \eqref{loss: self contrastive}, and denote its singular value decomposition as $W_{\CL}=(U_{\CL}\Sigma_{\CL}V_{\CL}^\top)^\top$, then we have
	\begin{align*}
		\mathbb{E}\left\|\sin \Theta\qty(U_1^\star, P_r(G_1^u))\right\|_F \lesssim\frac{r^{3/2}}{d}\log d+\sqrt{\frac{dr}{n}}. 
	\end{align*}
\end{lem}
Note that the assumption that the condition numbers of $\Sigma_\xi$ and $\Sigma_{\tilde \xi}$ are bounded above implies that $r(\Sigma_\xi) \gtrsim d_1$ and $r(\Sigma_{\tilde \xi}) \gtrsim d_1$. Ignoring the logarithmic term, and provided that $d_1 \asymp d_2$, we notice that the bound in Theorem \ref{thm: linear loss} improves the rate in Lemma \ref{lem: feature recovery from CL} by reducing the bias term $r^{3/2} \log d / d$, while the variance term remains almost the same.
%Thus the variance bound is almost the same ignoring the logarithm term.
The bias term appearing in the bound in Lemma \ref{lem: feature recovery from CL} is due to the fact that core feature $U_1^*$ loses its information when the random masking data augmentation is applied to the original data.
Also, note that our result \ref{thm: feature recovery via MMCL} does not require the incoherent constant assumption, because we can separate core features from noise using the fact that the core features are highly correlated, while noises are not correlated between two modalities.
%the improvement of bias term is due to our sharper bound of the rate of convergence for cross-covariance matrix.
\citet{ji2021power} provably showed that when the noise covariance shows strong heteroskedasticity, the feature recovery performance of representations obtained by autoencoders stays constant, while contrastive learning can mitigate the effect of heteroskedasticity. Therefore, under strong heteroskedasticity, MMCL can learn representations better than autoencoders applied to each modality separately.

\subsection{Analysis on Multimodal Contrastive Loss Function with InfoNCE Loss}\label{sec: MMCL with InfoNCE}

Before going to the proof, we modify the multimodal contrastive loss with InfoNCE loss with $\epsilon = 1$ as
\begin{align}
    &\mathcal{L}(G_1, G_2) \triangleq - \frac{\tau}{2n} \sum_{(i, j) \in \mathcal{E}^u} \log \frac{e^{\nu s_{ii}^u / \tau}}{\sum_{j \in [n]} e^{s_{ij}^u/\tau}} - \frac{\tau}{2n} \sum_{(i, j) \in \mathcal{E}^u} \log \frac{e^{\nu s_{ii}^u / \tau}}{\sum_{i \in [n]} e^{s_{ij}^u/\tau}} + R(G_1, G_2),\label{eq: MMCL long}
\end{align}
where $\nu \geq 1$. Setting $\nu > 1$ corresponds to choosing different temperature parameters for positive pairs and negative pairs.

Similar to the argument in Proposition \ref{prop: min nonlinear loss is SVD}, each step of the gradient descent of minimizing the loss in \eqref{eq: MMCL long} corresponds to performing gradient ascent to the objective function
$
    \tr(G_1 S G_2^\top) - (\rho/2) \|G_1^\top G_2\|_F^2
$
with the matrix $S$ given by
\begin{align}
    S = \frac{1}{n} \sum_{i \in [n]} \beta_i x_i \tilde x_i^\top - \frac{1}{n} \sum_{i \neq j} \beta_{ij} x_i \tilde x_j^\top,\label{eq: S}
\end{align}
where 
\begin{align}
    \beta_i &= \nu-1 + \frac{1}{2} \frac{\sum_{j': j' \neq i} \exp(s_{ij'}/\tau)}{\sum_{j': j' \in [n]} \exp(s_{ij'}/\tau)} + \frac{1}{2} \frac{\sum_{j': j' \neq i} \exp(s_{j'i}/\tau)}{\sum_{j': j' \in [n]} \exp(s_{j'i}/\tau)},\nonumber\\
    \beta_{ij} &= \frac{1}{2} \frac{\exp(s_{ij}/\tau)}{\sum_{j': j' \in [n]} \exp(s_{ij'}/\tau)} + \frac{1}{2}\frac{\exp(s_{ij}/\tau)}{\sum_{i': i' \in [n]} \exp(s_{i' j}/\tau)}.\label{eq: beta ij}
\end{align}
%Note that $0 < \beta_{ij} < 1$ for all $i,j \in [n]$.
%By Lemma \ref{lem: EYM 2}, each step of the gradient ascent in optimizing \eqref{eq: each step of GD} is equivalent to performing a rank-$r$ approximation of $S + \rho_2^2 \bar S$.
From this observation and Lemma \ref{lem: EYM 2}, we study the following loss minimization problem as 
an approximation to MMCL.
%It is noted that although the term $\bar{S}$ is contained in $S$, the second term of \eqref{eq: S} may reduce or completely cancels out the signal strength in the first term. Thus, adding the regularization term $\|G_1^\top G_2 - \bar S\|_F^2$ impacts the feature learning ability of the algorithm.
\begin{algorithm}
   \caption{Approximated Multimodal Contrastive Learning}\label{alg: MMCL}
    \begin{algorithmic}
        \State Input: Data $(x_i)_{i \in [n]}$ and $(\tilde x_i)_{i \in [n]}$, rank $r \geq 1$, temperature $\tau > 0$, parameter $\nu \geq 1$, initial representations $G_1^{(0)} \in \R^{r \times d_1}$ and $G_2^{(0)} \in \R^{r \times d_2}$.
        \State
            Calculate the similarity of pairs by $s_{ij} = \langle G_1^{(0)} x_i, G_2^{(0)} \tilde x_j \rangle$. Also calculate $\beta_i$ and $\beta_{ij}$ for $i \neq j$ according to \eqref{eq: beta ij}. Compute
            \begin{align*}
                S = \frac{1}{n} \sum_{i \in [n]} \beta_i x_i \tilde x_i^\top - \frac{1}{n} \sum_{i \neq j} \beta_{ij} x_i \tilde x_j^\top.
            \end{align*}
            Perform SVD on $S$ and write $S = \sum_{j=1}^{d_1 \wedge d_2} \lambda_j u_{1j} u_{2j}^\top$ so that $\lambda_1 \geq \dots \geq \lambda_{d_1 \land d_2}$.
            Let $G_1 \in \R^{r \times d_1}$ and $G_2 \in \R^{r \times d_2}$ satisfy $G_1^\top G_2 = \sum_{j=1}^r \lambda_j u_{1j} u_{2j}^\top$.
        \State Output: $G_1$ and $G_2$.
    \end{algorithmic}
\end{algorithm}
%Intuitively speaking, if the initial representations are good enough, we expect $s_{ij}$ to be large for $(i,j) \in \mathcal{E} \setminus \mathcal{C}$, and minimization of the loss in \eqref{eq: linear loss} is essentially equivalent to performing SVD on the matrix proportional to the cross-covariance matrix $\sum x_i \tilde x_i^\top$.
%Thus, we expect that $\beta_{ij}$ is the term that "detects" whether the pair $(x_i, \tilde x_j)$ is the ground truth.
%In other words, we can say that the InfoNCE loss function with $\epsilon=0$ (\eqref{eq: linear loss}) may have the ability to align ground-truth pairs.
%However, since the sign of the second term in $S$ is negative, there arises a loss of signal strength when the rate of noisy pairs is nonzero.
%This shows that learned $G_1$ and $G_2$ should be aligned with top-$r$ singular vectors of the term inside the parenthesis.

%Let us focus on the unobserved ground truth pair $(i_1, j_1) \in \mathcal{E} \setminus \mathcal{C}$. Since $(x_{i_1}, \tilde x_{j_1})$ is the edge of the ground truth in the bipartite graph, it is desirable to have $\beta_{i_1 j_1}$ large. 
%If $G_1^\top G_2$ is close to $U_1^* U_2^{* \top}$, we can see that $s_{i, j} \approx z_{i}^\top \tilde z_{j}$.
%This makes $\beta_{ij}$ large for $(i, j) \in \mathcal{E}\setminus \mathcal{C}$ while small for  $(i_1, j_1) \in \mathcal{E}^\perp \setminus \mathcal{C}$.

We introduce an assumption for initial representations.
\begin{asm}\label{asm: initial value}
    Assume that there exist a constant $q > 0$ and some small constant $c_q = c_q(\sigma, s_1, s_2, \kappa_z^2, q) > 0$ such that 
    the initial representations $G_1^{(0)}$ and $G_2^{(0)}$ satisfy
    \begin{align*}
        &\|G_1^{(0) \top} G_2^{(0)} - q U_1^* \Sigma_z^{1/2} \Sigma_{\tilde z}^{1/2} U_2^{* \top}\|^2 \leq c_q \frac{r}{(r + r(\Sigma_\xi))(r + r(\Sigma_{\tilde \xi})) \log n}.%\label{eq: initial value}
    \end{align*}
\end{asm}

The following lemma states that when the initial representations are good enough, then Algorithm \ref{alg: MMCL} can detect the unobserved ground-truth pairs.
\begin{lem}\label{lem: edge detection}
    Suppose Assumptions \ref{asm: signal condition number}, \ref{asm: signal-to-noise ratio}, \ref{asm: asymptotics 2}, and \ref{asm: initial value} hold.
    Fix any $\gamma > 0$ and $\nu \geq 1$.
    Choose $\tau \leq C(1 + \gamma)^{-1} \sqrt{r / \log n}$, where $C > 0$ is some constant depending on $\sigma, s_1, s_2, \kappa_z^2, \kappa_{\tilde z}^2$.
    Assume that $n$ satisfies
    \begin{align}
        n \geq \frac{C_q}{c_q} \frac{(r + r(\Sigma_\xi) + r(\Sigma_{\tilde \xi}))^{5/2} \log n \sqrt{\log (n+d_1+d_2)}}{r},\label{eq: n large enough}
    \end{align}
    where $C_q = C_q(\sigma, s_1, s_2, \kappa_z^2, \kappa_{\tilde z}^2, q) > 0$ is some constant. 
    Consider applying Algorithm \ref{alg: MMCL} to the data generated from the model \ref{model: multimodal}. Then, with probability $1 - O(n^{-1})$,
    \begin{align*}
        %\min_{(i, j) \in \mathcal{E} \setminus \mathcal{C}} \beta_{ij} &\geq \frac{1}{2},\\
        \min_{(i, j) \in \mathcal{E} \setminus \mathcal{C}} \beta_{ij} &= 1 - O\qty(\frac{1}{n^\gamma}),\ \ \max_{(i, j) \not\in \mathcal{E} \cup \mathcal{C}} \beta_{ij} \lesssim \frac{1}{n^{1+\gamma}},\\
        \min_{(i, i) \in \mathcal{E} \cap \mathcal{C}} \beta_{i} &= \nu - 1 + O\qty(\frac{1}{n^\gamma}),\ \ \max_{(i, i) \in \mathcal{C} \setminus \mathcal{E}} \beta_{i} = \nu - O\qty(\frac{1}{n^\gamma}).
    \end{align*}
\end{lem}

\if0
For the learned representations, we have the following theorem.
\begin{thm}\label{thm: feature recovery via MMCL}
    Suppose Assumptions \ref{asm: asymptotics 2}, \ref{asm: signal condition number}, \ref{asm: signal-to-noise ratio}, \ref{asm: initial value} hold and $p < 1/3$.
    Fix any $\gamma > 0$.
    Choose $\tau$ as in Lemma \ref{lem: edge detection}.
    Consider applying Algorithm \ref{alg: MMCL} to the data generated from \eqref{model: multimodal}.
    Then, with probability $1 - O(n^{-1})$,
    \begin{align*}
        &\|\sin\Theta(P_r(G_1^{(t)}), U_1^*)\|_F \vee \|\sin\Theta(P_r(G_2^{(t)}), U_2^*)\|_F \quad\lesssim \sqrt{\frac{r (r + r(\Sigma_\xi) + r(\Sigma_{\tilde \xi})) \log (nd_1 + nd_2)}{n}}
    \end{align*}
    holds for all $t \geq 1$.
\end{thm}
\fi

Based on Lemma \ref{lem: edge detection}, we can show that Algorithm \ref{alg: MMCL} can recover the core features:
\begin{thm}\label{thm: feature recovery via MMCL}
    Suppose that Assumptions \ref{asm: signal condition number}, \ref{asm: signal-to-noise ratio}, \ref{asm: asymptotics 2}, and \ref{asm: initial value} hold.
    Suppose that $p_n \leq 1 - \eta$ for some constant $\eta > 0$.
    Fix any $\gamma > 1$, $\nu \geq 1.1 \eta^{-1}$. %, and $\epsilon = 0$.
    Choose $\tau$ as in Lemma \ref{lem: edge detection}.
    Let $G_1$ and $G_2$ be the representations obtained from Algorithm \ref{alg: MMCL} applied to the data generated from \eqref{model: multimodal}.
    Suppose that $n$ satisfies \eqref{eq: n large enough}.
    Then, with probability $1 - O(n^{-1})$,
    \begin{align*}
        \|\sin\Theta(P_r(G_1), U_1^*)\|_F \vee \|\sin\Theta(P_r(G_2), U_2^*)\|_F &\lesssim \sqrt{r} \wedge \sqrt{\frac{r (r + r(\Sigma_\xi) + r(\Sigma_{\tilde \xi})) \log (n+d_1 + d_2)}{n}},
    \end{align*}
    and
    \begin{align}
        \|G_1^\top G_2 - (\nu - 1 - \nu p_n) U_1^* \Sigma_z^{1/2} \Sigma_{\tilde z}^{1/2} U_2^{* \top}\| &\leq c_q \frac{r}{(r + r(\Sigma_\xi)) (r + r(\Sigma_{\tilde \xi})) \log n}.\label{eq: new initial value}
    \end{align}
\end{thm}
It is noted that approximated multimodal contrastive learning can learn representations even in the presence of noisy pairs. \eqref{eq: new initial value} implies that we can further iterate the procedure by obtained representations $G_1, G_2$ to obtain the same theoretical guarantee.

\subsection{Extension to More Than Three Modalities}\label{sec: more than three modalities}

Here we discuss the extension of MMCL to the case where data have more than three modalities.
Specifically, we observe $n$ data $(x_i^\mu)_{i=1}^n \subset \R^{d_\mu}$ from $\mu$-th modality for all $\mu \in [M]$, where $M$ is the number of modalities.
As in the main body, we focus on linear representations. We train $M$ linear representations $G_\mu \in \R^{r \times d_\mu}$ for each modality.
Since the loss \ref{eq: general loss} is designed to contrast two modalities, 
one possible extension to multiple modalities is to sum up the contrastive loss for all pairs of modalities; we define the additive multimodal contrastive loss as follows:
\begin{align}
    \mathcal{L}_{M} (G_1, \dots, G_M) \triangleq \sum_{1 \leq \mu_1 < \mu_2 \leq M} \mathcal{L}(G_{\mu_1}, G_{\mu_2}),\label{eq: general loss with more than three modalities}
\end{align}
where $s_{ij}^{\mu_1, \mu_2} \triangleq \langle G_{\mu_1} x_i^{\mu_1}, G_{\mu_2} x_j^{\mu_2} \rangle$ and
\begin{align*}
    \mathcal{L}(G_{\mu_1}, G_{\mu_2}) &= \frac{1}{2 C_n} \sum_i \phi\qty(\epsilon \psi(0) + \sum_{j: j \neq i} \psi(s_{ij}^{\mu_1, \mu_2} - s_{ii}^{\mu_1, \mu_2}))\\
    &\quad+ \frac{1}{2 C_n} \sum_i \phi\qty(\epsilon \psi(0) + \sum_{j: j \neq i} \psi(s_{ji}^{\mu_1, \mu_2} - s_{ii}^{\mu_1, \mu_2}))
    + R(G_{\mu_1}, G_{\mu_2}).
\end{align*}
Since this is a simple addition of the contrastive loss, its gradient is also a sum of the gradients. 
We minimize the loss \ref{eq: general loss with more than three modalities} via coordinate descent; given the set of $G_1^{(t)}, \dots, G_M^{(t)}$ at step $t$, we obtain $G_1^{(t+1)}, \dots, G_M^{(t+1)}$ by
\begin{align*}
    G_\mu^{(t+1)} = G_\mu^{(t)} - \iota \eval{\frac{\partial \mathcal{L}_M}{\partial G_\mu}}_{G_1 = G_1^{(t)}, \dots, G_M = G_M^{(t)}},
\end{align*}
where $\iota > 0$ is the learning rate.

Then, we have the following result as a corollary from Proposition \ref{prop: min nonlinear loss is SVD restatement}.

\begin{cor}[Corollary of Proposition \ref{prop: min nonlinear loss is SVD}]\label{cor: min nonlinear loss is SVD with more than three modalities}
    %Let $s_{ij} \triangleq \langle x_i, \tilde x_j \rangle$.
    Consider minimizing the nonlinear loss function $\mathcal{L}_M$ defined in \eqref{eq: general loss with more than three modalities} by coordinate descent.
    Suppose that the regularizer $R$ is symmetric, i.e., $R(G_{\mu_1}, G_{\mu_2}) = R(G_{\mu_2}, G_{\mu_1})$ for any $\mu_1 \neq \mu_2$. Then,
    \begin{align*}
        \frac{\partial \mathcal{L}_M}{\partial G_\mu} = -\eval{\frac{\partial}{\partial G_\mu} \sum_{\mu' \neq \mu} \tr(G_\mu S_{\mu,\mu'}(\beta) G_{\mu'}^{\top}) }_{\beta=\beta(G_1, \dots, G_M)} + \frac{\partial}{\partial G_\mu} \sum_{\mu' \neq \mu} R(G_\mu, G_{\mu'}),\ \ \ \ \ \ \mu \in [M],
    \end{align*}
    where the contrastive cross-covariance $S_{\mu,\mu'}(\beta)$ is given by:
    \begin{align*}
        &S_{\mu,\mu'}(\beta) = \frac{1}{C_n} \sum_{i=1}^n \beta_i^{\mu,\mu'} x_i^\mu x_i^{\mu'} - \frac{1}{C_n} \sum_{i\neq j} \beta_{ij}^{\mu,\mu'} x_i^\mu x_j^{\mu'},\\ 
        &\beta_{ij}^{\mu,\mu'} \triangleq \frac{\alpha_{ij}^{\mu,\mu'} + \bar \alpha_{ji}^{\mu,\mu'}}{2},\ \ \beta_i^{\mu,\mu'} \triangleq \nu \sum_{j \in [n]} \frac{\alpha_{ij}^{\mu,\mu'} + \bar \alpha_{ij}^{\mu,\mu'}}{2} - 1,
    \end{align*}
    with
    \begin{align*}
        \alpha_{ij}^{\mu,\mu'} &= \epsilon_{ij} \phi'\qty(\sum_{j' \in [n]} \epsilon_{ij} \psi(s_{ij'}^{\mu,\mu'} - \nu s_{ii}^{\mu,\mu'})) \psi'(s_{ij}^{\mu,\mu'} - \nu s_{ii}^{\mu,\mu'}),\\ 
        \bar\alpha_{ij}^{\mu,\mu'} &= \epsilon_{ij} \phi'\qty(\sum_{j' \in [n]} \epsilon_{ij} \psi(s_{j'i}^{\mu,\mu'} - \nu s_{ii}^{\mu,\mu'})) \psi'(s_{ji}^{\mu,\mu'} - \nu s_{ii}^{\mu,\mu'}),
    \end{align*}
    where $\nu \geq 1$ and $\epsilon_{ij} = 1$ for $i \neq j$ and $\epsilon_{ij} = \epsilon$ for $i = j$.
\end{cor}
Corollary \ref{cor: min nonlinear loss is SVD with more than three modalities} shows that when $R(G_\mu, G_{\mu'}) = \|G_{\mu_1}^\top G_{\mu_2}\|_F^2$, each step of gradient descent in minimizing the additive contrastive loss \ref{eq: general loss with more than three modalities} can be interpreted as maximizing the sum of SVD objectives, which is an analogy of the results in Section \ref{sec: multimodal and SVD}.

\section{Proofs}

\subsection{Proof of Lemma \ref{lem: EYM 2}}\label{proof: lem: EYM 2}

Here we prove Lemma \ref{lem: EYM 2}.
\begin{proof}
    Observe that
    \begin{align*}
        &-2 \tr(G_1 S G_2^\top) + \rho \|G_1^\top G_2\|_F^2 = \norm{ \rho^{1/2} G_1^\top G_2 - \frac{1}{\rho^{1/2}} S }_F^2 - \frac{1}{\rho} \|S\|_F^2.
    \end{align*}
    Eckart-Young-Mirsky theorem (see, for example, Theorem 2.4.8 in \cite{golub2013matrix}) concludes the proof.
\end{proof}

\subsection{Proof of Proposition \ref{prop: min nonlinear loss is SVD}}

Before going to the proof, we restate Proposition \ref{prop: min nonlinear loss is SVD} in a slightly generalized way. 
Suppose that we have parameters $\theta_1$ and $\theta_2$ such that $G_1 = G_1(\theta_1)$ and $G_2 = G_2(\theta_2)$ are smooth functions of $\theta_1$ and $\theta_2$, respectively.

Define the loss function $\mathcal{L}'$ as
\begin{align}
    &\mathcal{L}'(\theta_1, \theta_2) \triangleq \frac{1}{2 C_n} \sum_i \phi\qty(\sum_{j \in [n]} \epsilon_{ij} \psi(s_{ij} - \nu s_{ii})) + \frac{1}{2 C_n} \sum_i \phi\qty(\sum_{j \in [n]} \epsilon_{ij} \psi(s_{ji} - \nu s_{ii})) + R(\theta_1, \theta_2),\label{eq: general loss with different temp}
\end{align}
where $\nu \geq 1$ and $\epsilon_{ij} = 1$ for $i \neq j$ and $\epsilon_{ij} = \epsilon$ for $i = j$.
Recall that $s_{ij} = \langle G_1(\theta_1) x_i, G_2(\theta_j) \tilde x_j \rangle$.
The loss in \eqref{eq: general loss} corresponds to the loss in \eqref{eq: general loss with different temp} with $\nu = 1$ and $G_k = \theta_k \in \R^{r \times d_k}$ for $k=1,2$. Setting $\nu > 0$ corresponds to choosing different temperature parameters for positive pairs and negative pairs.

\begin{prop}[Restatement of Proposition \ref{prop: min nonlinear loss is SVD}]\label{prop: min nonlinear loss is SVD restatement}
    %Let $s_{ij} \triangleq \langle x_i, \tilde x_j \rangle$.
    Consider minimizing the nonlinear loss function $\mathcal{L}'(\theta_1, \theta_2)$ defined in \eqref{eq: general loss with different temp}. Then,
    \begin{align*}
        \frac{\partial \mathcal{L}'}{\partial \theta_k} = -\eval{\frac{\partial \tr(G_1 S(\beta) G_2^\top)}{\partial \theta_k}}_{\beta=\beta(G_1, G_2)} + \frac{\partial R(G_1, G_2)}{\partial \theta_k},\ \ \ \ \ \ k \in \{1,2\},
    \end{align*}
    where the contrastive cross-covariance $S(\beta)$ is given by:
    \begin{align*}
        S(\beta) &= \frac{1}{C_n} \sum_{i=1}^n \beta_i x_i \tilde x_i^\top - \frac{1}{C_n} \sum_{i\neq j} \beta_{ij} x_i \tilde x_j^\top,\ \ 
        \beta_{ij} = \frac{\alpha_{ij} + \bar \alpha_{ji}}{2},\ \ \beta_i = \nu \sum_{j \in [n]} \frac{\alpha_{ij} + \bar \alpha_{ij}}{2} - 1,
    \end{align*}
    with
    \begin{align*}
        \alpha_{ij} &= \epsilon_{ij} \phi'\qty(\sum_{j' \in [n]} \epsilon_{ij} \psi(s_{ij'} - \nu s_{ii})) \psi'(s_{ij} - \nu s_{ii}),\ \ 
        \bar\alpha_{ij} = \epsilon_{ij} \phi'\qty(\sum_{j' \in [n]} \epsilon_{ij} \psi(s_{j'i} - \nu s_{ii})) \psi'(s_{ji} - \nu s_{ii}).
    \end{align*}
    \if0
    Furthermore, if we use $s_{ij} = - \|G_1 x_i - G_2 \tilde x_j\|^2$, then
    \begin{align*}
        \frac{\partial \mathcal{L}}{\partial G_k} = -\eval{\frac{\partial \tr(G_1 S G_2^\top)}{\partial G_k}}_{\beta=\beta(G_1, G_2)} + \frac{\partial R(G_1, G_2)}{\partial G_k},\ \ \ \ \ \ k \in \{1,2\},
    \end{align*}
    \fi
\end{prop}

\begin{proof}
    Let $\theta_{2,\ell}$ be the $k$-th component of $\theta_2$.
    Observe that
    \begin{align*}
        \partial_{\theta_{2,\ell}} \mathcal{L}' &= \frac{1}{2 C_n} \sum_{i=1}^n \phi'\qty(\sum_{j=1}^n \epsilon_{ij} \psi(s_{ij}-\nu s_{ii})) \sum_{j=1}^n \epsilon_{ij} \psi'\qty(s_{ij} - \nu s_{ii}) (\partial_{\theta_{2,\ell}} s_{ij} - \nu \partial_{\theta_{2,\ell}} s_{ii})\\
        &\quad+ \frac{1}{2 C_n} \sum_{i=1}^n \phi'\qty(\sum_{j=1}^n \epsilon_{ij} \psi(s_{ji}-\nu s_{ii})) \sum_{j=1}^n \epsilon_{ij} \psi'\qty(s_{ji} - \nu s_{ii}) (\partial_{\theta_{2,\ell}} s_{ji} - \nu \partial_{\theta_{2,\ell}} s_{ii}) + \partial_{\theta_{2,\ell}} R. % + \sum_{i=1}^n \phi'(\sum_{j: j \neq i} \psi(s_{ji}-s_{ii})) \sum_{j: j \neq i} \psi'(s_{ji} - s_{ii}) (\partial_{G_2} s_{ji} - \partial_{G_2} s_{ii}).
    \end{align*}
    Define $\alpha_{ij} \triangleq \epsilon_{ij} \phi'(\sum_{j' \in [n]} \epsilon_{ij'} \psi(s_{ij'}-\nu s_{ii})) \psi'(s_{ij} - \nu s_{ii})$ and $\bar\alpha_{ij} \triangleq \epsilon_{ij} \phi'(\sum_{j' \in [n]} \epsilon_{ij'}  \psi(s_{j'i}-\nu s_{ii})) \psi'(s_{ji} - \nu s_{ii})$. %Then, $\bar\alpha_{ij} = \bar\alpha_{ji}$ since $s_{ij} = s_{ji}$. %, $\bar \bar\alpha_{ij} \triangleq \phi'(\sum_{j': j' \neq i} \psi(s_{j'i}-s_{ii})) \psi'(s_{ji} - s_{ii})$.
    %When $s(u,v) = u^\top v$,
    Then,
    \begin{align*}
        \partial_{\theta_{2,\ell}} \mathcal{L}' &= \frac{1}{2 C_n} \sum_{i=1}^n \sum_{j=1}^n \alpha_{ij} (\partial_{\theta_{2,\ell}} s_{ij} - \nu \partial_{\theta_{2,\ell}} s_{ii}) + \frac{1}{2 C_n} \sum_{i=1}^n \sum_{j=1}^n \bar \alpha_{ij} (\partial_{\theta_{2,\ell}} s_{ji} - \nu \partial_{\theta_{2,\ell}} s_{ii}) + \partial_{\theta_{2,\ell}} R.
    \end{align*}
    This further gives
    \begin{align*}
        -\partial_{\theta_{2,\ell}} \mathcal{L}' &= \frac{1}{2 C_n} \sum_{i, j \in [n]} \qty[ \nu (\alpha_{ij} + \bar\alpha_{ij}) \partial_{\theta_{2,\ell}} s_{ii} - \alpha_{ij} \partial_{\theta_{2,\ell}} s_{ij} - \bar \alpha_{ij} \partial_{\theta_{2,\ell}} s_{ji}] + \partial_{\theta_{2,\ell}} R\\
        %&= \frac{1}{C_n} G_1 \sum_{i \in [n]} \qty( \nu \sum_{j \in [n]} \frac{\alpha_{ij} + \bar\alpha_{ij}}{2} - 1) x_i \tilde x_i^\top - \frac{1}{C_n} G_1 \sum_{i \neq j} \qty(\alpha_{ij} x_i \tilde x_j^\top + \bar \alpha_{ij} x_j \tilde x_i^\top ) + \partial_{G_2} R\\
        %&= G_1 \qty( \sum_{i \neq j} \frac{\bar\alpha_{ij}}{\|G_1 x_i\| \|G_2 \tilde x_j\|} (x_j \tilde x_j^\top - x_i \tilde x_j^\top) + \sum_{i \neq j} \frac{\bar\alpha_{ji}'}{\|G_1 x_i\| \|G_2 \tilde x_j\|} (x_j \tilde x_j^\top - x_i \tilde x_j^\top) )\\
        &= \frac{1}{C_n} \sum_i \qty( \nu \sum_{j \in [n]} \frac{\alpha_{ij} + \bar\alpha_{ij}}{2} - 1 ) \partial_{\theta_{2,\ell}} s_{ii} - \frac{1}{C_n} \sum_{i \neq j} \frac{\alpha_{ij} + \bar\alpha_{ji}}{2} \partial_{\theta_{2,\ell}} s_{ij} + \partial_{\theta_{2,\ell}} R\\
        &= \frac{1}{C_n} \sum_i \beta_i \partial_{\theta_{2,\ell}} s_{ii} - \frac{1}{C_n} \sum_{i \neq j} \beta_{ij} \partial_{\theta_{2,\ell}} s_{ij} + \partial_{\theta_{2,\ell}} R.
    \end{align*}
    Since $\partial_{\theta_{2,\ell}} s_{ij} = \partial_{\theta_{2,\ell}} \tr(G_1 x_i \tilde x_j^\top G_2^\top)$,
    when $\beta_i$ and $\beta_{ij}$ do not depend on $\theta_1$ and $\theta_{2,\ell}$,
    we obtain
    \begin{align*}
        -\partial_{\theta_{2,\ell}} \mathcal{L}' &= \partial_{\theta_{2,\ell}} \tr( G_1 \qty(\frac{1}{C_n} \sum_i \beta_i x_i \tilde x_i^\top)  G_2^\top ) - \partial_{\theta_{2,\ell}} \tr( G_1 \qty(\frac{1}{C_n} \sum_{i \neq j} \beta_{ij} x_i \tilde x_j^\top) G_2^\top ) + \partial_{\theta_{2,\ell}} R.
    \end{align*}
    This yields the claim for $k=2$ case. By symmetry, we have a similar result for $k=1$.
\end{proof}

\subsection{Proof of Theorem \ref{thm: linear loss}}\label{sec: thm: linear loss}

\if0
\begin{thm}\label{thm: linear loss master}
    Suppose we a set of pairs $(x_i, \tilde x_i)_{i=1}^{n}$ generated according to the model \ref{model: multimodal}.
    Suppose Assumption \ref{asm: signal-to-noise ratio} holds.
    Let $G_1$ and $G_2$ be the solution to the minimization of the loss \ref{eq: general loss} with $\phi$ and $\psi$ set as identity functions. 
    %Also set $\rho^2 = 1$.
    If $r \log(nd_1 + nd_2) \lesssim n$, then, with probability $1 - O(n^{-1})$, we have
    \begin{align*}
        \norm{G_1^\top G_2 - \frac{1}{\rho} U_1^* \Sigma_z^{1/2} \Sigma_{\tilde z}^{1/2} U_2^{* \top}} \lesssim \sqrt{\frac{r \log (nd_1 + nd_2)}{n}}.
    \end{align*}
\end{thm}
\fi
We restate Theorem \ref{thm: linear loss}.
\begin{thm}[Restatement of Theorem \ref{thm: linear loss}]\label{thm: linear loss master}
    Suppose that we have a collection of pairs $(x_i, \tilde x_i)_{i=1}^{n}$ generated according to the model \ref{model: multimodal}.
    Suppose Assumptions \ref{asm: signal condition number} and \ref{asm: signal-to-noise ratio} hold.
    Let $G_1$ and $G_2$ be the solution to minimizing the loss \ref{eq: linear loss}.
    %Also set $\rho^2 = 1$.
    Then, with probability $1 - O(n^{-1})$, there exists some constant $C = C(\sigma, s_1, s_2, \kappa_z^2, \kappa_{\tilde z}^2) > 0$ such that
    \begin{align*}
        \norm{G_1^\top G_2 - \frac{1}{\rho} U_1^* \Sigma_z^{1/2} \Sigma_{\tilde z}^{1/2} U_2^{* \top}} \leq \frac{C}{\rho} \sqrt{\frac{ (r + r(\Sigma_\xi) + r(\Sigma_{\tilde \xi}) ) \log (n+d_1+d_2)}{n}},
    \end{align*}
    and
    \begin{align*}
        &\|\sin\Theta(P_r(G_1), U_1^*)\|_F \vee \|\sin\Theta(P_r(G_2), U_2^*)\|_F\\
        &\quad\lesssim \sqrt{r} \wedge \frac{\sqrt{r}}{1 - p_n}\sqrt{\frac{(r + r(\Sigma_\xi) + r(\Sigma_{\tilde \xi})) \log (n + d_1 + d_2)}{n}} \\
        &\quad\quad\times\qty(1 + \frac{1}{1 - p_n}\sqrt{\frac{(r + r(\Sigma_\xi) + r(\Sigma_{\tilde \xi})) \log (n+d_1 + d_2)}{n}}).
    \end{align*}
\end{thm}
Let $\varepsilon_n \triangleq \sqrt{(r + r(\Sigma_\xi) + r(\Sigma_{\tilde \xi})) \log(n+d_1+d_2)/n}$.
From Theorem \ref{thm: linear loss master}, we can see that the feature recovery bound becomes a trivial bound $\sqrt{r}$ when $1 - p_n \lesssim \varepsilon_n$. Otherwise the feature recovery ability is bounded by $(1-p_n)^{-1} \sqrt{r} \varepsilon_n$. This implies that even if the portion of noisy pairs grows, MMCL can still recover the core features as $n \to \infty$ as long as the growth is mild.

\begin{cor}
    Assume the same conditions as in Theorem \ref{thm: linear loss master}.
    If $p_n \leq 1 - \eta$ for some $\eta \in (0, 1]$, then
    \begin{align*}
        \|\sin\Theta(P_r(G_1), U_1^*)\|_F \vee \|\sin\Theta(P_r(G_2), U_2^*)\|_F &\lesssim \sqrt{r} \wedge \frac{1}{\eta} \sqrt{\frac{r (r + r(\Sigma_\xi) + r(\Sigma_{\tilde \xi})) \log (n+d_1+d_2)}{n}}.
    \end{align*}
    %where in the second inequality we used Assumption \ref{asm: signal condition number} and $p_n \leq 1 - \eta < 1$.
\end{cor}

\begin{proof}
    Since the loss function \ref{eq: general loss} can be rewritten as
    \begin{align*}
        &- \tr(G_1 \qty(\frac{1}{n-1} \sum_{i=1}^n (x_i - \bar x) (\tilde x_i - \bar{\tilde x})^\top) G_2^\top) + (\rho^2/2) \|G_1^\top G_2\|_F^2\\
        &\quad= (\rho^2/2) \norm{G_1^\top G_2 - \frac{1}{\rho^2}\frac{1}{n-1} \sum_{i=1}^n (x_i - \bar x) (\tilde x_i - \bar{\tilde x})^\top}_F^2 - \frac{1}{2\rho^2} \norm{\frac{1}{n-1} \sum_{i=1}^n (x_i - \bar x) (\tilde x_i - \bar{\tilde x})^\top}_F^2,
    \end{align*}
    where $\bar x = (1/n) \sum_i x_i$ and $\bar{\tilde x} = (1/n) \sum_i \tilde x_i$.
    By Eckart-Young-Mirsky theorem (see, for example, Theorem 2.4.8 in \cite{golub2013matrix}), we have $G_1^\top G_2 = \SVD_r(\rho^{-1} (n-1)^{-1} \sum_{i=1}^n (x_i - \bar x) (\tilde x_i - \bar{\tilde x})^\top )$. 
    For notational brevity, define $\Sigma = \Sigma_z^{1/2} \Sigma_{\tilde z}^{1/2}$ and $\bar S \triangleq (n-1)^{-1} \sum_{i=1}^n (x_i - \bar x) (\tilde x_i - \bar{\tilde x})^\top$.
    Observe that
    \begin{align}
        \norm{\SVD_r(\bar S) - U_1^* \Sigma U_2^{* \top}}\nonumber &\leq \norm{\SVD_r(\bar S) - \bar S} + \norm{\bar S - U_1^* \Sigma U_2^{* \top}}\nonumber\\
        &= \lambda_{r+1}\qty(\bar S) + \norm{\bar S - U_1^* \Sigma U_2^{* \top}}\nonumber\\
        &\leq \lambda_{r+1}\qty(U_1^* \Sigma U_2^{* \top}) + 2 \norm{\bar S - U_1^* \Sigma U_2^{* \top}}\\
        &= 2 \norm{\bar S - U_1^* \Sigma U_2^{* \top}}.\label{eq: SVD concentration}
    \end{align}
    \if0
    \begin{align*}
        \norm{\SVD_r\qty(\frac{1}{n} \sum_{i=1}^n x_i \tilde x_i^\top) - U_1^* \Sigma U_2^{* \top}} &\leq \norm{\SVD_r\qty(\frac{1}{n} \sum_{i=1}^n x_i \tilde x_i^\top) - \frac{1}{n} \sum_{i=1}^n x_i \tilde x_i^\top} + \norm{\frac{1}{n} \sum_{i=1}^n x_i \tilde x_i^\top - U_1^* \Sigma U_2^{* \top}}\\
        &= \lambda_{r+1}\qty(\frac{1}{n} \sum_{i=1}^n x_i \tilde x_i^\top) + \norm{\frac{1}{n} \sum_{i=1}^n x_i \tilde x_i^\top - U_1^* \Sigma U_2^{* \top}}\\
        &\leq \lambda_{r+1}\qty(U_1^* \Sigma U_2^{* \top}) + 2 \norm{\frac{1}{n} \sum_{i=1}^n x_i \tilde x_i^\top - U_1^* \Sigma_z^{1/2} \Sigma_{\tilde z}^{1/2} U_2^{* \top}}\\
        &= 2 \norm{\frac{1}{n} \sum_{i=1}^n x_i \tilde x_i^\top - U_1^* \Sigma_z^{1/2} \Sigma_{\tilde z}^{1/2} U_2^{* \top}}.
    \end{align*}
    \fi
    Note that
    \if0
    \begin{align*}
        \frac{1}{n-1} \sum_{i=1}^n (x_i - \bar x) (\tilde x_i - \bar{\tilde x})^\top = \frac{1}{n} \sum_{i=1}^n x_i \tilde x_i^\top - \frac{1}{n(n-1)} \sum_{i \neq j} x_i \tilde x_j^\top.
    \end{align*}
    \fi
    \begin{align*}
        \frac{1}{n-1} \sum_{i=1}^n (x_i - \bar x) (\tilde x_i - \bar{\tilde x})^\top = \frac{1}{n-1} \sum_{i=1}^n x_i \tilde x_i^\top - \frac{n}{n-1} \bar x \bar{\tilde x}^\top.
    \end{align*}
    Since $x_i$ is independent sub-Gaussian random vectors with parameter $\sqrt{\sigma_z^2 + \sigma_{\xi}^2}$, from Hoeffding bound (see, for example, Proposition 2.5 in \cite{wainwright2019high}), it holds with probability $1 - O(n^{-1})$ that
    \begin{align}
        \|\bar x\| \leq \sqrt{\frac{2(\sigma_z^2 + \sigma_\xi^2) \log (nd_1)}{n}} \lesssim \sqrt{\frac{ \log (nd_1)}{n}},\label{eq: x bar}
    \end{align}
    where the last inequality follows from Assumption \ref{asm: signal-to-noise ratio} and $\sigma_z^2 + \sigma_\xi^2 \leq \sigma (1 + s_1^{-1})$. A similar bound holds for $\bar{\tilde x}$.
    
    For the term $(n-1)^{-1} \sum_i x_i \tilde x_i^\top$, observe that
    \begin{align*}
        \frac{1}{n} \sum_{i \in [n]} x_i \tilde x_i^\top = \frac{1}{n} \sum_{(i, i) \in \mathcal{C} \cap \mathcal{E}} x_i \tilde x_i^\top + \frac{1}{n} \sum_{(i, i) \in \mathcal{C} \setminus \mathcal{E}} x_i \tilde x_i^\top \triangleq T_1 + Q_1.
    \end{align*}
    Suppose for a moment that $0 < m < n$.
    
    For the term $T_1$, note that
    \begin{align*}
        T_1 &= U_1^* \qty(\frac{1}{n} \sum_{(i, i) \in \mathcal{C} \cap \mathcal{E}} z_i \tilde z_i^\top) U_2^{* \top} + \frac{1}{n} \sum_{(i, i) \in \mathcal{C} \cap \mathcal{E}} U_1^* z_i \tilde \xi_i^\top + \frac{1}{n} \sum_{(i, i) \in \mathcal{C} \cap \mathcal{E}} \xi_i \tilde z_i^\top U_2^{* \top} + \frac{1}{n} \sum_{(i, i) \in \mathcal{C} \cap \mathcal{E}} \xi_i \tilde \xi_i^\top.
    \end{align*}
    Using Proposition \ref{prop: cross-covariance concentration} and Assumption \ref{asm: signal-to-noise ratio}, the following bound holds with probability $1 - O(n^{-1})$:
    \begin{align*}
        \|\frac{1}{m} \sum_{(i, i) \in \mathcal{C} \cap \mathcal{E}} z_i \tilde z_i^\top - \Sigma_z^{1/2} \Sigma_{\tilde z}^{1/2}\| &\lesssim (\|\Sigma_z\| \vee \|\Sigma_{\tilde z}\|) \sqrt{\frac{r \log (nr)}{m}},\\
        \|\frac{1}{m} \sum_{(i, i) \in \mathcal{C} \cap \mathcal{E}} z_i \tilde \xi_i^\top\| &\lesssim \|\Sigma_z\|^{1/2} \|\Sigma_{\tilde z}\|^{1/2} \sqrt{\frac{(r + r(\Sigma_{\tilde \xi})) \log (nr + nd_2)}{m}},\\
        \|\frac{1}{m} \sum_{(i, i) \in \mathcal{C} \cap \mathcal{E}} \xi_i \tilde z_i^\top\| &\lesssim \|\Sigma_z\|^{1/2} \|\Sigma_{\tilde z}\|^{1/2} \sqrt{\frac{(r + r(\Sigma_\xi)) \log (nr + nd_1)}{m}},\\
        \|\frac{1}{m} \sum_{(i, i) \in \mathcal{C} \cap \mathcal{E}} \xi_i \tilde z_i^\top\| &\lesssim \|\Sigma_z\|^{1/2} \|\Sigma_{\tilde z}\|^{1/2} \sqrt{\frac{(r(\Sigma_\xi) + r(\Sigma_{\tilde \xi})) \log (nd_1 + nd_2)}{m}}.
    \end{align*}
    Thus, with probability $1 - O(n^{-1})$,
    \begin{align}
        \|T_1 - (1-p_n) U_1^* \Sigma_{z}^{1/2} \Sigma_{\tilde z}^{1/2} U_2^{* \top}\| \lesssim \sqrt{1-p_n} \sqrt{\frac{(r + r(\Sigma_\xi) + r(\Sigma_{\tilde \xi})) \log (nd_1 + nd_2)}{n}}.\label{eq: T1}
    \end{align}
    
    Next we bound $\|Q_1\|$. Note that $Q_1$ is a sum of $n-m$ independent random variables. Using Proposition \ref{prop: cross-covariance concentration},
    \begin{align}
        \|\frac{n}{n-m} Q_1\| &\lesssim \qty( \tr(\Sigma_{\tilde x}) \|\Sigma_x\| \vee \tr(\Sigma_x) \|\Sigma_{\tilde x}\|)^{1/2} \sqrt{\frac{\log(nd_1 + nd_2)}{n - m}}\\
        &\leq \qty(\|\Sigma_{\tilde z}\| (r + s_2^{-2} r(\Sigma_{\tilde \xi}))^{1/2} (1 + s_2^{-2})^{1/2} \vee \|\Sigma_z\| (r + s_1^{-2} r(\Sigma_\xi))^{1/2} (1 + s_1^{-2})^{1/2})\\ &\quad\times \sqrt{\frac{\log(nd_1 + nd_2)}{n - m}}\\
        &\lesssim \sqrt{\frac{(r + r(\Sigma_\xi) + r(\Sigma_{\tilde \xi})) \log(nd_1 + nd_2)}{n - m}}\label{eq: Q1}
    \end{align}
    holds with probability at least $1 - n^{-1}$, where the last inequality holds from Assumption \ref{asm: signal-to-noise ratio}.
    
    \if0
    Since $E[x_i \tilde x_i^\top] = U_1^* \Sigma_z^{1/2} E[w_i \tilde w_i^\top] \Sigma_{\tilde z}^{1/2} U_2^{* \top} = U_1^* \Sigma_z^{1/2} \Sigma_{\tilde z}^{1/2} U_2^{* \top}$ for $(i, i) \in \mathcal{C} \cap \mathcal{E}$, by Proposition \ref{prop: cross-covariance concentration},
    \begin{align*}
        \norm{\frac{1}{m} \sum_{(i, i) \in \mathcal{C} \cap \mathcal{E}} x_i \tilde x_i^\top - U_1^* \Sigma U_2^{* \top}} \lesssim \sqrt{\frac{r \log (md_1 + md_2)}{m}}.
    \end{align*}
    Also, since $E[x_i \tilde x_i^\top] = 0$ for $(i, i) \in \mathcal{C} \setminus \mathcal{E}$,
    \begin{align*}
        \norm{\frac{1}{n-m} \sum_{(i, i) \in \mathcal{C} \cap \mathcal{E}} x_i \tilde x_i^\top} \lesssim \sqrt{\frac{r \log ((n-m)d_1 + (n-m)d_2)}{n-m}}.
    \end{align*}
    Thus,
    \begin{align*}
        \norm{\frac{1}{n} \sum_{i \in [n]} x_i \tilde x_i^\top - (1 - p_n) U_1^* \Sigma U_2^{* \top}} 
        &\leq \norm{\frac{1}{n} \sum_{(i, i) \in \mathcal{C} \cap \mathcal{E}} x_i \tilde x_i^\top - \frac{m}{n} U_1^* \Sigma U_2^{* \top}} + \norm{\frac{1}{n} \sum_{(i, i) \in \mathcal{C} \setminus \mathcal{E}} x_i \tilde x_i^\top}\\
        &\lesssim (\sqrt{1-p_n} + \sqrt{p_n}) \sqrt{\frac{r \log (nd_1 + nd_2)}{n}}.
    \end{align*}
    \fi
    Therefore, combined with \eqref{eq: x bar},
    \begin{align*}
        \norm{\bar S - (1 - p_n) U_1^* \Sigma U_2^{* \top} } &\lesssim \sqrt{\frac{ (r + r(\Sigma_\xi) + r(\Sigma_{\tilde \xi}) ) \log (nd_1 + nd_2)}{n}} + \frac{1-p_n}{n-1} + \sqrt{\frac{\log(nd_1)}{n}}\\
        &\lesssim \sqrt{\frac{ (r + r(\Sigma_\xi) + r(\Sigma_{\tilde \xi}) ) \log (nd_1 + nd_2)}{n}}
    \end{align*}
    holds with probability at least $1 - O(n^{-1})$.
    If $m = 0$ or $m = n$, a similar argument gives the same bound with probability $1 - O(n^{-1})$.
    From \eqref{eq: SVD concentration}, we obtain
    \begin{align*}
        \norm{G_1^\top G_2 - \frac{1}{\rho}U_1^* \Sigma U_2^{* \top}} 
        %\leq \norm{G_1^\top G_2 - U_1^* \Sigma_z^{1/2} \Sigma_{\tilde z}^{1/2} U_2^{* \top}} 
        \lesssim \frac{1}{\rho}
        \sqrt{\frac{ (r + r(\Sigma_\xi) + r(\Sigma_{\tilde \xi}) ) \log (nd_1 + nd_2)}{n}}
    \end{align*}
    with probability at least $1 - O(n^{-1})$.
    
    Define $\Sigma' \triangleq (1-p_n) U_1^* \Sigma U_2^{* \top}$. Since $\lambda_{\max}(\Sigma) \leq (1-p_n)$ and $\lambda_{\min}(\Sigma') \geq (1-p_n) \lambda_{\min}(\Sigma_z^{1/2} \Sigma_{\tilde z}^{1/2})$, from Theorem 3 in \cite{yu2015useful} with $s \leftarrow r$, $r \leftarrow 1$,
    \begin{align*}
        &\|\sin\Theta(P_r(G_1), U_1^*)\|_F \vee \|\sin\Theta(P_r(G_2), U_2^*)\|_F\\
        &\leq \sqrt{r} \wedge \frac{2(2(1-p_n)\lambda_{\max}(\Sigma_{z}^{1/2} \Sigma_{\tilde z}^{1/2}) + \|\bar S - \Sigma'\|) r^{1/2} \|\bar S - \Sigma'\| }{(1 - p_n)^2 \lambda_{\min}^2(\Sigma_{z}^{1/2} \Sigma_{\tilde z}^{1/2})}\\
        &\lesssim \sqrt{r} \wedge \sqrt{r} \qty(1 + \frac{\|\bar S - \Sigma'\|}{1 - p_n}) \frac{\|\bar S - \Sigma'\|}{1 - p_n}.
        %&\quad\lesssim \sqrt{r} \wedge \sqrt{\frac{r (r + r(\Sigma_\xi) + r(\Sigma_{\tilde \xi})) \log (nd_1 + nd_2)}{n}},
    \end{align*}
    where in the second inequality we used Assumption \ref{asm: signal condition number}.
    $\log(nd_1 + nd_2) \lesssim \log(n + d_1 + d_2)$ concludes the proof.
\end{proof}

\if0
Before going to the proof, we define the multimodal contrastive loss as
\begin{align*}
    &\mathcal{L}'(G_1, G_2) \triangleq - \frac{1}{2N} \sum_{(i, j) \in \mathcal{E}^u} \log \frac{e^{s_{ii}^u / \tau}}{\sum_{j \in [N]} e^{s_{ji}^u/\tau}} - \frac{1}{2N} \sum_{(i, j) \in \mathcal{E}^u} \log \frac{e^{s_{ii}^u / \tau}}{\sum_{j \in [N]} e^{s_{ij}^u/\tau}} + R(G_1, G_2).
\end{align*}\label{eq: MMCL}
\fi

\subsection{Proof of Lemma \ref{lem: edge detection}}

\if0
\begin{lem}[Restatement of Lemma \ref{lem: edge detection}]\label{lem: edge detection restatement}
    Suppose Assumptions \ref{asm: asymptotics 2}, \ref{asm: signal condition number}, \ref{asm: signal-to-noise ratio} hold for some constant $\eta > 0$.
    Fix any $\gamma > 0$ and $\nu \geq 1$
    Choose $\tau \leq C(1 + \gamma)^{-1} \sqrt{r / \log n}$, where $C > 0$ is some constant depending on $\sigma, s_1, s_2, \kappa_z^2, \kappa_{\tilde z}^2$.
    Suppose that there exist a constant $q' > 0$ and some small constant $c_{q'} = c_{q'}(\sigma, s_1, s_2, \kappa_z^2, q') > 0$ such that $G_1^{(t)}$ and $G_2^{(t)}$ satisfy
    \begin{align*}
        \|G_1^{(t) \top} G_2^{(t)} - q' U_1^* \Sigma_z^{1/2} \Sigma_{\tilde z}^{1/2} U_2^{* \top}\| \leq c_{q'} \frac{r}{(r + r(\Sigma_\xi))(r + r(\Sigma_{\tilde \xi}) \log n}.
    \end{align*}
    Consider applying Algorithm \ref{alg: MMCL} to the data generated from the model \ref{model: multimodal}. Then, with probability $1 - O(n^{-1})$,
    \begin{align*}
        %\min_{(i, j) \in \mathcal{E} \setminus \mathcal{C}} \beta_{ij} &\geq \frac{1}{2},\\
        \min_{(i, j) \in \mathcal{E} \setminus \mathcal{C}} \beta_{ij}^{(t)} &= 1 - O\qty(\frac{1}{n^\gamma}),\ \ \max_{(i, j) \not\in \mathcal{E} \cup \mathcal{C}} \beta_{ij}^{(t)} \lesssim \frac{1}{n^{1+\gamma}},\\
        \min_{(i, i) \in \mathcal{E} \cap \mathcal{C}} \beta_{i}^{(t)} &= \nu-1 + O\qty(\frac{1}{n^\gamma}),\ \ \max_{(i, i) \in \mathcal{C} \setminus \mathcal{E}} \beta_{ij}^{(t)} = \nu - O\qty(\frac{1}{n^\gamma}).
    \end{align*}
\end{lem}
\fi

Before going to the proof of Lemma \ref{lem: edge detection}, we introduce the following result.
\begin{lem}\label{lem: similarity}
    Suppose Assumptions \ref{asm: signal condition number}, \ref{asm: signal-to-noise ratio}, \ref{asm: asymptotics 2} and \ref{asm: initial value} hold. Let $(x_i, \tilde x_i)_{i=1}^n$ be the paired data generated from the model described in Section \ref{sec: results of MMCL}. Let $\mathcal{E}$ be the set of ground-truth pairs.
    Let $s_{ij} = \langle G_1^{(0)} x_i, G_2^{(0)} \tilde x_i \rangle$ be the similarity evaluated by the initial representations.
    Then, there exists some constant $C' = C'(\sigma, s_1, s_2, \kappa_z^2, \kappa_{\tilde z}^2, q) > 0$ satisfying
    \begin{align*}
        \min_{(i_1, j_1) \in \mathcal{E}} \qty(s_{i_1 j_1} - \max_{j: (i_1, j) \not\in \mathcal{E}} s_{i_1, j} \vee \max_{i: (i, j_1) \not\in \mathcal{E}} s_{i, j_1}) \geq C' \sqrt{r \log n}.
    \end{align*}
\end{lem}

Using Lemma \ref{lem: similarity}, we prove Lemma \ref{lem: edge detection}.
\begin{proof}
    Fix any $(i_1, j_1) \in \mathcal{E} \setminus \mathcal{C}$.
    From \eqref{eq: beta ij},
    \begin{align*}
        \beta_{i_1 j_1} &= \frac{1}{2} \frac{\exp(s_{i_1 j_1}/\tau)}{\sum_{j': j' \in [n]} \exp(s_{i_1 j'}/\tau)} + \frac{1}{2}\frac{\exp(s_{i_1 j_1}/\tau)}{\sum_{i': i' \in [n]} \exp(s_{i' j_1}/\tau)}.
    \end{align*}
    
    \if0
    Since the loss function is InfoNCE with $\epsilon = 1$, that is, $\phi(x) = \tau \log(1 + x)$ and $\psi(x) = e^{x/\tau}$,
    \begin{align*}
        \alpha_{ij} &= \frac{\exp(s_{ij}/\tau)}{\exp(\nu s_{ii}/\tau) + \sum_{j': j' \neq i} \exp(s_{ij'}/\tau)},\ \
        \bar\alpha_{ij} = \frac{\exp(s_{ji}/\tau)}{\exp(\nu s_{ii}/\tau) + \sum_{j': j' \neq i} \exp(s_{j' i}/\tau)}.
    \end{align*}
    \fi
    
    Taking any $\delta \leq C' \sqrt { r \log n } / (\gamma + 1)$ and choosing the temperature parameter as $\tau = \delta/\log n$ gives,
    \begin{align*}
        \frac{e^{s_{i_1 j_1}/\tau}}{\sum_{j': j' \in [n]} e^{s_{i_1 j'}/\tau}}
        &\geq \frac{n^{s_{i_1 j_1}/\delta}}{n^{\max_{j': j' \neq j_1} s_{i_1 j'}/\delta + 1} + n^{s_{i_1 j_1}/\delta}}
        = \frac{1}{n^{- (s_{i_1 j_1}/\delta - \max_{j': j' \neq j_1} s_{i_1 j'}/\delta - 1)} + 1}.
        %&\geq \frac{n^{s_{i_1 j_1}/\delta}}{(1 + \epsilon n^{-1}) n^{\max_{j': j' \neq j_1} s_{i_1 j'}/\delta + 1} + n^{s_{i_1 j_1}/\delta}}\\
        %&= \frac{1}{(1 + \epsilon n^{-1}) n^{- (s_{i_1 j_1}/\delta - \max_{j': j' \neq j_1} s_{i_1 j'}/\delta - 1)} + 1}.
    \end{align*}
    %\begin{align*}
    %    \beta_{i_1 j_1} &= \frac{e^{s_{i_1 j_1}/\tau}}{\sum_{j': j' \neq i_1 } e^{s_{i_1 j'}/\tau}} \geq \frac{n^{s_{i_1 j_1}/\delta}}{n^{\max_{j': j' \neq i_1, j_1} s_{i_1 j'}/\delta + 1} + n^{s_{i_1 j_1}/\delta}} = \frac{1}{1 + n^{- (s_{i_1 j_1}/\delta - \max_{j': j' \neq i_1, j_1} s_{i_1 j'}/\delta - 1)}}.
    %\end{align*}
    From the above argument,
    we can see that
    \begin{align*}
        s_{i_1 j_1}/\delta - \max_{j': j' \neq i_1} s_{i_1 j'}/\delta &\geq \gamma + 1
    \end{align*}
    with probability $1 - O(n^{-1})$.
    %We can assume that the left hand side is greater than $\gamma$ without loss of generality.
    Hence
    %\begin{align*}
    %    \beta_{i_1 j_1} \gtrsim \frac{1}{1 + n^{-\gamma}} = 1 - O(n^{-\gamma})
    %\end{align*}
    \begin{align*}
        %\alpha_{i_1 j_1} \gtrsim \frac{1}{(1 + \epsilon n^{-1}) n^{-\gamma} + 1} = 1 - O(n^{-\gamma} \vee \epsilon n^{-(\gamma+1)})
        \frac{e^{s_{i_1 j_1}/\tau}}{\sum_{j': j' \in [n]} e^{s_{i_1 j'}/\tau}} = 1 - O(n^{-\gamma})
    \end{align*}
    holds uniformly over all $(i_1, j_1) \in \mathcal{E}\setminus \mathcal{C}$.
    Similarly, we have
    \begin{align*}
        %\bar \alpha_{j_1 i_1} \gtrsim 1 - O(n^{-\gamma} \vee \epsilon n^{-(\gamma+1)})
        \frac{\exp(s_{i_1 j_1}/\tau)}{\sum_{i': i' \in [n]} \exp(s_{i' j_1}/\tau)} = 1 - O(n^{-\gamma})
    \end{align*}
    holds uniformly over all $(i_1, j_1) \in \mathcal{E}\setminus \mathcal{C}$. These give %$\min_{(i_1, j_1) \in \mathcal{E} \setminus \mathcal{C}} \beta_{i_1 j_1} \gtrsim 1 - O(n^{-\gamma} \vee \epsilon n^{-(\gamma+1)})$.
    $\min_{(i_1, j_1) \in \mathcal{E} \setminus \mathcal{C}} \beta_{i_1 j_1} = 1 - O(n^{-\gamma})$.
    
    For any $(i_1, j_1) \not\in \mathcal{E} \cup \mathcal{C}$, take another node $j_2$ satisfying $(i_1, j_2) \in \mathcal{E} \setminus\mathcal{C}$. Then, by a similar argument
    \begin{align*}
        %\alpha_{i_1 j_1} &= \frac{e^{s_{i_1 j_1}/\tau}}{\epsilon e^{s_{i_1 i_1} / \tau} + \sum_{j': j' \neq i_1 } e^{s_{i_1 j'}/\tau}} \leq \frac{n^{s_{i_1 j_1}/\delta + 1} n^{-1}}{\epsilon n^{s_{i_1 i_1}/\delta} + n^{s_{i_1 j_2}/\delta}}\\
        %&= \frac{n^{s_{i_1 j_1}/\delta + 1} n^{-1}}{n^{s_{i_1 j_2}/\delta} (1 + \epsilon n^{-1} n^{s_{i_1 i_1}/\delta + 1 - s_{i_1 j_2}/\delta})} \lesssim \frac{1}{n^{1+\gamma}} \wedge \frac{1}{\epsilon n}
        \frac{e^{s_{i_1 j_1}/\tau}}{\sum_{j': j' \in [n]} e^{s_{i_1 j'}/\tau}} \leq \frac{n^{s_{i_1 j_1}/\delta + 1} n^{-1}}{n^{s_{i_1 j_2}/\delta}} \leq \frac{1}{n^{1+\gamma}}
        %\frac{1}{n^{s_{i_1 j_2}/\delta - \max_{j: j \neq i_1, j_2} s_{i_1 j}/\delta - 1}} \frac{1}{n} \lesssim \frac{1}{n^{1 + \gamma}}
    \end{align*}
    holds with probability $1 - O(n^{-1})$. Similarly,
    \begin{align*}
        %\bar \alpha_{j_1 i_1} \gtrsim 1 - O(n^{-\gamma} \vee \epsilon n^{-(\gamma+1)})
        \frac{\exp(s_{i_1 j_1}/\tau)}{\sum_{i': i' \in [n]} \exp(s_{i' j_1}/\tau)} \lesssim n^{-\gamma}.
    \end{align*}
    Since this bound is uniform over all $(i_1, j_1) \not\in \mathcal{E} \cup \mathcal{C}$, we obtain $\max_{(i_1, j_1) \not\in \mathcal{E} \cup \mathcal{C}} \beta_{i_1 j_1} \leq n^{-(1 + \gamma)}$.
    
    \if0
    For $\beta_{i_1}$, note that we can rewrite it as
    \begin{align*}
        \beta_{i_1} &= \frac{1}{2} \sum_{j: j \neq i_1} \frac{e^{s_{i_1 j}/\tau}}{\sum_{j': j' \neq i_1 } e^{s_{i_1 j'}/\tau}} + \frac{1}{2} \sum_{j: j \neq i_1} \frac{e^{s_{j i_1}/\tau}}{\sum_{j': j' \neq i_1 } e^{s_{j' i_1}/\tau}} \equiv 1.
    \end{align*}
    \fi
    For $\beta_{i_1}$, note that we can rewrite it as
    \begin{align*}
        \beta_{i_1} &= \nu-1 + \frac{1}{2} \qty(\frac{\sum_{j: j \neq i_1} e^{s_{i_1 j}/\tau}}{\sum_{j: j \in [n]} e^{s_{i_1 j}/\tau}} + \frac{\sum_{j: j \neq i_1} e^{s_{j i_1}/\tau}}{\sum_{j: j \in [n]} e^{s_{j i_1}/\tau}} ).
    \end{align*}
    %When $= 0$, $\beta_{i} \equiv 1$ for all $i \in [n]$.
    Similar to the above arguments,
    if $(i_1, i_1) \in \mathcal{E}$,
    \begin{align*}
        \frac{\sum_{j: j \neq i_1} e^{s_{i_1 j}/\tau}}{\sum_{j: j \in [n]} e^{s_{i_1 j}/\tau}} \lesssim \frac{n^{\max_{j: j \neq i_1} s_{i_1 j}/\delta + 1}}{n^{s_{i_1 i_1} / \delta}} \lesssim n^{-\gamma}.
    \end{align*}
    Similarly,
    \begin{align*}
        \frac{\sum_{j: j \neq i_1} e^{s_{j i_1}/\tau}}{\sum_{j: j \in [n]} e^{s_{j i_1}/\tau}} \lesssim n^{-\gamma}.
    \end{align*}
    Thus $\beta_{i1} = \nu - 1 + O(n^{-\gamma})$.
    If $(i_1, i_1) \not\in \mathcal{E}$, there exists some $j_1 \neq i_1$ such that
    \begin{align*}
        \frac{\sum_{j: j \neq i_1} e^{s_{i_1 j}/\tau}}{\sum_{j: j \in [n]} e^{s_{i_1 j}/\tau}} 
        &\geq \frac{e^{s_{i_1 j_1}/\tau}}{e^{s_{i_1 i_1} / \tau} + \sum_{j: j \neq i_1} e^{s_{i_1 j}/\tau}}
        \geq \frac{n^{s_{i_1 j_1}/\delta}}{n^{s_{i_1 i_1} / \delta} + n^{\max_{j: j \neq i_1} s_{i_1 j}/\delta + 1}}
        = 1 - O(n^{-\gamma}).
    \end{align*}
    By a similar argument, we obtain $\beta_{i_1} = \nu - O(n^{-\gamma})$.
\end{proof}

\if0
\begin{thm}[Restatement of Theorem \ref{thm: feature recovery via MMCL}]\label{thm: feature recovery via MMCL restatement}
    Suppose that Assumptions \ref{asm: asymptotics 2}, \ref{asm: signal condition number}, and \ref{asm: signal-to-noise ratio} hold and $p < 1/3$.
    Fix any $\gamma > 0$ and $\epsilon \geq 0$.
    Choose $\tau$ as in Lemma \ref{lem: edge detection}.
    Suppose that there exist a constant $q' > 0$ and some small constant $c_{q'} = c_{q'}(\sigma, s_1, s_2, \kappa_z^2, q') > 0$ such that $G_1^{(t)}$ and $G_2^{(t)}$ satisfy
    \begin{align*}
        \|G_1^{(t) \top} G_2^{(t)} - q' \Sigma\| \leq c_{q'} \frac{r}{(r + r(\Sigma_\xi))(r + r(\Sigma_{\tilde \xi}) \log n}.
    \end{align*}
    Consider applying Algorithm \ref{alg: MMCL} to the data generated from \eqref{model: multimodal}.
    Then, with probability $1 - O(n^{-1})$,
    
    \begin{align*}
        \|\sin\Theta(P_r(G_1^{(t+1)}), U_1^*)\|_F \vee \|\sin\Theta(P_r(G_2^{(t+1)}), U_2^*)\|_F &\lesssim \sqrt{r} \qty[\frac{(r + r(\Sigma_\xi) + r(\Sigma_{\tilde \xi})) \log n}{n^\gamma} + \sqrt{\frac{(r + r(\Sigma_\xi) + r(\Sigma_{\tilde \xi})) \log (nd_1 + nd_2)}{n}}],\\
        \|G_1^{(t+1) \top} G_2^{(t+1)} - (1-2p) \Sigma\| &\lesssim \frac{(r + r(\Sigma_\xi) + r(\Sigma_{\tilde \xi})) \log n}{n^\gamma} + \sqrt{\frac{(r + r(\Sigma_\xi) + r(\Sigma_{\tilde \xi})) \log (nd_1 + nd_2)}{n}}.
    \end{align*}
\end{thm}
\fi

\subsection{Proof of Lemma \ref{lem: similarity}}

Here we prove Lemma \ref{lem: similarity}.
\begin{proof}%[Proof of Lemma \ref{lem: edge detection}]%\label{proof: lem: edge detection restatement}
    %We first prove that $\min_{(i_1, j_1) \in \mathcal{E}} s_{i_1 j_1} - \max_{j: (i_1, j) \not\in \mathcal{E}} s_{i_1, j}$ is bounded away from $0$ with high probability. 
    We first prove
    \begin{align*}
        \min_{(i_1, j_1) \in \mathcal{E}} \qty(s_{i_1 j_1} - \max_{j: (i_1, j) \not\in \mathcal{E}} s_{i_1, j}) \geq C' \sqrt{r \log n}.
    \end{align*}
    Fix any $(i_1, j_1) \in \mathcal{E} \setminus \mathcal{C}$. 
    Define $\Sigma = \Sigma_z^{1/2} \Sigma_{\tilde z}^{1/2}$ for notational brevity.
    %We omit the superscript $(t)$ from $G_1^{(t)}$, $G_2^{(t)}$, $s_{ij}^{(t)}$, $\beta_{ij}^{(t)}$ and $S^{(t)}$.
    % Also fix any $(i_1, j_2) \in \mathcal{E}^\perp_{i_1,\cdot}$.
    Recall that $\|\Sigma_z\| = \|\Sigma_{\tilde z}\| = 1$.
    Since $\tilde x_{j_1} = U_2^* \tilde z_{j_1} + \tilde \xi_{j_1}$ is a sub-Gaussian random vector with parameter $\sqrt{\sigma_{\tilde z}^2 + \sigma_{\tilde \xi}^2}$, $x_{i_1}^\top U_1^* \Sigma U_2^{* \top} \tilde x_{j_1}$ is a sub-Gaussian random variable with parameter $\sqrt{\sigma_{\tilde z}^2 + \sigma_{\tilde \xi}^2} \|U_1^{* \top} x_{i_1}\|$ conditioned on $x_{i_1}$. 
    Note that since $(i_1, j_1) \in \mathcal{E} \setminus \mathcal{C}$, $(x_{i_1}, \tilde x_{j_1})$ is independent of $\{(x_{i_1}, \tilde x_j: (i_1, j) \not \in \mathcal{E}\}$.
    By Lemma \ref{lem: maxima of sub-Gaussian} applied to independent sub-Gaussian random variables $(x_{i_1}^\top U_1^* \Sigma U_2^{* \top} \tilde x_{j})_{j: (i_1, j) \not\in \mathcal{E}}$ conditioned on $x_{i_1}$ and $\tilde x_{j_1}$,
    \begin{align}
        &\mathbb{P}\qty(x_{i_1}^\top U_1^* \Sigma U_2^{* \top} \tilde x_{j_1} - \max_{j: (i_1, j) \not\in \mathcal{E}} x_{i_1}^\top U_1^* \Sigma U_2^{* \top} \tilde x_j \leq t \|U_1^{* \top} x_{i_1}\| \middle| x_{i_1}, \tilde x_{j_1})\nonumber\\
        &\quad= \mathbb{P}\Biggl(\max_{j: (i_1, j) \not\in \mathcal{E}} x_{i_1}^\top U_1^* \Sigma U_2^{* \top} \tilde x_j - \sqrt{2(\sigma_{\tilde z}^2 + \sigma_{\tilde \xi}^2) \log(n-1)} \|U_1^{* \top} x_{i_1}\|\nonumber\\
        &\quad\quad\quad\quad\geq x_{i_1}^\top U_1^* \Sigma U_2^{* \top} \tilde x_{j_1} - t \|U_1^{* \top} x_{i_1}\| - \sqrt{2(\sigma_{\tilde z}^2 + \sigma_{\tilde \xi}^2) \log(n-1)} \|U_1^{* \top} x_{i_1}\| \Bigg| x_{i_1}, \tilde x_{j_1}\Biggr)\nonumber\\
        &\quad\leq \exp\Biggl(-\frac{1}{2(\sigma_{\tilde z}^2 + \sigma_{\tilde \xi}^2) \|U_1^{* \top} x_{i_1}\|^2} \Big( x_{i_1}^\top U_1^* \Sigma U_2^{* \top} \tilde x_{j_1} - t \|U_1^{* \top} \tilde x_{j_1}\| - \sqrt{2(\sigma_{\tilde z}^2 + \sigma_{\tilde \xi}^2) \log(n-1)} \|U_1^{* \top} x_{i_1}\| \Big)^2 \Biggr)\nonumber\\
        &\quad= \exp(-\frac{1}{2(\sigma_{\tilde z}^2 + \sigma_{\tilde \xi}^2)} \qty( \frac{x_{i_1}^\top U_1^* \Sigma U_2^{* \top} \tilde x_{j_1}}{\|U_1^{* \top} x_{i_1}\|} - t - \sqrt{2(\sigma_{\tilde z}^2 + \sigma_{\tilde \xi}^2) \log(n-1)} )^2 ).\label{eq: conditional probability of delta}
    \end{align}

    Set $t \leftarrow t_0 \triangleq \sqrt{\log n}$.
    We further bound the far right hand side in \eqref{eq: conditional probability of delta}.
    %Note that
    %\begin{align*}
    %    x_{i_1}^\top U_1^* U_2^{* \top} \tilde x_j = z_i^\top \tilde z_j + \xi_i^\top U_1^* \tilde z_j + z_i^\top U_2^{* \top} \tilde \xi_j + \xi_i^\top U_1^* U_2^{* \top} \tilde \xi_j.
    %\end{align*}
    Note that by Assumption \ref{asm: signal-to-noise ratio} and $\|\Sigma_z\| = \|\Sigma_{\tilde z}\| = 1$,
    \begin{align*}
        \sqrt{(\sigma_{\tilde z}^2 + \sigma_{\tilde \xi}^2) \log(n-1)} \leq \sigma \sqrt{ (1 + s_2^{-2}) } t_0.
    \end{align*}
    From Lemma \ref{lem: good event}, there exists an event $E$ with $\mathbb{P}(E) = 1 - O(n^{-1})$ such that on the event $E$ the following holds uniformly for all $(i_1, j_1) \in \mathcal{E}$: 
    there exists a constant $C_1 = C_1(\sigma, s_1, s_2) > 0$ satisfying
    \begin{align}
        &\frac{x_{i_1}^\top U_1^* \Sigma U_2^{* \top} \tilde x_{j_1}}{\|U_1^{* \top} x_{i_1}\|} - t_0 - \sqrt{2(\sigma_{\tilde z}^2 + \sigma_{\tilde \xi}^2) \log(n-1)}\nonumber\\
        &\quad\geq \frac{z_{i_1}^\top \Sigma \tilde z_{j_1} - \max_{(i, j) \in \mathcal{E}} |\xi_{i}^\top U_1^* \Sigma \tilde z_{j} + z_{i}^\top \Sigma U_2^{* \top} \tilde \xi_{j} + \xi_{i}^\top U_1^* \Sigma U_2^{* \top} \tilde \xi_{j}|}{\max_{i \in [n]} \|U_1^{* \top} x_i\|} - t_0 - \sqrt{2(\sigma_{\tilde z}^2 + \sigma_{\tilde \xi}^2) \log(n-1)}\nonumber\\
        %&\quad\geq \frac{\tr(\Sigma_z^{1/2} \Sigma_{\tilde z}^{1/2}) - C(\sigma) \|\Sigma_z\|^{1/2} \|\Sigma_{\tilde z}\|^{1/2} \sqrt{r \log(2n^2r) } - C(\sigma) r \|\Sigma_{\xi}\|^{1/2} \|\Sigma_{\tilde z}\|^{1/2} \log(2n^2r) + C(\sigma) r \|\Sigma_{\tilde \xi}\|^{1/2} \|\Sigma_{z}\|^{1/2} \log(2n^2r) + C(\sigma) r \|\Sigma_\xi\|^{1/2} \|\Sigma_{\tilde \xi}\|^{1/2} \log (2n^2r)}{(\sqrt{r} + C(\sigma) \sqrt{2\log n}) (\|\Sigma_z\|^{1/2} + \|\Sigma_\xi\|^{1/2})}\\
        %&\quad\quad- 2 \sqrt{2(\sigma_{\tilde z}^2 + \sigma_{\tilde \xi}^2) \log(n-1)}.
        &\quad\geq C_1 \Biggl( \frac{\tr(\Sigma_z^{1/2} \Sigma \Sigma_{\tilde z}^{1/2}) - \|\Sigma_z\|^{1/2} \|\Sigma\| \|\Sigma_{\tilde z}\|^{1/2} \sqrt{r \log n}}{\sqrt{r} \|\Sigma_z\|^{1/2}} - \frac{\sqrt{r \log n}\|\Sigma_z\|^{1/2} \|\Sigma\| \|\Sigma_{\tilde z}\|^{1/2} }{\sqrt{r}\|\Sigma_z\|^{1/2}} - \sqrt{\log n} \Biggr) \nonumber\\
        &\quad\geq C_1 ( \sqrt{r}/(\kappa_z^2 \kappa_{\tilde z}^2) - 3\sqrt{\log n} ) \nonumber\\
        &\quad\geq C_1 \sqrt{r} \qty(\frac{1}{\kappa_z^2 \kappa_{\tilde z}^2} - 3 \sqrt{c}), \nonumber
        %&\quad\gtrsim \frac{r - \sqrt{r \log n}}{\sqrt{r}} - \sqrt{\log n} \gtrsim \sqrt{r},\label{eq: inner product bound 1}
    \end{align}
    where in the third inequality, we used 
    \begin{align*}
        \tr(\Sigma_z^{1/2} \Sigma \Sigma_{\tilde z}^{1/2}) \geq r \lambda_{\min}(\Sigma_z) \lambda_{\min}(\Sigma_{\tilde z}) \geq r \|\Sigma_z\| \|\Sigma_{\tilde z}\| / (\kappa_z^2 \kappa_{\tilde z}^2),% \gtrsim r \|\Sigma_z\| \|\Sigma_{\tilde z}\|,
    \end{align*}
    which holds by Assumption \ref{asm: signal condition number}.
    Retaking $c \leftarrow c \wedge 2^{-1} (3 \kappa_z^2 \kappa_{\tilde z}^2)^2$, we obtain
    \begin{align}
        \frac{x_{i_1}^\top U_1^* \Sigma U_2^{* \top} \tilde x_{j_1}}{\|U_1^{* \top} x_{i_1}\|} - t_0 - \sqrt{2(\sigma_{\tilde z}^2 + \sigma_{\tilde \xi}^2) \log(n-1)}
        \geq \frac{C_2 \sqrt{r}}{2}\label{eq: inner product bound 1}
        %&\quad\gtrsim \frac{r - \sqrt{r \log n}}{\sqrt{r}} - \sqrt{\log n} \gtrsim \sqrt{r},
    \end{align}
    for some constant $C_2 = C_2(c, \sigma, s_1, s_2, \kappa_z^2, \kappa_{\tilde z}^2) > 0$.
    
    Since 
    \begin{align*}
        \sigma_{\tilde z}^2 + \sigma_{\tilde \xi}^2 \leq \sigma^2 \|\Sigma_{\tilde z}\| (1 + s_2^{-2}),
    \end{align*}
    \eqref{eq: conditional probability of delta} becomes
    \begin{align*}
        \exp(-\frac{1}{2(\sigma_{\tilde z}^2 + \sigma_{\tilde \xi}^2)} \qty( \frac{x_{i_1}^\top U_1^* \Sigma U_2^{* \top} \tilde x_{j_1}}{\|U_1^{* \top} x_{i_1}\|} - t_0 - \sqrt{2(\sigma_{\tilde z}^2 + \sigma_{\tilde \xi}^2) \log(n-1)} )^2) \leq \exp(-\frac{ C_2^2 r }{ 8\sigma^2 (1 + s_2^{-2})} ).
    \end{align*}
    Retaking $c \leftarrow c \wedge 2 C_2^2 / ( 8 \sigma^2 (1 + s_2^{-2}) )$, we have
    \begin{align*}
        \exp(-\frac{ C_2^2 r }{ 8\sigma^2 (1 + s_2^{-2})} ) \leq \exp(- 2\log n) = \frac{1}{n^2}.
    \end{align*}
    
    Therefore, by Lemma \ref{lem: good event},
    \begin{align*}
        x_{i_1}^\top U_1^* \Sigma U_2^{* \top} \tilde x_{j_1} - \max_{j: (i_1, j) \not\in \mathcal{E}} x_{i_1}^\top U_1^* \Sigma U_2^{* \top} \tilde x_j
        &\geq t_0 \|U_1^{* \top} x_{i_1}\| \geq \sqrt{ \frac{r \log n}{2 \kappa_z^{-2}} }
    \end{align*}
    holds uniformly for all $(i_1, j_1) \in \mathcal{E}$ with probability $1 - O(n^{-1})$.
    
    Furthermore,
    %since $\max_i (a_i + b_i) \geq \max_i b_i - \max_i a_i$ for any sequence $(a_i)_i, (b_i)_i \subset \R$,
    \begin{align*}
        &x_{i_1}^\top G_1^\top G_2 \tilde x_{j_1} - \max_{j: (i_1, j) \not\in \mathcal{E}} x_{i_1}^\top G_1^\top G_2 \tilde x_j\\
        &\quad= x_{i_1}^\top (G_1^\top G_2 - q U_1^* \Sigma U_2^{* \top}) \tilde x_{j_1} + q x_{i_1}^\top U_1^* U_2^{* \top} \tilde x_{j_1}\\
        &\quad\quad- \max_{j: (i_1, j) \not\in \mathcal{E}} [x_{i_1}^\top (G_1^\top G_2 - q U_1^* \Sigma U_2^{* \top}) \tilde x_j + q x_{i_1}^\top U_1^* \Sigma U_2^{* \top} \tilde x_j]\\
        &\quad\geq - 2 \max_{i, j \in [n]} R_{ij} + q (x_{i_1}^\top U_1^* \Sigma U_2^{* \top} \tilde x_{j_1} - \max_{j: (i_1, j) \not\in \mathcal{E}}x_{i_1}^\top U_1^* \Sigma U_2^{* \top} \tilde x_j ),
    \end{align*}
    where $R_{ij} \triangleq x_i^\top (G_1^\top G_2 - q U_1^* \Sigma U_2^{* \top}) \tilde x_j$.
    Note that
    \begin{align*}
        \max_{i, j \in [n]} R_{ij} \leq \|G_1^\top G_2 - q U_1^* \Sigma U_2^{* \top}\| \max_{i \in [n]} \|x_i\| \max_{i \in [n]} \|\tilde x_i\|.
    \end{align*}
    From Lemma \ref{lem: good event} and by assumption, there exists a constant $C_3 = C_3(\sigma, s_1, s_2)$ satisfying
    \begin{align}
        \max_{i, j \in [n]} R_{ij} &\leq C_3 \|G_1^\top G_2 - q U_1^* \Sigma U_2^{* \top}\| \|\Sigma_z\|^{1/2} \|\Sigma_{\tilde z}\|^{1/2} (r^{1/2} + r^{1/2}(\Sigma_\xi)) (r^{1/2} + r^{1/2}(\Sigma_{\tilde \xi})) \log n\\
        &\leq C_3 c_{q} \sqrt { r \log n }\label{eq: inner product bound 2}
    \end{align}
    on the event $E$, where the last inequality follows from the assumption on $ \|G_1^\top G_2 - q U_1^* \Sigma U_2^{* \top}\|$.
    We can take $c_{q}$ small enough so that $C_3 c_{q} \leq q \kappa_z/(4\sqrt{2})$. Then,
    \begin{align*}
        x_{i_1}^\top G_1^\top G_2 \tilde x_{j_1} - \max_{j: (i_1, j) \not\in \mathcal{E}} x_{i_1}^\top G_1^\top G_2 \tilde x_j \geq q \kappa_z\sqrt{\frac{r \log n}{2}} - q\kappa_z \frac{1}{2} \sqrt{\frac{r \log n}{2}} = q \kappa_z \frac{1}{2}\sqrt{\frac{r \log n}{2}}.
    \end{align*}
    
    Fix $\bar C > 0$ and let $H_{i_1, j_1}(\bar C)$ be the event defined as
    \begin{align*}
        H_{i_1, j_1}(\bar C) \triangleq \qty{ x_{i_1}^\top G_1^\top G_2 \tilde x_{j_1} - \max_{j: (i_1, j) \not\in \mathcal{E}} x_{i_1}^\top G_1^\top G_2 \tilde x_j \leq \bar C \sqrt{r \log n} }.
    \end{align*}
    Then, from above arguments, there exists a constant $C' = C'(\sigma, s_1, s_2, \kappa_z^2, \kappa_{\tilde z}^2, q) > 0$ and a universal constant $c' > 0$ such that
    \begin{align*}
        \max_{(i_1, j_1) \in \mathcal{E}} \mathbb{P}(H_{i_1, j_1}(C') | x_{i_1}, \tilde x_{j_1}) \leq c n^{-2}
    \end{align*}
    holds in the event $E$.
    Observe
    \begin{align*}
        \mathbb{P}\qty(\bigcup_{(i_1, j_1) \in \mathcal{E}} H_{i_1, j_1}(C')) &= \mathbb{P}\qty(E \cap \bigcup_{(i_1, j_1) \in \mathcal{E}} H_{i_1, j_1}(C')) + \mathbb{P}\qty(E^c \cap \bigcup_{(i_1, j_1) \in \mathcal{E}} H_{i_1, j_1}(C'))\\
        &\leq \sum_{(i_1, j_1) \in \mathcal{E}} \mathbb{P}\qty(H_{i_1, j_1}(C') \cap E) + \mathbb{P}(E^c).
    \end{align*}
    Note that
    \begin{align*}
        \mathbb{P}(H_{i_1, j_1}(C') \cap E) &\leq \mathbb{P}(H_{i_1, j_1}(C') \cap \{\mathbb{P}(H_{i_1, j_1}(C') | x_{i_1} \tilde x_{j_1}) \leq c'n^{-2} \})\\
        &= \mathbb{E}[\mathbb{E}[ \1_{H_{i_1,j_1}(C')} \1\{\mathbb{P}(H_{i_1, j_1}(C') | x_{i_1}, \tilde x_{j_1}) \leq c'n^{-2}\} | x_{i_1}, \tilde x_{j_1} ]]\\
        &= \mathbb{E}[\1\{\mathbb{P}(H_{i_1, j_1}(C') | x_{i_1}, \tilde x_{j_1}) \leq c'n^{-2}\} \mathbb{P}(H_{i_1,j_1}(C') | x_{i_1}, \tilde x_{j_1})]\\
        &\leq c'n^{-2}.
    \end{align*}
    Therefore $\mathbb{P}(\bigcup_{i_1, j_1} H_{i_1,j_i}(C')) \lesssim n^{-1}$ and thus
    \begin{align*}
        \min_{(i_1, j_1) \in \mathcal{E}} \qty(s_{i_1 j_1} - \max_{j: (i_1, j) \not\in \mathcal{E}} s_{i_1, j}) = \min_{(i_1, j_1) \in \mathcal{E}} \qty(x_{i_1}^\top G_1^\top G_2 \tilde x_{j_1} - \max_{j: (i_1, j) \not\in \mathcal{E}} x_{i_1}^\top G_1^\top G_2 \tilde x_j) \geq C' \sqrt { r \log n }
    \end{align*}
    holds with probability $1 - O(n^{-1})$. A similar argument gives
    $
        \min_{(i_1, j_1) \in \mathcal{E}} \qty(s_{i_1 j_1} - \max_{i: (i, j_1) \not\in \mathcal{E}} s_{i, j_1}) \geq C' \sqrt { r \log n }
    $.
\end{proof}

\subsection{Proof of Theorem \ref{thm: feature recovery via MMCL}}

Here we prove Theorem \ref{thm: feature recovery via MMCL}.
\if0
We restate Theorem \ref{thm: feature recovery via MMCL}.
\begin{thm}[Restatement of Theorem \ref{thm: feature recovery via MMCL}]\label{thm: feature recovery via MMCL restatement}
    Suppose that Assumptions \ref{asm: signal condition number}, \ref{asm: signal-to-noise ratio}, \ref{asm: asymptotics 2}, and \ref{asm: initial value} hold and $p < 1/3$.
    Fix any $\gamma > 1$ and $\epsilon \geq 0$.
    Choose $\tau$ as in Lemma \ref{lem: edge detection}.
    Consider applying Algorithm \ref{alg: MMCL} to the data generated from \eqref{model: multimodal}.
    Assume that $n$ satisfies
    \begin{align}
        n \geq \frac{C_q}{c_q} \frac{(r + r(\Sigma_\xi) + r(\Sigma_{\tilde \xi}))^{5/2} \log n \sqrt{\log (nd_1 + nd_2)}}{r},\label{eq: n large enough}
    \end{align}
    where $C_q = C_q(\sigma, s_1, s_2, \kappa_z^2, \kappa_{\tilde z}^2, q) > 0$ is some constant. 
    Then, there exists an event $E$ with probability $1 - O(n^{-1})$ such that on the event $E$ the following inequalities hold for all $t \geq 0$:
    \begin{align*}
        \|\sin\Theta(P_r(G_1^{(t+1)}), U_1^*)\|_F \vee \|\sin\Theta(P_r(G_2^{(t+1)}), U_2^*)\|_F &\lesssim \sqrt{r} \wedge \sqrt{\frac{r (r + r(\Sigma_\xi) + r(\Sigma_{\tilde \xi})) \log (nd_1 + nd_2)}{n}}.
    \end{align*}
\end{thm}
\fi

\begin{proof}
    %As in the proof of Lemma \ref{lem: edge detection}, we only consider $t=1$ and 
    %We start with $t=0$.
    %For brevity, we omit the superscript $(0)$ from $G_1^{(0)}$, $G_2^{(0)}$, $s_{ij}^{(0)}$, $\beta_i^{(0)}$, $\beta_{ij}^{(0)}$ and $S^{(0)}$.
    %We first deal with the case when $\epsilon = 0$.
    By Lemma \ref{lem: edge detection}, we can rewrite $S$ in Algorithm \ref{alg: MMCL} as
    \begin{align*}
        S &= \frac{1}{n} \sum_{(i, i) \in \mathcal{C} \cap \mathcal{E}} \beta_i x_i \tilde x_i^\top + \frac{1}{n} \sum_{(i, i) \in \mathcal{C} \setminus \mathcal{E}} \beta_i x_i \tilde x_i^\top - \frac{1}{n} \sum_{(i, j) \in \mathcal{E} \setminus \mathcal{C}} \beta_{ij} x_i \tilde x_j^\top - \frac{1}{n} \sum_{(i, j) \not\in \mathcal{E} \cup \mathcal{C}} \beta_{ij} x_i \tilde x_j^\top\\
        &= \frac{\nu-1}{n} \sum_{(i, i) \in \mathcal{C} \cap \mathcal{E}} x_i \tilde x_i^\top + \frac{\nu}{n} \sum_{(i, i) \in \mathcal{C} \setminus \mathcal{E}} x_i \tilde x_i^\top - \frac{1}{n} \sum_{(i, j) \in \mathcal{E} \setminus \mathcal{C}} x_i \tilde x_j^\top + \frac{1}{n} \sum_{(i, j) \in \mathcal{E} \setminus \mathcal{C}} (1 - \beta_{ij}) x_i \tilde x_j^\top\\
        &\quad- \frac{1}{n} \sum_{(i, i) \in \mathcal{C} \cap \mathcal{E}} (\nu - 1 - \beta_i) x_i \tilde x_i^\top - \frac{1}{n} \sum_{(i, i) \in \mathcal{C} \setminus \mathcal{E}} (\nu - \beta_i) x_i \tilde x_i^\top - \frac{1}{n} \sum_{(i, j) \not\in \mathcal{E} \cup \mathcal{C}} \beta_{ij} x_i \tilde x_j^\top\\
        &\triangleq T_1 + Q_1 - T_2 + R_1 - R_2 - R_3 - R_4.
    \end{align*}
    
    We first bound $R_1$, $R_2$, $R_3$ and $R_4$. For the term $R_1$, from Lemma \ref{lem: good event} and Lemma \ref{lem: edge detection},
    \begin{align*}
        \|R_1\| &\leq \frac{n - m}{n} \max_{(i, j) \in \mathcal{E} \setminus \mathcal{C}} (1 - \beta_{ij}) \max_{i \in [n]} \|x_i\| \max_{i \in [n]} \|\tilde x_i\|\\
        &\lesssim \frac{n}{n^{1 + \gamma}} \|\Sigma_z\|^{1/2} \|\Sigma_{\tilde z}\|^{1/2} (r + r(\Sigma_\xi))^{1/2} (r + r(\Sigma_{\tilde \xi}))^{1/2} \log n\\
        &\lesssim \frac{1}{n^\gamma} (\|\Sigma_z\| \vee \|\Sigma_{\tilde z}\|) (r + r(\Sigma_\xi) + r(\Sigma_{\tilde \xi})) \log n
    \end{align*}
    holds with probability at least $1 - O(n^{-1})$.
    Similary,
    \begin{align*}
        \|R_2\| &\leq \frac{m}{n} \max_{(i, i) \in \mathcal{C} \cap \mathcal{E}} (\nu - 1 - \beta_i) \max_{i \in [n]} \|x_i\| \max_{i \in [n]} \|\tilde x_i\|\\
        &\lesssim \frac{1}{n^{\gamma}} \|\Sigma_z\|^{1/2} \|\Sigma_{\tilde z}\|^{1/2} (r + r(\Sigma_\xi))^{1/2} (r + r(\Sigma_{\tilde \xi}))^{1/2} \log n\\
        &\lesssim \frac{1}{n^\gamma} (\|\Sigma_z\| \vee \|\Sigma_{\tilde z}\|) (r + r(\Sigma_\xi) + r(\Sigma_{\tilde \xi})) \log n,\\
        \|R_3\| &\leq \frac{n-m}{n} \max_{(i, i) \in \mathcal{C} \setminus \mathcal{E}} (\nu - \beta_i) \max_{i \in [n]} \|x_i\| \max_{i \in [n]} \|\tilde x_i\|\\
        &\lesssim \frac{1}{n^{\gamma}} \|\Sigma_z\|^{1/2} \|\Sigma_{\tilde z}\|^{1/2} (r + r(\Sigma_\xi))^{1/2} (r + r(\Sigma_{\tilde \xi}))^{1/2} \log n\\
        &\lesssim \frac{1}{n^\gamma} (\|\Sigma_z\| \vee \|\Sigma_{\tilde z}\|) (r + r(\Sigma_\xi) + r(\Sigma_{\tilde \xi})) \log n,\\
        \|R_4\| &\leq \frac{n^2 - 2n + m}{n} \max_{(i, j) \not\in \mathcal{E} \cup \mathcal{C}} \beta_{ij} \max_{i \in [n]} \|x_i\| \max_{i \in [n]} \|\tilde x_i\|\\
        &\lesssim \frac{1}{n^\gamma} (\|\Sigma_z\| \vee \|\Sigma_{\tilde z}\|) (r + r(\Sigma_\xi) + r(\Sigma_{\tilde \xi})) \log n
    \end{align*}
    holds with probability at least $1 - O(n^{-1})$.
    We can bound the terms $T_1$ and $Q_1$ as in \eqref{eq: T1} and \eqref{eq: Q1}.
    
    Similar to the argument in \eqref{eq: Q1}, with probability $1 - O(n^{-1})$,
    \begin{align*}
        \norm{T_2 - \frac{n - m}{n} U_1^* \Sigma_{z}^{1/2} \Sigma_{\tilde z}^{1/2} U_2^{* \top}} \lesssim \frac{n - m}{n} (\|\Sigma_z\| \vee \|\Sigma_{\tilde z}\|) \sqrt{\frac{(r + r(\Sigma_\xi) + r(\Sigma_{\tilde \xi})) \log (nd_1 + nd_2)}{n - m}}.
    \end{align*}
    Therefore,
    \begin{align}
        &\norm{S - \qty(\nu \frac{m}{n} - 1) U_1^* \Sigma_{z}^{1/2} \Sigma_{\tilde z}^{1/2} U_2^{* \top}}\\
        &\leq \norm{T_1 - (\nu - 1) \frac{m}{n} U_1^* \Sigma_{z}^{1/2} \Sigma_{\tilde z}^{1/2} U_2^{* \top}} + \norm{T_2 - \frac{n-m}{n} U_1^* \Sigma_{z}^{1/2} \Sigma_{\tilde z}^{1/2} U_2^{* \top}}\\
        &\quad+ \|Q_1\| + \|R_1\| + \|R_2\| + \|R_3\| + \|R_4\|\\
        %\|S - \frac{2 m - n}{n} U_1^* \Sigma_{z}^{1/2} \Sigma_{\tilde z}^{1/2} U_2^{* \top}\|
        &\lesssim \frac{(r + r(\Sigma_\xi) + r(\Sigma_{\tilde \xi})) \log n}{n^\gamma}\\
        &\quad+ \qty[ \qty(1 - \frac{m}{n})^{1/2} +  \qty(\frac{m}{n})^{1/2} + \qty(1 - \frac{m}{n})^{1/2}] \sqrt{\frac{(r + r(\Sigma_\xi) + r(\Sigma_{\tilde \xi})) \log (nd_1 + nd_2)}{n}}\nonumber\\
        &\lesssim \frac{(r + r(\Sigma_\xi) + r(\Sigma_{\tilde \xi})) \log n}{n^\gamma} + \sqrt{\frac{(r + r(\Sigma_\xi) + r(\Sigma_{\tilde \xi})) \log (nd_1 + nd_2)}{n}}.\label{eq: S concentration}
    \end{align}
    Let $\Sigma' \triangleq (\nu m / n - 1) U_1^* \Sigma_z^{1/2} \Sigma_{\tilde z}^{1/2} U_2^{* \top} = (\nu - 1 - \nu p_n) U_1^* \Sigma_z^{1/2} \Sigma_{\tilde z}^{1/2} U_2^{* \top}$. Then $\lambda_r(\Sigma') = |\nu - 1 - \nu p_n| \lambda_{\min}(\Sigma_z^{1/2} \Sigma_{\tilde z}^{1/2})$ and $\lambda_{\max}(\Sigma') = |\nu - 1 - \nu p_n| \lambda_{\max}(\Sigma_z^{1/2} \Sigma_{\tilde z}^{1/2})$.
    From Theorem 3 in \cite{yu2015useful} with $s \leftarrow r$, $r \leftarrow 1$,
    \begin{align}
        &\|\sin\Theta(P_r(G_1), U_1^*)\|_F \vee \|\sin\Theta(P_r(G_2), U_2^*)\|_F\nonumber\\
        &\quad\leq \sqrt{r} \wedge \frac{2(2|\nu - 1 - \nu p_n|\lambda_{\max}(\Sigma_{z}^{1/2} \Sigma_{\tilde z}^{1/2}) + \|S - \Sigma'\|) r^{1/2} \|S - \Sigma'\| }{(\nu - 1 - \nu p_n)^2 \lambda_{\min}^2(\Sigma_{z}^{1/2} \Sigma_{\tilde z}^{1/2})}\nonumber\\
        &\quad\lesssim \sqrt{r} \wedge \sqrt{r} \qty[\frac{(r + r(\Sigma_\xi) + r(\Sigma_{\tilde \xi})) \log n}{n^\gamma} + \sqrt{\frac{(r + r(\Sigma_\xi) + r(\Sigma_{\tilde \xi})) \log (nd_1 + nd_2)}{n}}]\nonumber\\
        &\quad\lesssim \sqrt{r} \wedge \sqrt{\frac{r (r + r(\Sigma_\xi) + r(\Sigma_{\tilde \xi})) \log (nd_1 + nd_2)}{n}},\label{eq: feature recovery bound}
    \end{align}
    where in the second inequality we used Assumption \ref{asm: signal condition number} and $\nu - 1 - \nu p_n \geq \nu \eta - 1 \geq 0.1$.
    
    Furthermore, from \eqref{eq: S concentration} and since $G_1^\top G_2 = \SVD_r(S)$, there exists a constant $C_q = C_q(\sigma, s_1, s_2, \kappa_z^2, \kappa_{\tilde z}^2, q) > 0$ satisfying
    \begin{align*}
        \|G_1^\top G_2 - (\nu - 1 - \nu p_n) \Sigma\| &= \|\SVD_r(S) - \Sigma'\|\\
        &\leq \lambda_{r+1}(S) + \|S - \Sigma'\|\\
        &\leq \lambda_{r+1}(\Sigma') + 2\|S - \Sigma'\|\\
        &\leq C_q \qty(\sqrt{r} \wedge \sqrt{\frac{(r + r(\Sigma_\xi) + r(\Sigma_{\tilde \xi})) \log (nd_1 + nd_2)}{n}}).\label{eq: feature recovery bound 2}
    \end{align*}
    Thus, the condition in \eqref{eq: n large enough} implies that $\|G_1^\top G_2 - (\nu - 1 - \nu p_n) \Sigma\| \leq c_q r / ( (r + r(\Sigma_\xi)) (r + r(\Sigma_{\tilde \xi})) \log n )$.
\end{proof}

\subsection{Proof of \eqref{eq: min nonlinear unsupervised loss is SVD}}

Here we restate the result of \ref{eq: min nonlinear unsupervised loss is SVD}, which follows by a similar argument in the proof of Proposition \ref{prop: min nonlinear loss is SVD}.
\begin{prop}\label{prop: min nonlinear unsupervised loss is SVD restatement}
    %Let $s_{ij} \triangleq \langle x_i, \tilde x_j \rangle$.
    Consider minimizing the nonlinear loss function $\mathcal{L}^u$ defined in \eqref{eq: RMMCL}. Then,
    \begin{align*}
        \frac{\partial \mathcal{L}}{\partial G_k} = -\eval{\frac{\partial \tr(G_1 S(\beta) G_2^\top)}{\partial G_k}}_{\beta^u=\beta^u(G_1, G_2)} + \frac{\partial R(G_1, G_2)}{\partial G_k},\ \ \ \ \ \ k \in \{1,2\},
    \end{align*}
    where the contrastive cross-covariance $S(\beta^u)$ is given by:
    \begin{align*}
        S(\beta^u) &= \frac{1}{N} \sum_{(i,j) \in \mathcal{\bar E}^u} \beta_i^u x_i \tilde x_j^\top - \frac{1}{N} \sum_{(i,j) \not\in \mathcal{\bar E}^u} \beta_{ij}^u x_i \tilde x_j^\top,
    \end{align*}
    with
    \begin{align*}
        \beta_i^u &= \nu - 1 + \frac{1}{2} \frac{e^{s_{ij}^u/\tau}}{\sum_{j' \in [N]} e^{s_{ij'}^u/\tau}} + \frac{1}{2} \frac{e^{s_{ij}^u/\tau}}{\sum_{i' \in [N]} e^{s_{i'j}^u/\tau}},\\
        \beta_{ij}^u &= \frac{1}{2} \frac{e^{s_{ij}^u/\tau}}{\sum_{j' \in [N]} e^{s_{ij'}^u/\tau}} + \frac{1}{2} \frac{e^{s_{ij}^u/\tau}}{\sum_{i' \in [N]} e^{s_{i'j}^u/\tau}}.
    \end{align*}
    \if0
    Furthermore, if we use $s_{ij} = - \|G_1 x_i - G_2 \tilde x_j\|^2$, then
    \begin{align*}
        \frac{\partial \mathcal{L}}{\partial G_k} = -\eval{\frac{\partial \tr(G_1 S G_2^\top)}{\partial G_k}}_{\beta=\beta(G_1, G_2)} + \frac{\partial R(G_1, G_2)}{\partial G_k},\ \ \ \ \ \ k \in \{1,2\},
    \end{align*}
    \fi
\end{prop}

\if0
\begin{proof}
    Using an argument similar to the proof of Proposition \ref{prop: min nonlinear loss is SVD restatement},
    we first take derivative of the loss function \ref{eq: RMMCL}. Observe
    \begin{align*}
        \partial_{G_2} \qty( -\frac{1}{n} \sum_i \log \frac{1}{1 + \sum_j e^{s_{ij}/\tau}} + R(G_1, G_2))
        &= G_1 \qty( \frac{1}{n} \sum_i \sum_j \frac{e^{s_{ij}/\tau}}{1 + \sum_{j'} e^{s_{ij'}/\tau}} x_i \tilde x_j^\top ) + \partial_{G_2} R(G_1, G_2)\\
        &= - G_1 \qty(-\frac{1}{n} \sum_{i,j} \beta_{ij} x_i \tilde x_j^\top),
        %&= \partial_{G_2} \tr( G_1 \frac{1}{n} \sum_{i,j} \beta'_{ij} x_i \tilde x_j^\top G_2^\top ) + \partial_{G_2} R(G_1, G_2)\\
        %&= - \partial_{G_2} \tr( G_1 S(\beta) G_2^\top ) + \partial_{G_2} R(G_1, G_2),
    \end{align*}
    where $\beta'_{ij} \triangleq \frac{e^{s_{ij}/\tau}}{1 + \sum_{j'} e^{s_{ij'}/\tau}}$.
    \if0
    \begin{align*}
        \partial_{G_2} \qty( -\frac{1}{N} \sum_i \log \frac{1}{1 + \sum_j e^{s_{ij}^u/\tau}} + R(G_1, G_2))
        %&= G_1 \qty( \frac{1}{n} \sum_i \sum_j \frac{e^{s_{ij}/\tau}}{1 + \sum_{j'} e^{s_{ij'}/\tau}} x_i \tilde x_j^\top ) + \partial_{G_2} R(G_1, G_2)\\
        &= \partial_{G_2} \tr( G_1 \frac{1}{N} \sum_{i,j} \beta_{ij}^u x_i \tilde x_j^\top G_2^\top ) + \partial_{G_2} R(G_1, G_2).
    \end{align*}
    \fi
    Combined with $\partial_{G_2} \tr(G_1 S G_2^\top) = G_1 S$, and by symmetry, this gives the desired result.
\end{proof}
\fi

\subsection{Proof of Lemma \ref{lem: edge detection RMMCL}}\label{sec: edge detection RMMCL}

%\subsection{Iterative Algorithms to Incorporate Unpaired Data}\label{sec: iterative RMMCL}
%\lm{FIx}

Here we consider Algorithm \ref{alg: MMCL Unlabeled}.
%\zhun{this sentence is not useful. We could either remove it, or say more about why it helps clarity.}
\if0
If the effective ranks of the noise covariance matrices satisfy $r(\Sigma_{\xi}) \vee r(\Sigma_{\tilde \xi}) \lesssim r$, the initial value condition in Assumption \ref{asm: initial value} requires
$
    \|G_1^{(0) \top} G_2^{(0)} - U_1^* \Sigma_z^{1/2} \Sigma_{\tilde z}^{1/2} U_2^{* \top}\|^2 = o\qty( 1/(r \log N) ).
$
\fi
% Suppose that the observed pairs are ground-truth, that is, $w_i = \tilde w_i$ for all $i \in [n]$. 
%The proof follows directly from Theorem \ref{thm: linear loss master} in the Appendix and thus omitted.
%The proof of Lemma \ref{lem: edge detection} is deferred to Appendix \ref{proof: lem: edge detection}.
We restate Lemma \ref{lem: edge detection RMMCL}.
\begin{lem}[Restatement of Lemma \ref{lem: edge detection RMMCL}]\label{lem: edge detection RMMCL restatement}
    Suppose Assumptions \ref{asm: signal condition number}, \ref{asm: signal-to-noise ratio}, \ref{asm: n large enough}, and \ref{asm: asymptotics 2} hold.
    Fix any $\gamma > 0$ and $\nu \geq 1$.
    Choose $\tau \leq C(1 + \gamma)^{-1} \sqrt{r / \log N}$, where $C > 0$ is some constant depending on $\sigma, s_1, s_2, \kappa_z^2, \kappa_{\tilde z}^2$.
    Consider applying Algorithm \ref{alg: MMCL Unlabeled} to the data generated from the model \ref{model: multimodal}.
    Then, with probability $1 - O(N^{-1} \vee n^{-1})$, $\mathcal{\hat E}^u = \mathcal{E}^u$ and
    \begin{align*}
        %\min_{(i, j) \in \mathcal{E} \setminus \mathcal{C}} \beta_{ij} &\geq \frac{1}{2},\\
        \min_{(i, j) \in \mathcal{E}^u \setminus \mathcal{\bar E}^u} \beta_{ij}^{u(0)} &= 1 - O\qty(\frac{1}{N^\gamma}),\ \ \max_{(i, j) \not\in \mathcal{E}^u \cup \mathcal{\bar E}^u} \beta_{ij}^{u(0)} \lesssim \frac{1}{N^{1+\gamma}},\\
        \min_{(i, i) \in \mathcal{E}^u \cap \mathcal{\bar E}^u} \beta_{i}^{u(0)} &= \nu-1 + O\qty(\frac{1}{N^\gamma}),\ \ \max_{(i, i) \in \mathcal{\bar E}^u \setminus \mathcal{E}^u} \beta_{i}^{u(0)} = \nu - O\qty(\frac{1}{N^\gamma}).
    \end{align*}
\end{lem}

\begin{proof}
    %With the help of Theorem \ref{thm: linear loss master}, we can show that Assumption \ref{asm: initial value} can be satisfied 
    %by applying SVD on another set of paired data $(x_i, \tilde x_i)_{i=1}^{n}$ following to the data generating process described in Section \ref{sec: data generation}.
    
    Since the initial representations $G_1^{(0)}$ and $G_2^{(0)}$ are the solution to the minimization of the loss \ref{eq: linear loss} with the dataset $(x_i, \tilde x_i)_{i=1}^{n}$, Theorem \ref{thm: linear loss master} and Assumption \ref{asm: n large enough} give
    \begin{align}
        \norm{G_1^{(0) \top} G_2^{(0)} - \frac{1}{\rho} U_1^* \Sigma_z^{1/2} \Sigma_{\tilde z}^{1/2} U_2^{* \top}} \leq c_q \frac{r}{(r + r(\Sigma_\xi)) (r + r(\Sigma_{\tilde \xi})) \log N}
    \end{align}
    with probability $1 - O(n^{-1})$.
    
    From Lemma \ref{lem: similarity}, with probability $1 - O(N^{-1} \vee n^{-1})$,
    $$
        \min_{(i_1, j_1) \in \mathcal{E}^u} \qty(s_{i_1 j_1} - \max_{j: (i_1, j) \not\in \mathcal{E}^u} s_{i_1, j} \vee \max_{i: (i, j_1) \not\in \mathcal{E}^u} s_{i, j_1}) \geq C' \sqrt{r \log N}
    $$.
    This implies that $\{(i, j) \in [N]^2 : i \in [N], j \in \argmax_{j'} s_{ij'}^u\} = \{(i, j) \in [N]^2 : j \in [N], i \in \argmax_{i'} s_{i'j}^u\} = \mathcal{E}^u$ with high probability.
    %This implies that $\mathcal{\hat E}^u = \{(i, j) : s_{i j} \geq s_{(N)}\} = \mathcal{E}^u$ with high probability. 
    
    The conclusion directly follows by Lemma \ref{lem: edge detection} with substitution $\mathcal{C} \leftarrow \mathcal{\bar E}^u$ and $\mathcal{E} \leftarrow \mathcal{E}^u$, since $(x_i, \tilde x_i)_{i=1}^n$ and $(x_i^u, \tilde x_i^u)_{i=1}^N$ are independent.
\end{proof}

\subsection{Proof of Theorem \ref{thm: feature recovery via RMMCL}}

We restate Theorem \ref{thm: feature recovery via RMMCL}.
\begin{thm}[Restatement of Theorem \ref{thm: feature recovery via RMMCL}]\label{thm: feature recovery via RMMCL restatement}
    Suppose Assumptions \ref{asm: signal condition number}, \ref{asm: signal-to-noise ratio}, \ref{asm: n large enough} and \ref{asm: asymptotics 2} hold.
    Fix any $\gamma > 2$ and $\nu \geq 1.1$.
    Choose $\tau$ as in Lemma \ref{lem: edge detection RMMCL restatement}.
    Consider applying Algorithm \ref{alg: MMCL Unlabeled} to the data $(x_i^u, \tilde x_i^u)_{i=1}^N$ generated from \eqref{model: multimodal}, whose association is unknown.
    %Assume that $n$ satisfies \eqref{eq: n large enough}.
    %\begin{align}
    %    |p - \frac{1}{2}| \geq \frac{(r + r(\Sigma_\xi) + r(\Sigma_{\tilde \xi})) \log n}{n^\gamma} + \sqrt{\frac{(r + r(\Sigma_\xi) + r(\Sigma_{\tilde \xi})) \log (nd_1 + nd_2)}{n}},
    %\end{align}
    Then, with probability $1 - O(N^{-1} \vee n^{-1})$,
    \begin{align*}
        \|\sin\Theta(P_r(G_1), U_1^*)\|_F \vee \|\sin\Theta(P_r(G_2), U_2^*)\|_F &\lesssim \sqrt{r} \wedge \sqrt{\frac{r (r + r(\Sigma_\xi) + r(\Sigma_{\tilde \xi})) \log (N+d_1 + d_2)}{N}}.
    \end{align*}
\end{thm}

\begin{proof}
    From Lemma \ref{lem: edge detection RMMCL restatement}, we have $\mathcal{\hat E}^u = \mathcal{E}^u$ with high probability.
    Thus we treat $\mathcal{E}^u$ as known for brevity and let $\mathcal{E}^u = \{(1, 1), \dots, (N, N)\}$ without loss of generality. 
    Since the loss function in \ref{eq: RMMCL} is exactly the same as the loss function in \ref{eq: MMCL long} when $\mathcal{E}^u = \{(1, 1), \dots, (N, N)\}$, the conclusion follows by Theorem \ref{thm: feature recovery via MMCL} applied to $(x_i^u, \tilde x_i^u)_{i=1}^N$ with $p_n \equiv 0$.
\end{proof}

%\section{Proofs of Omitted Contents}

\section{Auxiliary Results}

Here, we list the auxiliary results that are used in the proofs.

\begin{lem}\label{lem: good event}
    Suppose Assumptions \ref{asm: signal condition number} and \ref{asm: signal-to-noise ratio} hold.
    Fix any $\Sigma \in \R^{r \times r}$.
    There exists some constant $c = c(\sigma, s_1, \kappa_z^2) \in (0, 1]$ such that if $\log n \leq c r$,
    the following inequalities hold with probability $1 - O(n^{-1})$:
    \begin{align*}
        &\max_{i \in [n]} \|x_i\| \leq C_1 \|\Sigma_z\|^{1/2} (r^{1/2} + r^{1/2}(\Sigma_\xi)) \sqrt{\log n},\\
        &\max_{i \in [n]} \|\tilde x_i\| \leq C_2 \|\Sigma_{\tilde z}\|^{1/2} (r^{1/2} + r^{1/2}(\Sigma_{\tilde \xi})) \sqrt{\log n},\\
        &\max_{(i, j) \in \mathcal{E}} |\xi_{i}^\top U_1^* \Sigma \tilde z_{j} + z_{i}^\top \Sigma U_2^{* \top} \tilde \xi_{j} + \xi_{i}^\top U_1^* \Sigma U_2^{* \top} \tilde \xi_{j}| \leq C_3 \sqrt{r \log n} \|\Sigma_{z}\|^{1/2} \|\Sigma\| \|\Sigma_{\tilde z}\|^{1/2},\\
        &\max_{(i, j) \in \mathcal{E}} \abs{ z_{i}^\top \Sigma \tilde z_{j} - \tr(\Sigma_z^{1/2} \Sigma \Sigma_{\tilde z}^{1/2}) } \leq C_4 \|\Sigma_z\|^{1/2} \|\Sigma\| \|\Sigma_{\tilde z}\|^{1/2} \sqrt{r \log n},\\
        &\max_{i \in [n]} \|U_1^{* \top} x_i\| \leq C_5 \sqrt{r \|\Sigma_z\|},\\
        &\min_{i \in [n]} \|U_1^{* \top} x_i\| \geq \sqrt{\frac{r \|\Sigma_z\|}{2\kappa_z^{-2}}},
    \end{align*}
    where $C_1 = C_1(\sigma, s_1), C_2 = C_2(\sigma, s_2), C_3 = C_3(\sigma, s_1, s_2), C_4 = C_4(\sigma)$ and $C_5 = C_5(\sigma, s_1) > 0$ are some constants.
\end{lem}

\begin{proof}[Proof of Lemma \ref{lem: good event}]
    Let 
    \begin{align*}
        c = 2^{-1} (2 C'''(\sigma) (1 \vee s_1^{-1}) \kappa_z^2)^{-2} \wedge 1.
    \end{align*}

    From Corollary \ref{cor: concentration of norm}, Assumption \ref{asm: signal-to-noise ratio}, and by the union bound argument,
    \begin{align*}
        \max_i \|x_i\| &\leq \max_i \|z_i\| + \max_i \|\xi_i\|\\
        &\leq \tr^{1/2}(\Sigma_z) + \tr^{1/2}(\Sigma_\xi) + C(\sigma) (\|\Sigma_z\|^{1/2} + \|\Sigma_\xi\|^{1/2}) \sqrt{2\log n}\\
        &\leq \|\Sigma_z\|^{1/2} (r^{1/2} + s_1^{-1} r^{1/2}(\Sigma_{\xi})) (1 + C(\sigma) \sqrt{2 \log n})\\
        &\leq C_1(\sigma, s_1) \|\Sigma_z\|^{1/2} (r^{1/2} + r^{1/2}(\Sigma_{\xi})) \sqrt{\log n}
    \end{align*}
    holds with probability at least $1 - 2n^{-1}$, where $C_1(\sigma, s_1) \triangleq (1 \vee s_1^{-1})(1 \vee C(\sigma))$.
    Similarly,
    \begin{align*}
        \max_i \|\tilde x_i\| &\leq C_2(\sigma, s_2) \|\Sigma_{\tilde z}\|^{1/2} (r^{1/2} + r^{1/2}(\Sigma_{\tilde \xi})) \sqrt{\log n}
    \end{align*}
    holds with probability at least $1 - 2n^{-1}$, where $C_2(\sigma, s_2) \triangleq (1 \vee s_2^{-1})(1 \vee C(\sigma))$.
    
    %By setting $t \leftarrow \log n^2$ in Proposition \ref{prop: convergence rate of cross-covariance} and by the union bound argument,
    By Lemma \ref{lem: cross hanson wright} and the union bound argument, there exists some constant $C'(\sigma) > 0$ such that
    \begin{align*}
        \max_{(i, j) \in \mathcal{E}} |\xi_i^\top U_1^* \tilde \Sigma z_j| &\leq C'(\sigma) (\|\Sigma_{\xi}^{1/2} U_1^* \Sigma \Sigma_{\tilde z}^{1/2}\|_F \sqrt{2\log n} \vee 2 \|\Sigma_{\xi}^{1/2} U_1^* \Sigma \Sigma_{\tilde z}^{1/2}\| \log n)\\
        &\leq C'(\sigma) ( \tr^{1/2}(U_1^{* \top} \Sigma_\xi U_1^* \Sigma \Sigma_{\tilde z}) \sqrt{2 \log n} \vee 2\|\Sigma_\xi\|^{1/2} \|\Sigma\| \|\Sigma_{\tilde z}\|^{1/2} \log n )\\
        &\leq C'(\sigma) \|\Sigma_\xi\|^{1/2} \|\Sigma\| \|\Sigma_{\tilde z}\|^{1/2} ( \sqrt{2r \log n} \vee 2 \log n )
    \end{align*}
    holds with probability at least $1 - n^{-1}$.
    Since $\log n \leq \sqrt{r \log n}$, the far right-hand side can be further bounded by
    $2C'(\sigma) s_1^{-1} \|\Sigma_z\|^{1/2} \|\Sigma\| \|\Sigma_{\tilde z}\|^{1/2} \sqrt{r \log n}$. By a similar argument combined with Assumption \ref{asm: signal-to-noise ratio}, there exists some constant $C_3(\sigma, s_1, s_2) > 0$ such that
    \begin{align*}
        &\max_{(i, j) \in \mathcal{E}} |\xi_i^\top U_1^* \Sigma \tilde z_j + z_i^\top \Sigma U_2^{* \top} \tilde \xi_j + \xi_i^\top U_1^* \Sigma U_2^{* \top} \tilde \xi_j|\\
        &\quad\leq \max_{(i, j) \in \mathcal{E}} |\xi_i^\top U_1^* \Sigma \tilde z_j| + \max_{(i, j) \in \mathcal{E}} |z_i^\top \Sigma U_2^{* \top} \tilde \xi_j| + \max_{(i, j) \in \mathcal{E}} |\xi_i^\top U_1^* \Sigma U_2^{* \top} \tilde \xi_j|\\
        %&\quad\leq \max_i \|U_1^* \xi_i \tilde z_i^\top\| + \max_i \|U_2^{* \top} \tilde \xi_i z_i^\top\| + \max_i \|U_2^{* \top} \tilde \xi_i \xi_i^\top U_1^*\|\\
        %&\quad\leq 2C(\sigma) \log n ( \|\Sigma_{\xi}^{1/2} U_1^* \Sigma_{\tilde z}^{1/2}\|_F + \|\Sigma_{\tilde \xi}^{1/2} U_2^* \Sigma_{z}^{1/2}\|_F + \|\Sigma_{\xi}^{1/2} U_1^* U_2^{* \top} \Sigma_{\tilde \xi}^{1/2}\|_F )\\
        %&\quad\leq 2C(\sigma) \log n ( \tr^{1/2}(U_1^{* \top} \Sigma_{\xi} U_1^*) \|\Sigma_{\tilde z}\|^{1/2} + \tr^{1/2}(U_2^{* \top} \Sigma_{\tilde \xi} U_2^*) \|\Sigma_{z}\|^{1/2} + \tr^{1/2}(U_1^{* \top} \Sigma_{\xi}^{1/2} U_1^* U_2^{* \top} \Sigma_{\tilde \xi} U_2^*) )\\
        %&\quad\leq 2C(\sigma) \sqrt{r} \log n ( \|\Sigma_\xi\|^{1/2} \|\Sigma_{\tilde z}\|^{1/2} + \|\Sigma_{\tilde \xi} \|^{1/2} \|\Sigma_{z}\|^{1/2} + \|\Sigma_{\xi}\|^{1/2} \|\Sigma_{\tilde \xi}\|^{1/2} )\\
        %&\quad\lesssim \sqrt{r \log n} \|\Sigma_{z}\|^{1/2} \|\Sigma\| \|\Sigma_{\tilde z}\|^{1/2} ( s_1^{-1} + s_2^{-1} + s_1^{-1} s_2^{-1} )\\
        %&\quad\lesssim \sqrt{r \log n} \|\Sigma_{z}\|^{1/2} \|\Sigma\| \|\Sigma_{\tilde z}\|^{1/2}
        &\quad\leq C_3(c, \sigma, s_1, s_2) \|\Sigma_{z}\|^{1/2} \|\Sigma\| \|\Sigma_{\tilde z}\|^{1/2} \sqrt{r \log n}
    \end{align*}
    holds with probability at least $1 - 3n^{-1}$, where we used Cauchy-Schwarz inequality in the last inequality.
    
    %The following lemma is adapted from the proof of Theorem 6.3.2 in \cite{vershynin2018high}.
    %Recall that $\Sigma_z$ and $\Sigma_{\tilde z}$ are assumed to be diagonal.
    Fix any $(i, j) \in \mathcal{E}$.
    Since $\Sigma_z^{-1/2} z_i = \Sigma_{\tilde z}^{-1/2} \tilde z_j$,
    applying Lemma \ref{lem: hanson wright} with $X \leftarrow \Sigma_z^{-1/2} z_i$, $A \leftarrow \Sigma_z^{1/2} \Sigma \Sigma_{\tilde z}^{1/2}$, and $t \leftarrow \log n^2$ yields the following. There exists a constant $C''(\sigma) > 0$ such that
    \begin{align*}
        \abs{ z_i^\top \Sigma \tilde z_j - \tr(\Sigma_z^{1/2} \Sigma \Sigma_{\tilde z}^{1/2}) } &\leq C''(\sigma) \qty(\|\Sigma_z^{1/2} \Sigma \Sigma_{\tilde z}^{1/2}\|_F \sqrt{\log n^2} \vee \|\Sigma_z^{1/2} \Sigma \Sigma_{\tilde z}^{1/2}\| \log n^2)\\
        &\leq C_4(\sigma) \|\Sigma_z^{1/2} \Sigma \Sigma_{\tilde z}^{1/2}\| \sqrt{r \log n}
    \end{align*}
    holds with probability at least $1 - n^{-2}$, where the last inequality is again from $\log n \leq \sqrt{r \log n}$.
    By the union bound argument, we obtain
    \begin{align*}
        \max_{(i, j) \in \mathcal{E}} \abs{ z_i^\top \Sigma \tilde z_j - \tr(\Sigma_z^{1/2} \Sigma \Sigma_{\tilde z}^{1/2}) } \leq C_4(\sigma) \|\Sigma_z^{1/2}\| \|\Sigma\| \|\Sigma_{\tilde z}^{1/2}\| \sqrt{r \log n}
    \end{align*}
    with probability at least $1 - n^{-1}$.
    
    From Corollary \ref{cor: concentration of norm}, Assumption \ref{asm: signal-to-noise ratio} and by the union bound argument,
    \begin{align*}
        \max_i \|U_1^{* \top} x_i\| &\leq \max_i \|z_i\| + \max_i \|U_1^{* \top} \xi_i\|\\
        &\leq \tr^{1/2}(\Sigma_z) + \tr^{1/2}(U_1^{* \top} \Sigma_\xi U_1^*) + C(\sigma) ( \|\Sigma_z\|^{1/2} + \|\Sigma_\xi\|^{1/2} ) \sqrt{2 \log n}\\
        &\leq (\sqrt{r} + C(\sigma) \sqrt{2\log n}) \|\Sigma_z\|^{1/2} (1 + s_1^{-1})\\
        &\leq C_5(\sigma, s_1) \sqrt{r} \|\Sigma_z\|^{1/2}
    \end{align*}
    holds with probability at least $1 - 2n^{-1}$, where $C_5 > 0$ is some constant.
    
    \if0
    \begin{align*}
        \max_i |z_i^\top \tilde z_i - \tr(\Sigma_{z, \tilde z})| = r \max_i \|z_i \tilde z_i^\top - \mathbb{E}[z_i \tilde z_i^\top]\| \lesssim r^{3/2} \|\Sigma_z\|^{1/2} \|\Sigma_{\tilde z}\|^{1/2} \log(nr).
    \end{align*}
    \fi

    Since $\|z_i\|^2 = \|U_1^{* \top} x_i\|^2 - 2 z_i^\top U_1^{* \top} \xi_i - \|U_1^{* \top} \xi_i\|^2 \leq \|U_1^{* \top} x_i\|^2 - 2 z_i^\top U_1^{* \top} \xi_i$, there exists some constant $C'''(\sigma) > 0$ such that
    \begin{align*}
        \min_i \|U_1^{* \top} x_i\|^2 &\geq \min_i \|z_i\|^2 - 2 \max_i |z_i^\top U_1^{* \top} \xi_i|\nonumber\\
        &\geq \tr(\Sigma_z) - C'''(\sigma) (\|\Sigma_z\|^{1/2} \tr^{1/2}(\Sigma_z) \sqrt{2\log n} \vee 2 \|\Sigma_z\| \log n)\nonumber\\
        &\quad- C'''(\sigma) \|\Sigma_{\xi}\|^{1/2} \|\Sigma_{z}\|^{1/2} (\sqrt{2r \log n} \vee 2 \log n)\nonumber\\
        &\geq r \|\Sigma_z\| \frac{\lambda_{\min}(\Sigma_z)}{\|\Sigma_z\|} - 2 C'''(\sigma) (1 \vee s_1^{-1}) \|\Sigma_z\| \sqrt{r \log n}\nonumber\\
        &= r \|\Sigma_z\| \frac{\lambda_{\min}(\Sigma_z)}{\|\Sigma_z\|} \qty(1 - 2 C'''(\sigma) (1 \vee s_1^{-1}) \frac{\|\Sigma_z\|}{\lambda_{\min}(\Sigma_z)} \sqrt{\frac{\log n}{r}})\nonumber
        %&\gtrsim \|\Sigma_z\| r \qty(1 - \sqrt{\frac{\log n}{r}}) \gtrsim \|\Sigma_z\| r\label{eq: U1x lower bound}
    \end{align*}
    holds with probability at least $1 - 2n^{-1}$, where the second inequality follows from Lemma \ref{lem: hanson wright} and Assumption \ref{asm: signal-to-noise ratio}.
    From Assumption \ref{asm: signal condition number} and the definition of $c$,
    \begin{align*}
        \min_i \|U_1^{* \top} x_i\|^2 &\geq r \|\Sigma_z\| \kappa_z^{-2} \qty(1 - 2 C'''(\sigma) (1 \vee s_1^{-1}) \kappa_z^2 \sqrt{c}) \geq (1/2) r \|\Sigma_z\| \kappa_z^{-2}.
        %&\gtrsim \|\Sigma_z\| r \qty(1 - \sqrt{\frac{\log n}{r}}) \gtrsim \|\Sigma_z\| r\label{eq: U1x lower bound}
    \end{align*}
\end{proof}

\begin{lem}\label{lem: maxima of sub-Gaussian}
    Suppose $X_1, \dots, X_n$ are i.i.d. sub-Gaussian random variables with parameter $\sigma$. Then,
    \begin{align*}
        \mathbb{P}(\max_i X_i - \sqrt{2\sigma^2 \log n} \geq t) &\leq \exp(-\frac{t^2}{2\sigma^2})
    \end{align*}
    holds for all $t \geq 0$.
\end{lem}

\begin{proof}
    Observe that
    \begin{align*}
         \mathbb{P}(\max_i X_i - \sqrt{2\sigma^2 \log n} \geq t) &= \mathbb{P}\qty(\bigcup_{i} \{ X_i \geq t + \sqrt{2\sigma^2 \log n} \})\\
         &\leq n \mathbb{P}(X_1 \geq t + \sqrt{2\sigma^2 \log n})\\
         &\leq n \exp(-\frac{(t + \sqrt{2\sigma^2 \log n})^2}{2\sigma^2})\\
         &\leq \exp(-\frac{t^2}{2\sigma^2}).
    \end{align*}
\end{proof}

\begin{lem}\label{lem: cross hanson wright}
    Let $X = (X_1, \dots, X_{d_1})$ and $\tilde X = (\tilde X_1, \dots, \tilde X_{d_2})$ be mean zero random vectors taking values in $\R^d$.
    Let $A \in \R^{d_1 \times d_2}$ be a non-random matrix.
    Suppose $\Sigma_X^{-1/2} X$ and $\Sigma_{\tilde X} ^{-1/2} \tilde X$ are independent and have i.i.d. sub-Gaussian coordinates with parameter $\sigma$.
    Then, there exists a constant $C = C(\sigma) > 0$ such that
    with probability at least $1 - e^{-t}$,
    \begin{align*}
        |X^\top A \tilde X| &\leq C \qty(\|\Sigma_X^{1/2} A \Sigma_{\tilde X}^{1/2}\|_F \sqrt{t} \vee \|\Sigma_X^{1/2} A \Sigma_{\tilde X}^{1/2}\| t).
    \end{align*}
    holds for all $t > 0$.
\end{lem}
%Note that the right hand side can be futher bounded by $C \|\Sigma_X^{1/2} A \Sigma_{\tilde X}^{1/2}\|_F \log t$ when $t > e$.

\begin{proof}[Proof of Lemma \ref{lem: cross hanson wright}]
    The proof follows by a similar argument as in the proof of Theorem 6.2.1, Lemma 6.2.2 and Lemma 6.2.3 in \cite{vershynin2018high}.
\end{proof}

We also use the following Hanson-Wright inequality. See, for example, Theorem 6.2.1 in \cite{vershynin2018high}.
\begin{lem}\label{lem: hanson wright}
    Let $X = (X_1, \dots, X_{d_1})$ be mean zero random vectors taking values in $\R^d$.
    Let $A \in \R^{d_1 \times d_1}$ be a non-random matrix.
    Suppose $\Sigma_X^{-1/2} X$ have i.i.d. sub-Gaussian coordinates with parameter $\sigma$.
    Then, there exists a constant $C = C(\sigma) > 0$ such that
    with probability at least $1 - e^{-t}$,
    \begin{align*}
        |X^\top A X - \tr(A \Sigma_X)| &\leq C \qty(\|\Sigma_X^{1/2} A \Sigma_X^{1/2}\|_F \sqrt{t} \vee \|\Sigma_X^{1/2} A \Sigma_X^{1/2}\| t).
    \end{align*}
    holds for all $t > 0$.
\end{lem}

The following corollary is adapted from the proof of Theorem 6.3.2 in \cite{vershynin2018high}.
\begin{cor}\label{cor: concentration of norm}
    Let $X$ be a random vector in $\R^d$. 
    Suppose $\Sigma_X^{-1/2} X$ has i.i.d. sub-Gaussian coordinates with parameter $\sigma$.
    Let $A \in \R^{r \times d}$ be any non-random matrix.
    Then, there exists a constant $C = C(\sigma) > 0$ such that
    \begin{align*}
        |\|A X\| - \tr^{1/2}(A \Sigma_X A^\top)| \leq C(\sigma) \|A \Sigma_X A^\top\|^{1/2} \sqrt{t}
    \end{align*}
    holds with probability at least $1 - e^{-t}$ for all $t > 0$.
\end{cor}

\begin{asm}\label{asm: cross-covariance concentration}
    Let $X$ and $\tilde X$ be mean zero random vectors taking values in $\R^{d_1}$ and $\R^{d_2}$, respectively.
    Assume the following 
    \begin{itemize}
        \item $\mathbb{E}[(u^\top X)^2] \geq c_1 \|u^\top X\|_{\psi_2}^2$ holds for any $u \in \R^{d_1}$,
        \item $\mathbb{E}[(v^\top \tilde X)^2] \geq c_2 \|v^\top \tilde X\|_{\psi_2}^2$ holds for any $v \in \R^{d_2}$.
    \end{itemize}
\end{asm}

\begin{prop}\label{prop: cross-covariance concentration}
    Let $X$ and $\tilde X$ be mean zero random vectors taking values in $\R^{d_1}$ and $\R^{d_2}$, respectively.
    Suppose $X$ and $\tilde X$ satisfy Assumption \ref{asm: cross-covariance concentration}.
    Let $(a_i)_i$ be a bounded sequence of positive numbers such that $\max_i a_i \leq a$.
    Let $\{(X_i, \tilde X_i)\}_i$ be independent copies of $(X, \tilde X)$ and $\hat \Sigma_{X,\tilde X}^a \triangleq (1/n) \sum_{i=1}^n a_i X_i \tilde X_i^\top$. Let $\Sigma_{X,\tilde X}^a =  (1/n) (\sum_{i=1}^n a_i) \mathbb{E}[X \tilde X^\top]$.
    Then, there exists a constant $C = C(c_1, c_2) > 0$ such that
    with probability at least $1 - e^{-t}$,
    \begin{align*}
        &\|\hat \Sigma_{X,\tilde X}^a - \Sigma_{X,\tilde X}^a\|\\
        &\leq C a \qty[ (\tr(\Sigma_{\tilde X}) \|\Sigma_X\| \vee \tr(\Sigma_X) \|\Sigma_{\tilde X}\|)^{1/2} \sqrt{\frac{t + \log(d_1 + d_2)}{n}} \vee (\tr(\Sigma_X) \tr(\Sigma_{\tilde X}))^{1/2} \frac{t + \log(d_1 + d_2)}{n} ].
    \end{align*}
    holds for all $t > 0$.
\end{prop}

Notice that when $X = \tilde X$, we recover the bound given in Theorem 2.2 of \cite{bunea2015sample}.

\begin{proof}[Proof of Proposition \ref{prop: cross-covariance concentration}]
    Let $B_i \triangleq X_i \tilde X_i^\top - \Sigma_{X,\tilde X}$. 
    Define symmetric matrices $A_i$ of order $d_1 + d_2$ as
    \begin{align*}
        A_i \triangleq \begin{pmatrix}
            O & B_i\\
            B_i^\top & O
            \end{pmatrix}.
    \end{align*}
    Since $A_1^{2k} = \diag((B_1 B_1^\top)^k, (B_1^\top B_1)^k)$ for $k \geq 2$, $\|\mathbb{E}[A_1^{2k}]\| \leq \|\mathbb{E}[(B_1 B_1^\top)^k]\| \vee \|\mathbb{E}[(B_1^\top B_1)^k]\|$ and $A_1^{2k}$ is positive semi-definite.
    From Lemma \ref{lem: moments of B}, for $k \geq 1$,
    \begin{align*}
        \|\mathbb{E}[A_1^{2k}]\| \leq \frac{(2k)!}{2} R^{2k-2} \sigma^2,
    \end{align*}
    where $\sigma^2$ and $R$ are defined in Lemma \ref{lem: moments of B}. 
    Fix any $u \in \mathbb{S}^{d_1+d_2-1}$.
    By Cauchy-Schwarz inequality and Jensen's inequality, for $k \geq 2$,
    \begin{align*}
        \mathbb{E}[u^\top A_1^{2k-1} u] &\leq \mathbb{E}[\sqrt{u^\top A_1^{2k-2}u u^\top A_1^{2k} u}]
        \leq \sqrt{u^\top \mathbb{E}[A_1^{2k-2}] u u^\top \mathbb{E}[A_1^{2k}] u}
        \leq \frac{\sqrt{(2k)! (2k-2)!}}{2} R^{2k-3} \sigma^2.
    \end{align*}
    Observe that
    \begin{align*}
        \frac{\sqrt{(2k)! (2k-2)!}}{(2k-1)!} = \sqrt{\frac{2k}{2k-1}} \leq \frac{2}{\sqrt{3}}.
    \end{align*}
    Therefore, substituting $\sigma^2 \leftarrow (4/3)^{1/2} \sigma^2$, we obtain the bound $\|\mathbb{E}[A_1^k]\| \leq (k!/2) R^{k-2} \sigma^2$ for all $k \geq 2$.
    Applying Theorem 6.2 in \cite{tropp2012user} to $(1/n)\sum_{i=1}^n a_i A_i$, and using $\|B_i\| = \|(I_{d_1} O_{d_1\times d_2}) A_i (O_{d_2 \times d_1} I_{d_2})^\top\| \leq \|A_i\|$ and $\|E(a_i A_i)^{k}\| \leq (k)! (a R)^{k-2} (a\sigma)^2/2$,
    we obtain the following bound: there exists a constant $C > 0$ only depending on $c_1$ and $c_2$ such that
    \begin{align*}
        &\|\hat \Sigma_{X,\tilde X}^a - \Sigma_{X,\tilde X}^a\|\\
        &\leq C a \qty[ (\tr(\Sigma_{\tilde X}) \|\Sigma_X\| \vee \tr(\Sigma_X) \|\Sigma_{\tilde X}\|)^{1/2} \sqrt{\frac{t + \log(d_1 + d_2)}{n}} \vee (\tr(\Sigma_X) \tr(\Sigma_{\tilde X}))^{1/2} \frac{t + \log(d_1 + d_2)}{n} ].
    \end{align*}
    holds with probability at least $1 - e^{-t}$ for all $t > 0$.
    \if0
    \begin{align*}
        \|\hat \Sigma_{X,\tilde X} - \Sigma_{X,\tilde X}\| \leq (\tr(\Sigma_{\tilde X}) \|\Sigma_X\| \vee \tr(\Sigma_X) \|\Sigma_{\tilde X}\|)^{1/2} \sqrt{\frac{t + \log(d_1 + d_2)}{n}} \vee (\tr(\Sigma_X) \tr(\Sigma_{\tilde X}))^{1/2} \frac{t + \log(d_1 + d_2)}{n}.
    \end{align*}
    This further implies that
    \begin{align*}
        \|\hat \Sigma_{X,\tilde X} - \Sigma_{X,\tilde X}\| \leq c \qty( (\tr(\Sigma_{\tilde X}) \|\Sigma_X\| \vee \tr(\Sigma_X) \|\Sigma_{\tilde X}\|)^{1/2} \sqrt{\frac{\log (nd_1+nd_2)}{n}} \vee (\tr(\Sigma_X) \tr(\Sigma_{\tilde X}))^{1/2} \frac{\log (nd_1+nd_2)}{n} )
    \end{align*}
    holds with probability at least $1 - n^{-1}$ for some constant $c > 0$ only depending on $c_1, c_2$.
    \fi
\end{proof}

\begin{lem}\label{lem: moments of B}
    Let $X$ and $\tilde X$ be mean zero random vectors taking values in $\R^{d_1}$ and $\R^{d_2}$, respectively.
    Suppose $X$ and $\tilde X$ satisfy Assumption \ref{asm: cross-covariance concentration}. Then, for $k \geq 1$,
    \begin{align*}
        &\|\mathbb{E}[((X \tilde X^\top - \Sigma_{X,\tilde X}) (X \tilde X^\top - \Sigma_{X,\tilde X})^\top)^k]\| \vee
        \|\mathbb{E}[((\tilde X X^\top - \Sigma_{\tilde X,X}) (\tilde X X^\top - \Sigma_{\tilde X,X})^\top)^k]\| \leq \frac{(2k)!}{2} L^{2k-2} \sigma^2,
    \end{align*}
    where
    \begin{align*}
        \sigma^2 &\triangleq \frac{65 \cdot 16 e^2}{c_1 c_2 \wedge 1} (\tr(\Sigma_{\tilde X}) \|\Sigma_X\| \vee \tr(\Sigma_X) \|\Sigma_{\tilde X}\|)\\
        L &\triangleq \frac{8 e}{(c_1 c_2 \wedge 1)^{1/2}} ( \tr(\Sigma_X) \tr(\Sigma_{\tilde X}) )^{1/2}.
    \end{align*}
\end{lem}

\begin{proof}[Proof of Lemma \ref{lem: moments of B}]
    %Let $(X_i, \tilde X_i)_i$ be the independent copy of $(X, \tilde X)$.
    Define $B \triangleq X \tilde X^\top - \Sigma_{X,\tilde X}$. 
    To use the matrix bernstein inequality, we bound $\|\mathbb{E}[(B B^\top)^k]\|$. Fix any $u \in \mathbb{S}^{d_1-1}$.
    Then
    \begin{align*}
        u^\top (B B^\top)^k u & \leq \|B\|^{2k-2} u^\top B B^\top u.
    \end{align*}
    
    Since for any matrices $C$ and $D$, $2 CC^\top + 2 DD^\top - (C - D) (C - D)^\top = (C + D) (C + D)^\top$,
    we obtain $0 \preceq (C - D) (C - D)^\top \preceq 2 CC^\top + 2 DD^\top \preceq 2 CC^\top + 2 \|D\|^2 I$.
    Also, $\|C - D\|^{2k-2} \leq 2^{2k-3} (\|C\|^{2k-2} + \|D\|^{2k-2})$.
    These results give
    \begin{align*}
        u^\top (B B^\top)^k u &\leq \|B\|^{2k-2} u^\top B B^\top u\\
        &\leq 2^{2k-1} (\|X \tilde X^\top\|^{2k-2} + \|\Sigma_{X,\tilde X}\|^{2k-2}) (u^\top X \tilde X^\top \tilde X X^\top u + \|\Sigma_{X,\tilde X}\|^2)\\
        &= 2^{2k-1} (\|X \tilde X^\top\|^{2k-2} u^\top X \tilde X^\top \tilde X X^\top u + \|X \tilde X^\top\|^{2k-2} \|\Sigma_{X,\tilde X}\|^2)\\
        &\quad+ 2^{2k-1}(\|\Sigma_{X,\tilde X}\|^{2k-1} u^\top X \tilde X^\top \tilde X X^\top u + \|\Sigma_{X,\tilde X}\|^{2k}).
    \end{align*}

    \begin{align*}
        \mathbb{E}[u^\top (B B^\top)^k u] &\leq 2^{2k-1} (\mathbb{E}[\|X \tilde X^\top\|^{2k-2} u^\top X \tilde X^\top \tilde X X^\top u] + \mathbb{E}[\|X \tilde X^\top\|^{2k-2}] \|\Sigma_{X,\tilde X}\|^2)\\
        &\quad+ 2^{2k-1}(\|\Sigma_{X,\tilde X}\|^{2k-2} \mathbb{E}[u^\top X \tilde X^\top \tilde X X^\top u] + \|\Sigma_{X,\tilde X}\|^{2k}).
    \end{align*}
    
    Notice that
    \begin{align*}
        \mathbb{E}[\|X \tilde X^\top\|^{2k-2} u^\top X \tilde X^\top \tilde X X^\top u] &\leq \sqrt{ \mathbb{E}[\|X \tilde X^\top\|^{2(2k-2)}] \mathbb{E}[(u^\top X \tilde X^\top \tilde X X^\top u)^2] },
    \end{align*}
    where the first inequality follows from Cauchy-Schwarz inequality. %, and the second inequality follows from Jensen's inequality. 
    By Lemma A.2 from \cite{bunea2015sample} and Assumption \ref{asm: cross-covariance concentration},
    \begin{align*}
        \mathbb{E}[\|X \tilde X^\top\|^{2(2k-2)}] &\leq \sqrt{ \mathbb{E}[\|X\|^{4(2k-2)}] \mathbb{E}[\|\tilde X\|^{4(2k-2)}] }
        %&\leq \qty( \frac{(4(2k-1))^{2(2k-1)}}{c_1^{2(2k-1)}} (\tr(\Sigma_X))^{2(2k-1)} )^{1/2} \qty( \frac{(4(2k-1))^{2(2k-1)}}{c_2^{2(2k-1)}} (\tr(\Sigma_{\tilde X}))^{2(2k-1)} )^{1/2}
        \leq \frac{(4(2k-2))^{2(2k-2)}}{c_1^{2k-2} c_2^{2k-2}} (\tr(\Sigma_X))^{2k-2} (\tr(\Sigma_{\tilde X}))^{2k-2},\\
        \mathbb{E}[(u^\top X \tilde X^\top \tilde X X^\top u)^2] &= \mathbb{E}[(u^\top X)^4 \|\tilde X\|^4] \leq \sqrt{ \mathbb{E}[(u^\top X)^8] \mathbb{E}[\|\tilde X\|^8] }
        \leq \qty(\frac{64}{c_1 c_2} \tr(\Sigma_{\tilde X}) \|\Sigma_X\|)^2.
    \end{align*}
    Thus
    \begin{align*}
        \mathbb{E}[\|X \tilde X^\top\|^{2k-2} u^\top X \tilde X^\top \tilde X X^\top u] 
        &\leq \frac{64(4(2k-2))^{2k-2}}{c_1^{k} c_2^{k}} ( \tr(\Sigma_X) \tr(\Sigma_{\tilde X}) )^{(2k-2)/2} \tr(\Sigma_{\tilde X}) \|\Sigma_X\|.
    \end{align*}
    Similarly,
    \begin{align*}
        \mathbb{E}[\|X \tilde X^\top\|^{2k-2}] &\leq \sqrt{ \mathbb{E}[\|X\|^{2(2k-2)}] \mathbb{E}[\|\tilde X\|^{2(2k-2)}] }
        \leq \frac{2(2k-2)^{2k-2}}{c_1^{k-1} c_2^{k-1}} (\tr(\Sigma_X) \tr(\Sigma_{\tilde X}))^{(2k-2)/2},\\
        \mathbb{E}[u^\top X \tilde X^\top \tilde X X^\top u] &\leq \sqrt{ \mathbb{E}[(u^\top X)^4] \mathbb{E}[\|\tilde X\|^4] }
        \leq \frac{16}{c_1 c_2} \tr(\Sigma_{\tilde X}) \|\Sigma_X\|.
    \end{align*}
    Also, for any $(u, v) \in \mathbb{S}^{d_1-1} \times \mathbb{S}^{d_2-1}$,
    \begin{align*}
        u^\top \Sigma_{X,\tilde X} v = \mathbb{E}[u^\top X \tilde X^\top v] \leq \sqrt{\mathbb{E}[(u^\top X)^2] \mathbb{E}[(v^\top \tilde X)^2]} \leq \sqrt{\|\Sigma_X\| \|\Sigma_{\tilde X}\|} \leq \sqrt{\tr(\Sigma_{\tilde X}) \|\Sigma_X\|}.
    \end{align*}
    This gives $\|\Sigma_{X,\tilde X}\|^2 \leq \tr(\Sigma_{\tilde X}) \|\Sigma_X\| \leq \tr(\Sigma_{\tilde X}) \tr(\Sigma_X)$.
    Therefore,
    \begin{align*}
        \mathbb{E}[u^\top (B B^\top)^k u] &\leq 2^{2k-1} \frac{64(4(2k-2))^{2k-2}}{c_1^{k} c_2^{k}} ( \tr(\Sigma_X) \tr(\Sigma_{\tilde X}) )^{(2k-2)/2} \tr(\Sigma_{\tilde X}) \|\Sigma_X\|\\
        &\quad+ 2^{2k-1} \frac{(2(2k-2))^{2k-2}}{c_1^{k-1} c_2^{k-1}} (\tr(\Sigma_X) \tr(\Sigma_{\tilde X}))^{(2k-2)/2} \|\Sigma_{X,\tilde X}\|^2\\
        &\quad+ 2^{2k-1} \|\Sigma_{X,\tilde X}\|^{2k-2} \frac{16}{c_1 c_2} \tr(\Sigma_{\tilde X}) \|\Sigma_X\| + 2^{2k-1} \|\Sigma_{X,\tilde X}\|^{2k}\\
        &\leq 2^{2k-1} \qty(\frac{(4(2k-1))^{2k-1}}{c_1^{k-1} c_2^{k-1}} ( \tr(\Sigma_X) \tr(\Sigma_{\tilde X}) )^{(2k-2)/2} + \|\Sigma_{X,\tilde X}\|^{2k-2})\\
        &\quad\times\tr(\Sigma_{\tilde X}) \|\Sigma_X\| \qty(\frac{64}{c_1 c_2} + 1)\\
        &\leq 2^{2k} \frac{(4(2k-1))^{2k-1}}{c_1^{k-1} c_2^{k-1} \wedge 1} ( \tr(\Sigma_X) \tr(\Sigma_{\tilde X}) )^{(2k-2)/2} \tr(\Sigma_{\tilde X}) \|\Sigma_X\| \frac{65}{c_1 c_2 \wedge 1}.
    \end{align*}
    Hence
    \begin{align*}
        \|\mathbb{E}[(B B^\top)^k]\| &\leq \frac{(2k)!}{2} \frac{2}{(2k)!} 2^{2k} \frac{(4(2k-1))^{2k-1}}{c_1^{k-1} c_2^{k-1} \wedge 1} ( \tr(\Sigma_X) \tr(\Sigma_{\tilde X}) )^{(2k-2)/2} \tr(\Sigma_{\tilde X}) \|\Sigma_X\| \frac{65}{c_1 c_2 \wedge 1}.
    \end{align*}
    Using the Stirling's formula $k! \geq \sqrt{2\rho} k^{k+1/2} e^{-k} e^{1/(12k+1)} \geq 2 k^{k+1/2} e^{-k}$,
    \begin{align*}
        &\|\mathbb{E}[(B B^\top)^k]\|\\
        &\leq \frac{(2k)!}{2} \frac{e^{2k}}{(2k)^{2k+1/2}} 2^{2k} \frac{(4(2k-1))^{2k-1}}{c_1^{k-1} c_2^{k-1} \wedge 1} ( \tr(\Sigma_X) \tr(\Sigma_{\tilde X}) )^{(2k-2)/2} \tr(\Sigma_{\tilde X}) \|\Sigma_X\| \frac{65}{c_1 c_2 \wedge 1}\\
        &\leq \frac{(2k)!}{2} \qty(65\cdot 16e^2 \frac{1}{c_1 c_2 \wedge 1} \tr(\Sigma_{\tilde X}) \|\Sigma_X\|) 
        \qty(8e (\tr(\Sigma_X) \tr(\Sigma_{\tilde X}) )^{1/2} \frac{1}{c_1^{1/2} c_2^{1/2} \wedge 1})^{2k-2}.
    \end{align*}
    By symmetry of $X, \tilde X$, this concludes the proof.
\end{proof}

\end{appendices}

\end{document}